\documentclass[accepted]{article} 
\usepackage[british]{babel}

\usepackage{microtype}
\usepackage{graphicx}
\usepackage{fullpage}
\usepackage{booktabs} %
\usepackage{array}
\usepackage{appendix}
\usepackage{bbm}
\usepackage{thm-restate}
\usepackage{textcomp}
\usepackage{hyperref}

\usepackage{algorithm}
\usepackage{algorithmicx}
\usepackage{algpseudocode}

\usepackage{amsmath}
\usepackage{amssymb}
\usepackage{mathtools}
\usepackage{amsthm}
\usepackage{subcaption}
\usepackage{xspace}
\usepackage{nicefrac}

\usepackage[inline]{enumitem}

\usepackage{xargs}  
\usepackage[numbers,sort]{natbib}
\usepackage[dvipsnames]{xcolor}

\usepackage{authblk}

\usepackage[colorinlistoftodos,prependcaption,textwidth=15mm,textsize=small]{todonotes}
\newcommandx{\change}[2][1=]{\todo[linecolor=blue,backgroundcolor=blue!25,bordercolor=blue,#1]{#2}}

\usepackage{zref-xr,zref-user}

\usepackage{cleveref}
\usepackage{amartya_ltx}

\newcommand{\err}[1]{\mathrm{err}_m\xspace\br{#1}}
\newcommand{\fair}[1]{\Gamma_m\xspace \br{#1}}
\newcommand{\errinf}[1]{\mathrm{err}\xspace\br{#1}}
\newcommand{\fairinf}[1]{\Gamma\xspace \br{#1}}

\allowdisplaybreaks

\newcommand{\dist}{\Pi_{p,N}}
\newcommand{\Sell}{S^\ell}
\newcommand{\Soell}{S_{2}^\ell}
\newcommand{\Stell}{S_{1}^\ell}
\newcommand{\Slabel}{S_f}
\newcommand{\hsample}[1]{h\sim\cA\br{S_{#1}}}
\newcommand{\clblpriorx}[1]{F_{\vert S\setminus\bc{#1},f}}
\newcommand{\clblprior}{F_{\vert S\setminus\bc{x'},f}}
\newcommand{\Qlblprior}{\cF_{\vert S\setminus\bc{x},Q}}
\newcommand{\Qlblset}{F_{\vert S\setminus\bc{x},Q}}
\newcommand{\Qxlblprior}{\cF_{\vert \zeta,Q}}
\newcommand{\lblprior}{\cF}
\newcommand{\lblset}{F}
\newcommand{\cifar}{CIFAR-10\xspace}
\newcommand{\indic}[1]{\mathbbm{I}\bc{#1}}
\newcommand{\pmvc}[1]{\mathrm{PMVC}\br{#1}}

\newcounter{relctr} %
\everydisplay\expandafter{\the\everydisplay\setcounter{relctr}{0}} %

\newcommand\labelrel[2]{%
  \begingroup
    \refstepcounter{relctr}%
    \stackrel{\textnormal{(\alph{relctr})}}{\mathstrut{#1}}%
    \originallabel{#2}%
  \endgroup
}
\AtBeginDocument{\let\originallabel\label} %

\makeatletter
\def\@fnsymbol#1{\ensuremath{\ifcase#1\or \dagger\or \ddagger\or
   \mathsection\or \mathparagraph\or \|\or **\or \dagger\dagger
   \or \ddagger\ddagger \else\@ctrerr\fi}}
    \makeatother

\title{How unfair is private learning?}

\author[1,3]{\href{mailto:<amartya18x@gmail.com>?Subject=fairness
privacy paper}{Amartya Sanyal}\textsuperscript{*}}
\author[2]{Yaxi Hu\textsuperscript{*}}
\author[3]{Fanny Yang}
\affil[1]{%
    ETH AI Center\\
    ETH Z\"urich\\
    Z\"urich, Switzerland.
}
\affil[2]{%
    Department of Mathematics\\
    ETH Z\"urich\\
    Z\"urich, Switzerland.
}
\affil[3]{%
    Department of Computer Science\\
    ETH Z\"urich\\
    Z\"urich, Switzerland.
  }

\graphicspath{{./figs/},}

\date{}
\begin{document}
\maketitle

\begin{abstract}
    As machine learning algorithms are deployed on sensitive data in
    critical decision making processes, it is becoming increasingly
    important that they are also private and fair. In this paper, we show
    that, when the data has a long-tailed structure, it is not possible to
    build accurate learning algorithms that are both private and results
    in higher accuracy on minority subpopulations. We further show that
    relaxing overall accuracy can lead to good fairness even with strict
    privacy requirements. To corroborate our theoretical results in
    practice, we provide an extensive set of experimental results using a
    variety of synthetic, vision~(\cifar and CelebA), and tabular~(Law School)
    datasets and learning algorithms.
    \end{abstract}

\section{Introduction}
\label{sec:intro}
In recent years, reliability of machine learning algorithms have
become ever more important due to their widespread use in daily life.
Fairness and privacy are two instances of such reliability traits that
are desirable but often absent in modern machine learning
algorithms~\citep{shokri2017membership,de2019does,buolamwini2018gender}%
As a result, there has
been a flurry of recent works that aim to improve these properties in
commonly used learning algorithms.  However, most of these works
discuss these two properties individually with relatively less
attention paid to how they affect each other.

There is a multitude of definitions for privacy and fairness in their
respective literatures. Perhaps the most widespread statistical notion
of privacy is that of {\em Differential Privacy}~\citep{Dwork_2006}
and its slightly relaxed variant, {\em Approximate Differential
Privacy}~\citep{DMNS2006}.  Despite its marginally weaker privacy
guarantees, {\em Approximate Differential Privacy} enjoys better
theoretical guarantees in terms of statistical complexity for
learning~\citep{feldman2014sample,beimel2013private}. It is also more
widely used in practice~\citep{opacus,Abadi_2016}. Thus, we always use
approximate differential privacy in this paper and for the sake of
brevity, refer to it as differential privacy~(DP). Intuitively, DP
limits the amount of influence any single data point has on the output
of the DP algorithm. This ensures that DP algorithms do not leak
information about whether any particular data point was given as input
to the algorithm. While DP was initially popular as a theoretical
construct, it has recently been put to practical use by large
companies~\citep{googledp,appledp} and governments~\citep{uscensus}
alike. Its popularity is largely due to its strong privacy guarantees,
ease of implementation, and the quantitative nature of differential
privacy. 

\looseness=-1
There are many notions of fairness in machine
learning~\citep{hardt2016equality,hebert2018multicalibration,dwork2012fairness}.
Minority or worst group accuracy~\citep{koh2021wilds,du2021fairness}
and its difference from the overall accuracy is a common notion of
fairness used in recent works. We define this difference as {\em
accuracy discrepancy} and use it to measure the degree of unfairness
in this paper. However, we expect our results to translate to other
related fairness metrics as well.~\citet{Sagawa2020Distributionally}
observed that robust optimisation methods~(under appropriate
regularisations) obtain higher minority group accuracy but at the cost
of a lower overall accuracy compared to vanilla training. Several
other works have also observed this
behavior in practice thereby suggesting a possible trade-off between
overall accuracy and fairness metrics. Subsequent works
including~\citet{goel2020} and~\citet{menon2021overparameterisation}
have tried to minimise this trade-off.

While there have been a large body of works that aim to minimising the
trade-off between accuracy and
fairness~\citep{goel2020,Sagawa2020Distributionally,du2021fairness}
and between accuracy and privacy~\citep{DMNS2006,ghazi2021}, there is
relatively few works~\citep{fioretto2022differential} that investigate
the intersection of privacy, fairness, and accuracy. In this paper, we
provide theoretical and experimental results to show that private and
accurate algorithms are necessarily unfair. We further show that
achieving privacy and fairness simultaneously leads to inaccurate
algorithms. 

\noindent\textbf{Contributions}
Our main contributions can be stated as ---
\begin{itemize}[leftmargin=5.5mm]
    \item In~\Cref{thm:thm_2,thm:thm_1}, we provide asymptotic lower
    bounds for unfairness~(accuracy discrepancy) of DP algorithms,
    that are accurate, showing that privacy and accuracy comes at the
    cost of fairness.
    \item In~\Cref{thm:thm_3}, we show that in a very strict privacy
    regime, fairness can be achieved at the cost of accuracy.
    \item In~\Cref{sec:exp}, we conduct experiments using
    multiple architectures on synthetic and real world
    datasets~(CelebA, \cifar, and Law School)
    to validate our theory.
\end{itemize}

\subsection*{Related works}

It is now well understood that by imposing these additional conditions
of DP more data are required to achieve high accuracy. A string of
theoretical
works~\citep{feldman2014sample,Beimel_2010,bun2020equivalence}
have shown that the sample complexity of learning certain concept
classes privately and accurately can be arbitrarily larger than
learning the same classes non-privately~(i.e. with high accuracy but
without privacy). On the other hand, it is easy to guarantee any
arbitrary level of differential privacy if high accuracy is not
desired. This can be achieved by simply composing the output of an
accurate classifier with a properly calibrated {\em randomised
response} mechanism~\citep{Warner1965}. This allows for a tradeoff
between differential privacy and accuracy.

One of the most popular notions of fairness is {\em group fairness}
that compares the performance of the algorithm on a minority group
with other groups in the data. A popular instantiation of this,
especially for deep learning algorithms, is comparing the accuracy on
the minority group against the entire
population~\citep{Sagawa2020Distributionally,koh2021wilds,du2021fairness,
goel2020}. In the fairness
literature,~\citet{buolamwini2018gender,Raz_2021} shows extensively
that this discrepancy is large between different groups of people for
popular facial recognition systems.

DP-SGD~\citep{Abadi_2016} is a widely used algorithm for implementing
differentially private deep learning
models.~\citet{bagdasaryan2019differential} provides some experimental
evidence that DP-SGD can have disparate impact on
accuracy.~\citet{tran2021differentially} provides theoretical
explanation for why DP-SGD and output perturbation based DP methods
suffer from fairness issues. Conversely,~\citet{Chang_2021} shows,
experimentally, that fairness aware machine learning algorithms suffer
from less privacy. However, unlike these works, we provide theoretical
results that are model agnostic and that discuss the dependance of the
trade-off on the subpopulation sizes and frequencies.

\citet{Cummings2019} and~\citet{Agarwal2021} were one of the first to
consider the impact of privacy on fairness theoretically. They
construct a distribution where any algorithm that is always fair and
private will necessarily output a trivial constant classifier, thereby
suggesting a tradeoff between fairness and privacy. However, there are
multiple drawbacks with their work. First, their work only discusses
pure differential privacy which is not only theoretically more
restrictive than approximate differential
privacy~\citep{beimel2013characterizing,beimel2013private,feldman2014sample}
but also rarely used in practice. Second, their proof heavily relies
on it being {\em pure} differential privacy and the algorithm being
{\em always} fair; and their proofs are not amenable to relaxations of
these assumptions. Third, their result does not show how unfairness
increases with stricter privacy or different properties of the
distribution. Further, they do not provide experiments to corroborate
their theory perhaps due to the unrealistic requirements of pure DP.
On the other hand, we look at approximate DP~(which is a stronger
result than pure DP), construct bounds for both fairness and error,
and provide experimental results to support our theory.  Perhaps, most
closely related to our work is that of~\citet{Feldman2019}, who
studies, mainly, the impact of memorisation on test accuracy for
long-tailed distributions. However, neither does their work focus on
differential privacy nor fairness.

\section{Theoretical results}
\label{sec:theory}
The main contribution of our work is to provide a qualitative
explanation for why and when differentially private algorithms cannot
be simultaneously accurate and fair. Real world data distributions
often contain a large number of subpopulations with very few examples
in each of them and a few subpopulations with a large number of
examples.

For example,~\Cref{fig:intro-assump}~(left) depicts the distribution
of subpupoluations in CelebA. Using the 40 attributes of the CelebA
dataset~\citep{liu2015faceattributes}, we partition the training set~(of
size \(m=160k\)) into \(2^{40}\) subpopulation bins. The blue shaded
area shows the group of the subpopulations with large number of
examples~(probability mass greater than \(\frac{1}{m}\)) and
the red shaded area corresponds to subpopulations which contain just
one example in them. We refer to the subpopulations in the red area as
{\em minority subpopulations} and the subpopulations in the blue area
as {\em majority subpopulations}. 
This long-tailed structure over subpopulations have also been observed
in~\citet{zhu2014capturing} in other vision datasets like the
SUN~\citep{xiao2010sun} and PASCAL~\citep{everingham2010pascal}
datasets.~\citet{babbar2019data} observes this structure in extreme multilabel
classification datasets like Amazon-670K~\citep{mcauley2013hidden} and
Wikipedia-31k~\citep{Bhatia16} datasets. Various other
works~\citep{van2017devil,krishna2017visual,wang2017learning,cui2018large}
have shown this structure in a range of datasets including
eBird~\citep{sullivan2009ebird}, Visual
Genome~\citep{krishna2017visual}, Pasadena
trees~\citep{wegner2016cataloging}, and
iNaturalist~\citep{van2018inaturalist}.

\subsection{Problem setup}
\label{sec:prelim}

Mathematically, the large number of small subpopulations discussed
above constitute the {\em long tail} of the distribution. We use this
structure of data distributions to illustrate the tension between
accuracy and fairness of private algorithms. For our theoretical
results, we view each subpopulation as an element of a discrete set
\(X\) without any intrinsic structure such as distance. This
distribution is inspired by the use of a similar distribution
in~\citet{Feldman2019}. Next, we define a distribution over the subpopulations in \(X\) that reflect
the {\em long-tailed} structure.

\begin{defn}[\(\br{p,N,k}\)-long-tailed distribution on \(X\)]
    \label{defn:comp-unif}
    Given \(p\in\br{0,1},~N\in\bN\), and \(1<k\ll N\), define two groups
    \begin{enumerate*}[label=(\roman*)]\item~the group of majority subpopulations~\(X_1\subset X\) where \(\abs{X_1}=\br{1-p}k\) and %
    \item~the group of minority subpopulations~\(X_2\subset X\setminus X_1 \), where  \(|X_2| = N\)\end{enumerate*}
    \footnote{WLOG we will assume that \(k\) is such that
    \(\br{1-p}k\) is an integer and if not, replace \(k\) with the
    closest number such that \(\br{1-p}k\) is an integer.}.   Now, define the distribution \(\Pi_{p, N, k}\) as  
    \begin{equation}
        	\Pi_{p,N, k}(x) = \left\{\begin{array}{lr}
		\frac{1}{k}& x\in X_1\\
		\frac{p}{N} & x\in X_2.
    		\end{array}
    \right.
    \end{equation}
    We use the terms {\em group of majority subpopulations} and {\em majority
group} interchangeably to denote \(X_1\) and the terms {\em group of
minority subpopulations} and {\em minority group} to refer to \(X_2\)
respectively. 
\end{defn}

\begin{figure}[t]\centering
    \begin{subfigure}[c]{0.45\linewidth}
        \centering
        \def\svgwidth{0.99\columnwidth}
        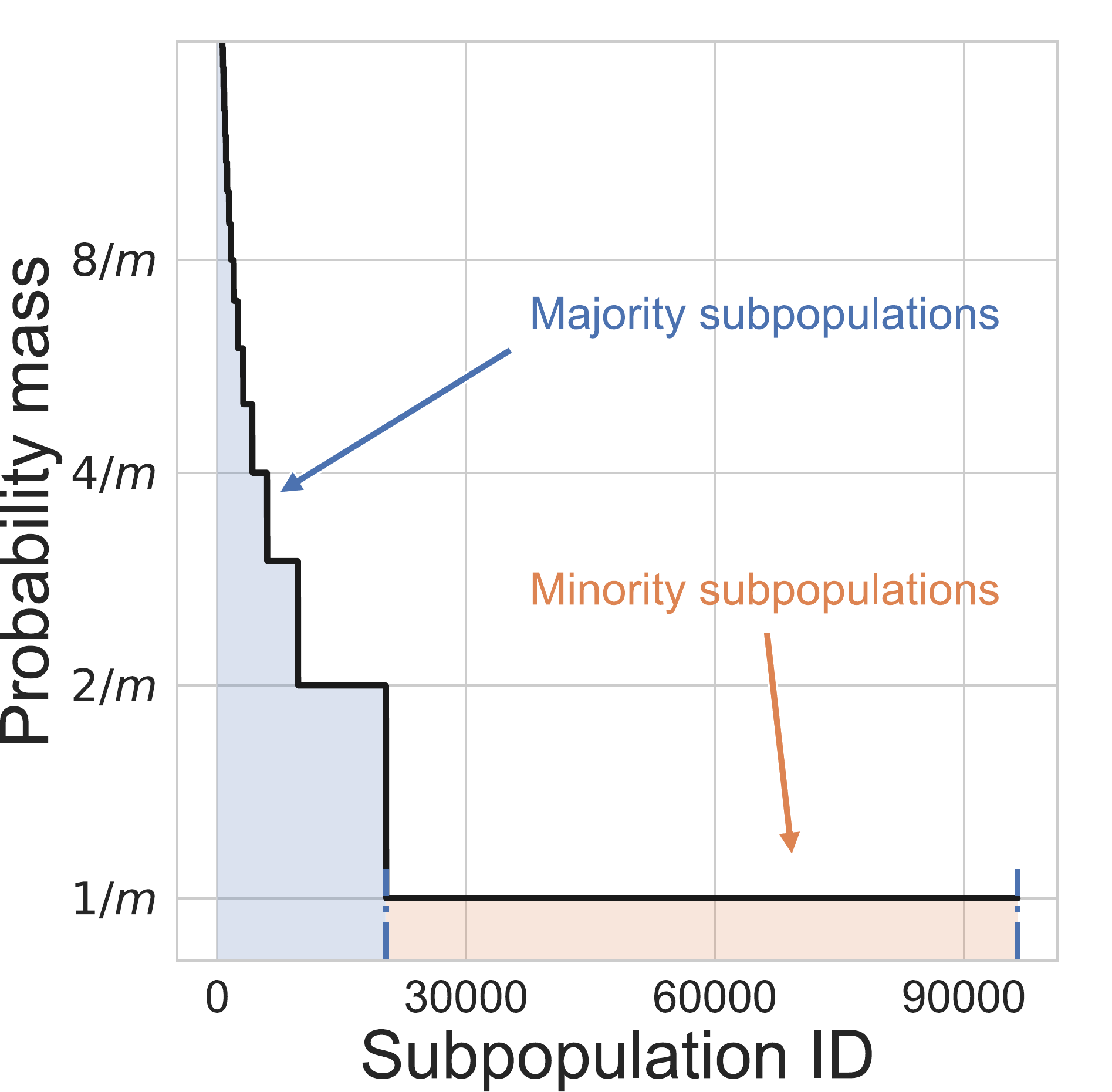%
        \end{subfigure}
    \begin{subfigure}[c]{0.54\linewidth}
        \centering
        \def\svgwidth{0.99\columnwidth}
        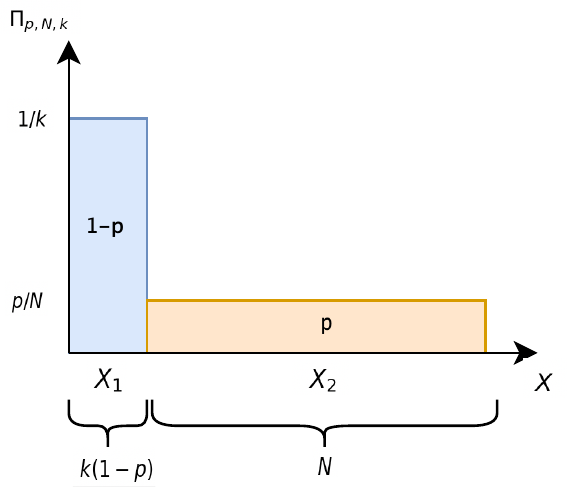%
        \end{subfigure}
    \caption{{\bfseries (Left)} Illustration of the distribution of
    majority and minority subpopulations of CelebA. Here \(m=160k\)
    is the total size of the training set of CelebA. {\bfseries (Right)}~Illustration of \(\Pi_{p,N,k}\).}
    \label{fig:intro-assump}
  \end{figure}

We provide an illustration of the distribution
in~\Cref{fig:intro-assump}~(right).  Here, \(p\) denotes the total probability mass
of the group of minority subpopulations under \(\Pi_{p,N,k}\) and
\(N\) denotes the number of minority subpopulations.  We let \(N\) go
to \(\infty\) and treat~\(k\) as a constant. Thus, for the sake of
simplicity, we remove \(k\) from the notation of the distribution.
Note that each minority subpopulation i.e. each element in \(X_2\) has a
probability mass of the order of \(\bigO{\frac{1}{N}}\) which is much
smaller than \(\frac{1}{k}=\Omega\br{1}\), i.e. the probability mass of each
element in \(X_1\).

Note that in the distribution for CelebA~(\Cref{fig:intro-assump}~(left))
the probability masses of the majority subpopulations not exactly
equal to \(\frac{1}{k}\) as in the distribution
of~\Cref{defn:comp-unif}. %
However, all our results hold true even if the majority subpopulations
have different probability masses, as long as
\(\dist\br{x}=\Omega\br{\frac{1}{k}}\) for some \(k=\bigO{1}\) for all
\(x\in X_1\). We set them to \(\frac{1}{k}\) only for simplicity of
the theoretical results. 

As we deal with a multiclass classification setup, we also define a
label space \(\cY\) and a function space \(\lblset\) of labelling
functions. We use \(\lblprior\) to represent a distribution on the
function space \(\lblset\) and refer to this distribution as the label
prior. Our results do not restrict the size of \(\cY\) and hence,
cover both binary and multi-class classification settings. We further
use the notation $\lblprior\br{x}$ to denote the probability mass
function induced over \(\cY\) as
\(\lblprior\br{x}\bs{y}=\bP_{f\sim\lblprior}\bs{f\br{x}=y}\).

For any \(\zeta\subset X\) and labelling \(Q\in\cY^{\abs{\zeta}}\),
let \(\Qxlblprior\) denote the marginal distribution over the subset
of labelling functions that satisfy the labelling \(Q\) on \(\zeta\).
For each \(x\in X\setminus \zeta\), this induces a probability mass
function over \(\cY\) as \(\Qxlblprior\br{x}\bs{y}=\bP_{f\sim
\Qxlblprior}\bs{f\br{x}=y}\). We say a label prior  \(\cF\) is {\em
subpopulation-wise independent} if the distribution
\(\Qxlblprior\br{x}\) is independent of the choice of \(Q\) and
\(\zeta\) for all \(x\in X\) i.e. \(\lblprior\br{x} =
\Qxlblprior\br{x}\) for all \(x\). Intuitively, this means that the
uncertainty in labelling a subpopulation \(x\in X\) does not change
when the labelling of any other group of subpopulations is fixed.
Finally, define
\begin{equation}\label{eq:un-entropy}
    \norm{\cF}_\infty = \max_{x\in X,y\in\cY}\cF\br{x}\bs{y}.
\end{equation}

To generate a dataset of size \(m\) from a
\(\br{p,N}\)-long-tailed distribution \(\Pi_{p,N}\) on \(X\), first
sample an unlabelled dataset  \(S=\bc{x_1,\cdots,x_m}\) of size \(m\)
from \(\dist\). Then, generate the labelled dataset
\(\Slabel=\bc{\br{x_1,f\br{x_1}},\ldots,\br{x_m,f\br{x_m}}}\) using a
labelling function \(f\sim\lblprior\). In all our theoretical results,
we consider an asymptotic regime where \(\frac{N}{m}\to c\) as
\(N,m\to\infty\). This is common in high-dimensional statistics where
the number of dimensions often grows to \(\infty\) along with the
sample size.  Intuitively, \(c\) quantifies the hardness of the
learning problem as it is inversely proportional to the number of data
points observed per minority subpopulation. 

Next, we define the error and fairness measure of an algorithm on the
distribution defined above. Consider a domain \(X\), a label space
\(\cY\), a set of labelling functions \(\lblset\), a label prior
\(\lblprior\), and a distribution \(\Pi_{p,N}\) on \(X\).%

\subsection{Privacy, error, and fairness}
In the context of this paper, a differentially private~(randomised)
learning algorithm generates similar distributions over classifiers
when trained on {\em neighbouring datasets}. Two datasets are
neighbouring when they differ in one entry. Formally, 
\begin{defn}[Approximate Differential Privacy~\citep{DMNS2006,Dwork_2006}]
    Given any two neighbouring datasets
    \(S,S^\prime\),~\(\epsilon>0\), and \(\delta\in\br{0,1}\) an algorithm \(\cA\) is called
    \(\br{\epsilon,\delta}\)-differentially private if for all sets of
    outputs \(\mathcal{Z}\), the following holds
    \[\bP\bs{\cA\br{S}\in\mathcal{Z}}\leq
    e^\epsilon\bP\bs{\cA\br{S^\prime}\in\mathcal{Z}}+\delta.\]
\end{defn}
    
    Next, we define the error of an algorithm in our problem setup.
    For a randomised learning algorithm \(\cA\), a distribution
    \(\dist\), a label prior \(\lblprior\), we can define the error of
    the algorithm as follows
\begin{defn}[Error measure on \(\Pi_{p,N}\)]\label{defn:err_dist} The
    error of the algorithm \(\cA\) trained on a dataset of size \(m\)
    from the distribution \(\Pi_{p,N}\) with respect to a label prior
    \(\lblprior\) is
    \begin{equation}
        \err{\cA,\Pi_{p,N},\lblprior} = \bE\bs{\indic{h\br{x}\neq f\br{x}}}
    \end{equation}
    where \(\indic{\cdot}\) is the indicator function and the
    expectation is over
    \(S\sim\Pi_{p,N}^m,f\sim\lblprior,\hsample{f}\), and~\(x\sim\Pi_{p,N}\).
\end{defn}

Note that the error metric with an expectation over~\(\lblprior\), was
previously used in~\citet{Feldman2019}.  In fact, for the purpose of
lower bounds on unfairness, this is a stronger notion than the worst
case \(f\in\lblset\) as, here, the lower bound is on the expectation
which is stronger than a lower bound on the worst case.  Next, we
define the  {\em accuracy discrepancy} of an algorithm, represented by
\(\Gamma\) over the distribution \(\dist\). For this purpose, for any
\(\dist\), define the marginal distribution on the group of  {\em
minority} subpopulations \(X_2\) as
\begin{equation}
    \label{eq:minority-marginal-dist}
        \enskip \dist^2\br{x}=\begin{cases} 
        \dfrac{\Pi_{p,N}\br{x}}{\sum_{x\in X_2}\Pi_{p,N}\br{x}}=\dfrac{\Pi_{p,N}\br{x}}{p} & x\in X_2 \\
        0 & x\not\in X_2
     \end{cases}
\end{equation}

\begin{defn}[Accuracy discrepancy on
    \(\Pi_{p,N}\)]\label{defn:fairness} For \(X, \lblprior,
    \Pi_{p,N},\) and \(\dist^2\) as defined above, the accuracy
    discrepancy of the algorithm \(\cA\) trained on a dataset of size
    \(m\)  on the distribution \(\Pi_{p,N}\), with respect to the
    label prior \(\lblprior\), is
    \begin{equation}\label{eq:defn-fair}
        \fair{\cA,\Pi_{p,N},\lblprior} = \err{\cA,\dist^2,\lblprior} - \err{\cA,\Pi_{p,N},\lblprior}.
    \end{equation} where \(\err{\cdot}\) is as defined in~\Cref{defn:err_dist}.
\end{defn}
This notion of group fairness is similar to the notion of subgroup
performance gap used in~\citet{goel2020} and has also been implicitly
used in multiple
works~\citet{Sagawa2020Distributionally,koh2021wilds,du2021fairness}
as discussed before. It has also been used in works related to the
privacy~\citep{bagdasaryan2019differential,Chang_2021} and fairness
literature~\citep{buolamwini2018gender,Raz_2021}.

\renewcommand{\arraystretch}{1.2}
\begin{table*}\centering
    \begin{tabular}{||c@{\quad\quad}l||}\toprule
        Notation&Description\\\midrule
       \(m\)& Size of dataset\\
        \(N\)& Number of minority subpopulations\\
        \(p\)& Probability of minority group\\
        \(k\)& Reciprocal of the probability of individual
        subpopulations in \(X_2\)\\
        \(c\)& Ratio of \(N\) and \(m\)\\
        \(\epsilon,\delta\)& Privacy parameters of Approximate
        Differential Privacy\\
        \(p_1\) & Parameter in Assumption~\ref{assump:A1}\\
        \(\cA\)& Randomised learning algorithm\\
        \(X, X_1,X_2\) & Entire data domain, majority, and minority group respectively\\
        \(\dist, \dist^2\)&Data distribution on \(X\) and marginal
        distributions on~\(X_2\) respectively \\
        \(S, \Slabel\)& m-sized unlabelled and labelled~(with \(f\in F\)) dataset respectively\\
        \(\Sell\)& All points that appear \(\ell\) times in \(S\)\\
        \(\lblset,\lblprior\) & Set of labelling functions and distribution over the set respectively\\\bottomrule
    \end{tabular}
    \caption{A table of notations frequently used in the text}
    \label{tab:notation-table}
\end{table*}
\renewcommand{\arraystretch}{1.}

\subsection{Privacy and accuracy at the cost of fairness}
\label{sec:main-theory}

The theoretical results below use the definitions and notations
described above and summarised in~\Cref{tab:notation-table}.  First,
in~\Cref{thm:thm_1}, we show that there are distributions~(within the
family of distributions defined in~\Cref{defn:comp-unif}) where any
accurate and approximately differentially private algorithm~(with
additional assumptions) is necessarily unfair.  In~\Cref{thm:thm_2},
we relax some of the stronger assumptions and present a more general
result. As discussed above, the dataset size \(m\) and the number of
minority subpopulations \(N\) both simultaneously go to \(\infty\) and
the ratio \(\frac{N}{m}\) asymptotically  approaches \(\frac{N}{m}\to
c\). Throughout this section, we also use the notation \(\errinf{\cA,
\Pi_{p,N}, \lblprior}=\lim_{m,N\to\infty} \errinf{\cA, \Pi_{p,N},
\lblprior}\) and \(\fairinf{\cA, \Pi_{p,N},
\lblprior}=\lim_{m,N\to\infty} \fairinf{\cA, \Pi_{p,N}, \lblprior}\)
to denote the asymptotic limit for the error and the accuracy
discrepancy metrics as \(m,N\to\infty\).
\begin{thm}
    \label{thm:thm_1}
For \(\epsilon\in\br{0,1.11}\) and  \(\delta\in\br{0,0.01}\), consider
any \(\br{\epsilon,\delta}\)-DP algorithm \(\cA\) that does not make mistakes on subpopulations occurring
  more than once in the dataset. Then, there exists a family of label priors \(\cF\) where for any \(\alpha\in\br{0,0.025}\), there exists \(p\in\br{0,\nicefrac{1}{2}},c>0\) such that, 
\[\errinf{\cA, \Pi_{p,N}, \lblprior} \leq
    \alpha~\quad\mathrm{and}\quad\fairinf{\cA, \Pi_{p,N}, \lblprior}\geq
    0.5.\] where \(\frac{N}{m}\to c\) as \(N,m\to\infty\).
\end{thm}

An immediate consequence of unfairness~\(\Gamma_m\) being greater than
\(0.5\), coupled with the very small error \(\alpha\), is that the
algorithm essentially behaves worse than random chance on the minority
subpopulations thereby rendering the algorithm useless for these
subpopulations. Below, we present a proof sketch and discuss the
results. A detailed version of the theorem along with its full proof
is relegated to~\Cref{app:thm_1}. 

\noindent\textbf{Proof sketch}  By~\Cref{defn:comp-unif}, the
probability mass of each majority subpopulation is~\(\Omega\br{1}\)
whereas the probability mass of each minority subpopulation
is~\(\bigO{\frac{1}{m}}\). Thus, for a large enough dataset~(i.e.
large \(m\)), we show that almost all majority subpopulations appear
more than once and consequently, the algorithm in~\Cref{thm:thm_1}
makes fewer mistakes on the majority subpopulations. As a result, the
error and the accuracy discrepancy are both caused by mistakes,
primarily, on the minority subpopulations.  We, then, count the number
of subpopulations that do not appear or appear just once among the
minority subpopulations and use that to provide the upper bounds for
error and lower bound for unfairness~(accuracy discrepancy). As these
bounds are expressed in terms of \(p\) and \(c\), the proof then
follows by showing the existence of \(p,c\) that satisfy the
inequalities in the theorem. 

While~\Cref{thm:thm_1} shows the existence of distributions under
which private and accurate algorithms are necessarily unfair,
in~\Cref{thm:thm_2}, we provide a quantitative lower bound for
unfairness of private algorithms. In addition, we also
generalise~\Cref{thm:thm_1} to include a much broader set of
algorithms. For this, we state an assumption below that we refer to as
the \emph{accuracy assumption}. For any \(\ell\in\bN\), define
\(\Sell\) to denote the set of examples that appear exactly \(\ell\)
times in \(S\).

 \paragraph{Accuracy Assumption} Given \(s_0\in\bN\) and
 \(p_1\in\br{0,1}\), we state that the algorithm \(\cA\) satisfies
 Assumption~\ref{assump:A1} with parameters \(p_1\) and \(s_0\) if for
 all datasets \(S\), for all $\ell > s_0$, and for all \(x\in \Sell\),
    \begin{equation}\label{assump:A1}\tag{A1}\bP_{f\sim \lblprior,\hsample{f}}[h(x) \neq f(x)]
    \leq p_1.\end{equation}

Assumption A1 essentially requires the algorithm to have small overall
train error.  Note that since our domain is discrete, high training
accuracy translates to high test accuracy in particular as the sample
size approaches infinity.  When \(p_1\) is small, algorithm \(\cA\)
obtains low training~(and hence test) error on frequently occurring or
{\em typical} data ~(i.e. \(\ell \ge s_0\)) where  $s_0$ can be
interpreted as the minimum frequency with which {\em typical data}
appears.

We note that the accuracy assumption in~\Cref{thm:thm_1}, that the
algorithm does not make mistakes on subpopulations that appear more
than once is an instantiation of Assumption~\ref{assump:A1} with the
parameters~\(s_0=1\) and \(p_1=0\). In the next theorem, we present a
detailed result showing how the unfairness of a DP algorithm varies
with a wide range of \(s_0,p_1\) in Assumption~\ref{assump:A1},
privacy parameters~\(\epsilon,\delta\), and distributional
parameters \(p,c\). For easier interpretation, we show a simplified version
in~\Cref{thm:thm_2} and highlight the key takeaways, and provide a
detailed version in~\Cref{app:thm_2}.  Consider $X, \lblprior$ as
defined in~\Cref{sec:prelim}.

\begin{thm}\label{thm:thm_2}

    For any \(p\in \br{0, \frac{1}{2}}\), \(c> 0\) such that
    $\frac{p}{c}\leq 1$, consider the distribution \(\Pi_{p, k, N}\)
    where \(\nicefrac{N}{m}\to c\) as \(N,m\to\infty\) and
    $k=\bigO{1}$. For any $\epsilon>0, \delta\in\br{0,\frac{1}{2}}$,
    such that
    \(\frac{1}{\epsilon},\log{\br{\frac{1}{\delta}}}=o\br{m}\),
    consider an \(\br{\epsilon,\delta}\)-DP algorithm $\cA$ that
    satisfies Assumption~\ref{assump:A1} with $s_0=
    \floor{\min\br{\frac{1}{\epsilon}\log{2},\log{\br{\frac{1}{2\delta}}}}}$
    and some \(p_1\in\br{0,1}\). Also assume that the label prior
    \(\lblprior\) is subpopulation wise independent. Then, the
    accuracy discrepancy is lower bounded as 
    \[\fairinf{\cA, \dist, \lblprior} \geq
    \frac{\br{1-\norm{\lblprior}_\infty}\br{1-p}}{3}\gamma_0\] where
    $\gamma_0$ is some constant depending on \(c, p, \epsilon\), and
    \(\delta\). Further, the error of the algorithm is upper-bounded
    by \(\errinf{\cA, \dist, \lblprior} \leq (1-p_1)p\alpha_0+p_1\)
    and \(\alpha_0\) depends on \(c,\epsilon,\delta,\) and \(p\). In the asymptotic
    limit \(c \to \infty\) and \(\epsilon\to 0\),~\(\gamma_0\)
    increases as \(1 - \bigO{\frac{\epsilon
    e^{-\nicefrac{c}{\epsilon^2}}}{\sqrt{c}}}\). As $p\to 0$,
    $\alpha_0$ increases as \(1 -
    \bigO{\sqrt{p}e^{-\nicefrac{1}{p}}}\).
\end{thm}

{\color{black} The detailed  expressions of \(\alpha_0\) and
\(\gamma_0\)~(including its dependance on \(\epsilon,\delta,c,\) and
\(p\)) can be found in~\Cref{thm:thm_2_detailed} in~\Cref{app:thm_2}.
We now briefly discuss how the theorem characterizes the effect of
privacy and accuracy~(via \(\epsilon\) and \(s_0\) respectively) of
the algorithm and the distribution and number of of subpopulations
(via $p,c$ respectively) on the accuracy discrepancy.
~\Cref{thm:thm_2} indicates that as \(c\) increases and \(\epsilon\)
decreases, $\gamma_0$ and hence unfairness increases. Intuitively,
this is because minority subpopulations appear infrequently and the
algorithm is more likely to be incorrect on infrequent subpopulations.
Therefore, the lower bound on unfairness increases with \(c\) and
decreases with \(\epsilon\). Further, as \(\norm{\lblprior}_\infty\)
decreases, which can be interpreted as an increase in the inherent
uncertainty of the labelling function, the unfairness
increases.\footnote{While the discussion here assumes the asymptotic
limit for \(c,\) and \(p\) as \(m\to\infty\), our results
in~\Cref{app:thm_2} shows non-asymptotic dependence on these terms.} }

\paragraph{Proof Sketch} Next, we provide a proof sketch
of~\Cref{thm:thm_2} and relegate the full proof to~\Cref{app:thm_2}.
First, note that when the ratio \(c\) of the number of minority
subpopulations with respect to the sample size is large, most minority
subpopulations appear infrequently in the observed dataset.
Second, a key component of the proof is showing that when \(\epsilon\)
and \(\norm{\cF}_\infty\) is relatively small, most infrequently
observed subpopulations are misclassified with a non-trivial
probability. This is formalised
by~\Cref{lem:large_error_infreq_sample} below which shows that
subpopulations with frequency less than
\(\frac{\log{2}}{\epsilon}\)~(for sufficiently small \(\delta\)) is
misclassified with probability greater than
\(\frac{\br{1-\norm{\lblprior}_\infty}}{3}\).

\begin{restatable}{lem}{infreqLemma}\label{lem:large_error_infreq_sample}
       Let \(\cF\) be any subpopulation-wise independent label prior over
    all labelling functions and \(S\) be any dataset. For any
    \(\br{\epsilon,\delta}\)-differentially private algorithm~\(\cA\),
    and for all subpopulations \(x \in X\) that appear fewer than
    \(s_0 =
    \floor{\min\br{\frac{1}{\epsilon}\log{2},\log{\br{\frac{1}{2\delta}}}}}\)
    times in the dataset \(S\), we have that
    \[\ \bP_{f\sim\cF,\hsample{f}}\bs{h\br{x}\neq f\br{x}}>
    \frac{\br{1-\norm{\lblprior}_\infty}}{3}.\] A detailed version
    of this lemma along with its proof is presented in~\Cref{lem:large_error_infreq_sample_app_v2} in the
    Appendix.   
\end{restatable}
\begin{figure*}[t]\centering
    \begin{subfigure}[c]{0.8\linewidth}
    \def\svgwidth{0.99\columnwidth}
    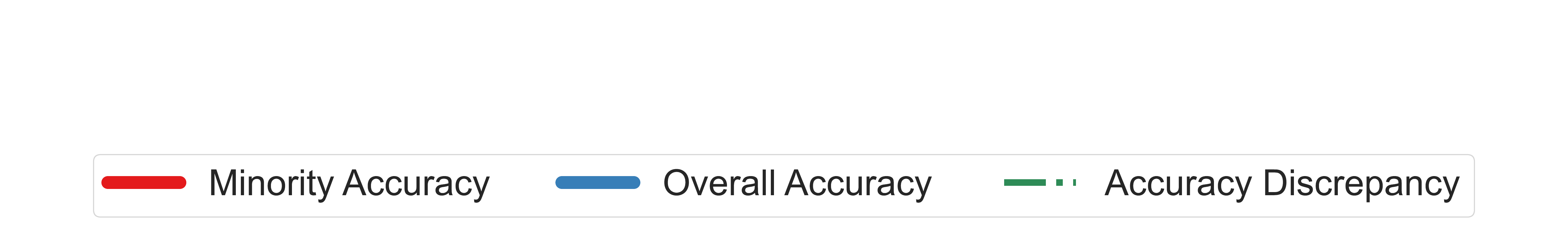%
\end{subfigure}\\
    \begin{subfigure}[c]{0.9\linewidth}
    \def\svgwidth{0.99\columnwidth}
    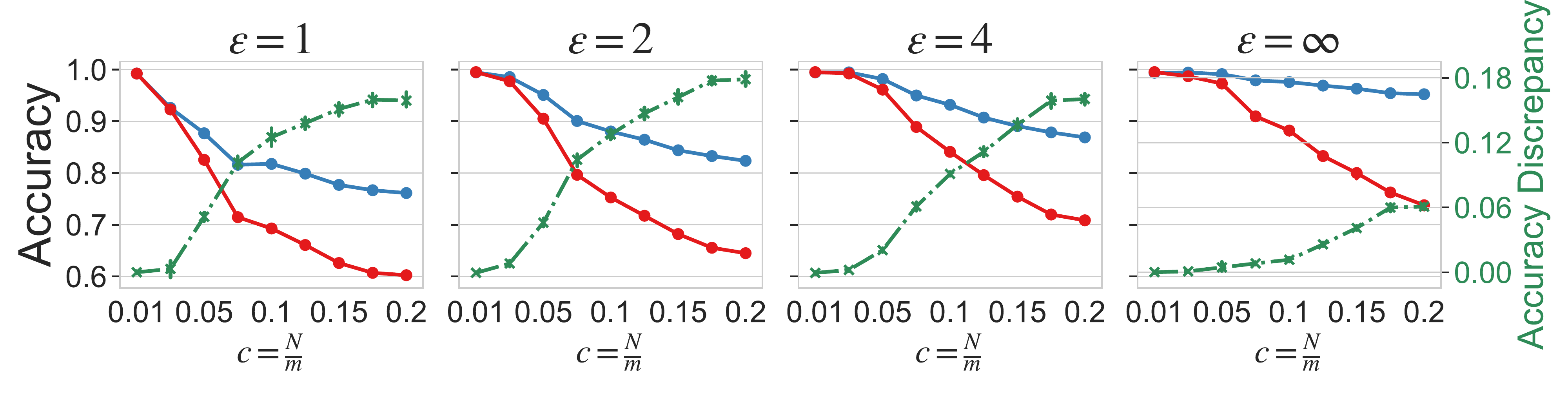%
\end{subfigure}
    \caption{%
    Each figure plots the accuracy discrepancy~(\(\Gamma\);
    higher is less fair) in {\color{OliveGreen} green dashed line}, the
    accuracy of the minority group with {\color{red} red}, and the
    overall accuracy with {\color{NavyBlue} blue} on the y-axis and the
    parameter \(c\) in the X-axis. The left most~\(\br{\epsilon=1}\)
    achieves the strictest level of privacy and the right
    most~\(\br{\epsilon=\infty}\) is vanilla training without any
    privacy constraints. The two figures in between achieve intermediate
    levels of privacy. Here \(p=0.2\). Experiment for \(p=0.5\) is
    in~\Cref{app:more-p}}
    \label{fig:incr-C-hurts-fairness}
\end{figure*}
\noindent Combining this with Assumption~\ref{assump:A1} and using standard
inequalities in probability theory gives
us~\cref{thm:thm_2}.

\paragraph{An illustrative example} We now provide an example of a
simple learning problem and a differentially private learning
algorithm for the problem that satisfies Assumption~\ref{assump:A1}.
The sole purpose of this is to help explain the concepts in the
preceding sections with a simple constructive example. In particular,
consider a binary classification problem on the domain \(X\). Let
\(\cA_\eta\) be an algorithm that accepts an \(m\)-sized dataset
\(\Slabel\in\br{X\times\bc{0,1}}^m\) and a noise rate
\(\eta\in\br{0,\frac{1}{2}}\) as input and outputs a dictionary
matching every subpopulation in \(X\) to a label in \(\bc{0,1}\).  The
algorithm first creates a dictionary where the set \(X\) is the set of
keys.  In order to assign values to every key, it first randomly flips
the label of every element in \(\Slabel\) with probability \(\eta\),
then for every unique key in \(\Slabel\), the algorithm computes the
majority label of that key in the flipped dataset and assigns that
majority label to the corresponding key. For elements in \(X\) not
present in \(\Slabel\), it assigns a random label.
~\Cref{lem:noisy-majority} provides privacy and accuracy guarantees
for this algorithm.

\begin{restatable}{lem}{noisyMajority}
\label{lem:noisy-majority}
The algorithm \(A_\eta\) is
\(\br{\log{\br{\frac{1-\eta}{\eta}}},0}\)-DP. Further, for any dataset
\(\Slabel\) and \(s_0\in\bN\),
\begin{itemize}
    \item if a subpopulation \(x\) appears more than \(s_0\) times in \(S\), \(\bP_{h\sim\cA_{\eta}\br{\Slabel}}\bs{h(x)\neq f(x)}\leq
e^{-\nicefrac{s_0\br{1-2\eta}^2}{8\br{1-\eta}}}\) and \\
\item if a subpopulation \(x\) appears less than \(s_0\) times in \(S\),  \(\bP_{h\sim\cA_{\eta}\br{\Slabel}}\bs{h(x)\neq f(x)}\geq \frac{1}{\sqrt{2s_0}}{\br{4\eta\br{1-\eta}}}^{\nicefrac{s_0}{2}}\).
\end{itemize}
Equivalently, algorithm \(A_\eta\) satisfies Assumption~\ref{assump:A1} with \(p_1=e^{-\nicefrac{s_0\br{1-2\eta}^2}{8\br{1-\eta}}}\).%
\end{restatable}

{\color{black} ~\Cref{lem:noisy-majority} shows that for all \(\epsilon
> 0\), we can find an \(\eta=\frac{1}{1+e^{\epsilon}}\) such that
\(\cA_\eta\) is $(\epsilon, 0)$-differentially private. For example,
\(\eta\approx 0.475\) provides \(\br{0.1,0}\) differential privacy.
Further, this algorithm is more accurate on frequently occurring
subpopulations and inaccurate on rare subpopulations.  For example,
with \(\eta=0.475\),~\(\bP_{h\sim\cA_{\eta}\br{\Slabel}}\bs{h(x)\neq
f(x)}\geq 0.35\)\footnote{This is possibly a loose lower bound.
Running a simulation shows that the lower bound is at least \(0.17\)
which matches the lower bound
from~\Cref{lem:large_error_infreq_sample} with
\(\norm{\lblprior}_{\infty}=0.5\)(equal probability for both classes).} for subpopulations occurring less
than \(5\) times.  Similarly, for subpopulations occurring more than
\(6000\) times, ~\(\bP_{h\sim\cA_{\eta}\br{\Slabel}}\bs{h(x)\neq
f(x)}\leq 0.05\). If we assume that the data is distributed according
to the long-tailed distribution defined in~\ref{defn:comp-unif} with
\(N=10^5,~k=10,p=0.2\) and a dataset of size \(m=10^5\) is drawn then
almost all subpopulations either occur more than \(6000\) times or
less than \(5\) times.}

\subsection{Privacy and fairness at the cost of accuracy}\label{sec:strict-priv}

So far we have shown that under strict privacy and high average accuracy requirements on the algorithm, fairness necessarily suffers. 
A natural question to ask is whether it is possible to sacrifice
accuracy for fairness. We present a simplified theorem statement here for easier interpretation and prove
a more precise version in~\Cref{app:thm_3} along with a discussion.
In words, the theorem states that for very strict privacy parameters, fairness can be achieved at the cost of accuracy.

{\color{black}
\begin{thm}\label{thm:thm_3} For any $p\in (0,
    \nicefrac{1}{2}), c> 0$ such that $p/c\leq 1$, consider the
    distribution \(\Pi_{p,N}\) where \(N\) is the number of minority
    subpopulations. For any $\alpha >
    0$, \(\epsilon = \frac{\log{2}}{\alpha\sqrt{m} +
    \br{\frac{2-3p}{2k\br{1-p}}}m}\) and
    \(\delta<2^{-\frac{1}{\epsilon}-1}\) consider any
    \(\br{\epsilon,\delta}\)-DP algorithm $\cA$. Further, let \(\nicefrac{N}{m}\to c\)~as
    \(m,N\to\infty\). Then,

    \[\errinf{\cA, \dist, \lblprior} \geq \br{\frac{1-\norm{\lblprior}_\infty}{3}}\br{1-\br{1-p}\br{e^{-c_1\alpha^2}}}\]  
        \noindent and
        \[\fairinf{\cA, \dist, \lblprior} \leq \br{1-p}\bs{1 - \frac{\br{1-\norm{\lblprior}_\infty}}{3}\br{1-e^{-c_1\alpha^2}}}\]
        where \(c_1=\frac{4(2-p)}{(2-3p)^2}\).
    \end{thm}

A detailed version of the theorem with the full proof can be found
in~\Cref{thm:thm_3_detailed} in~\Cref{app:thm_3}. We now briefly
discuss how the theorem characterises the effect of privacy and
accuracy~(via~\(\alpha\)) on fairness. When \(\alpha\) is large,
\(\epsilon\) is small, therefore making the algorithm more private.
From~\Cref{lem:large_error_infreq_sample_app_v2}, we know that a small
\(\epsilon\) renders the algorithm inaccurate on frequent
subpopulations~(i.e. majority subpopulations). Thus, \(\alpha\),
essentially, characterises what fraction of majority subpopulations
the algorithm is incorrect on.~First,~\Cref{thm:thm_3} shows that when
\(\alpha\) is large, the overall error increases.  Intuitively, this
is because, with increasing \(\alpha\), the algorithm is incorrect,
not only on  minority subpopulations, but also on {\em majority
subpopulations}. Second,~\Cref{thm:thm_3} shows that as \(\alpha\)
increases, i.e. privacy increases, the unfairness decreases. This
shows that by making the algorithm more private~(by increasing
\(\alpha\)), fairness can be achieved but at the cost of overall
accuracy.  }
\vspace{-10pt}

\section{Experimental results}
\label{sec:exp}

In this section, we conduct experiments to support our theoretical
results from~\Cref{sec:theory}. We note that our theoretical
results are model-agnostic and to demonstrate the universality of our
result, we conduct a broad set of experiments on both
synthetic~(in~\Cref{sec:synth}), and real world
datasets~(in~\Cref{sec:vision,sec:tab-data}), using multiple machine
learning models including deep neural networks and random forests. 

\vspace{-5pt}
\subsection{Synthetic experiments}
\label{sec:synth}

First, we look at a synthetic data distribution that closely emulates
the data distribution we use in our theoretical results
in~\Cref{thm:thm_1,thm:thm_2}.  Given \(N,k\in\bN,c\in\reals_+\), and \(p\in\br{0,0.5}\),  we construct a continuous version of the long-tailed
distribution~\(\Pi_{p,N,k}\)~(\Cref{defn:comp-unif}) on a domain
\(X\).
First of all, since the domain \(X\) is discrete, we can place each
element on a vertex of a \(\bigO{\log\br{N}}\)-dimensional hypercube.
 The continuous distribution we use in our experiments is a 
mixture of Gaussians where each Gaussian is centered around the vertices of the hypercube.
\begin{figure}[t]\centering
  \begin{subfigure}[c]{0.35\linewidth}
  \centering
  \def\svgwidth{0.99\columnwidth}
  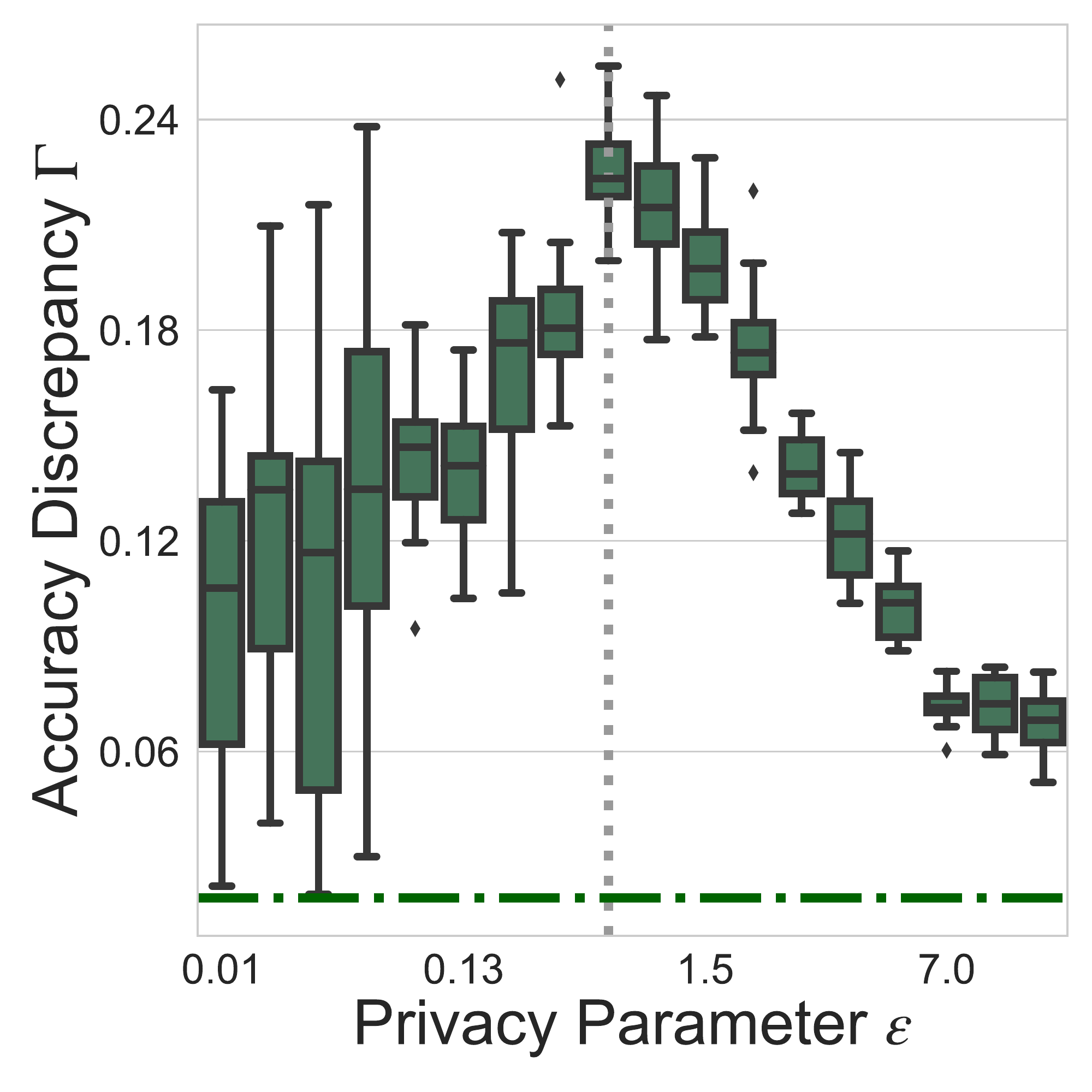%
  \end{subfigure}
  \begin{subfigure}[c]{0.35\linewidth}
      \centering
      \def\svgwidth{0.99\columnwidth}
      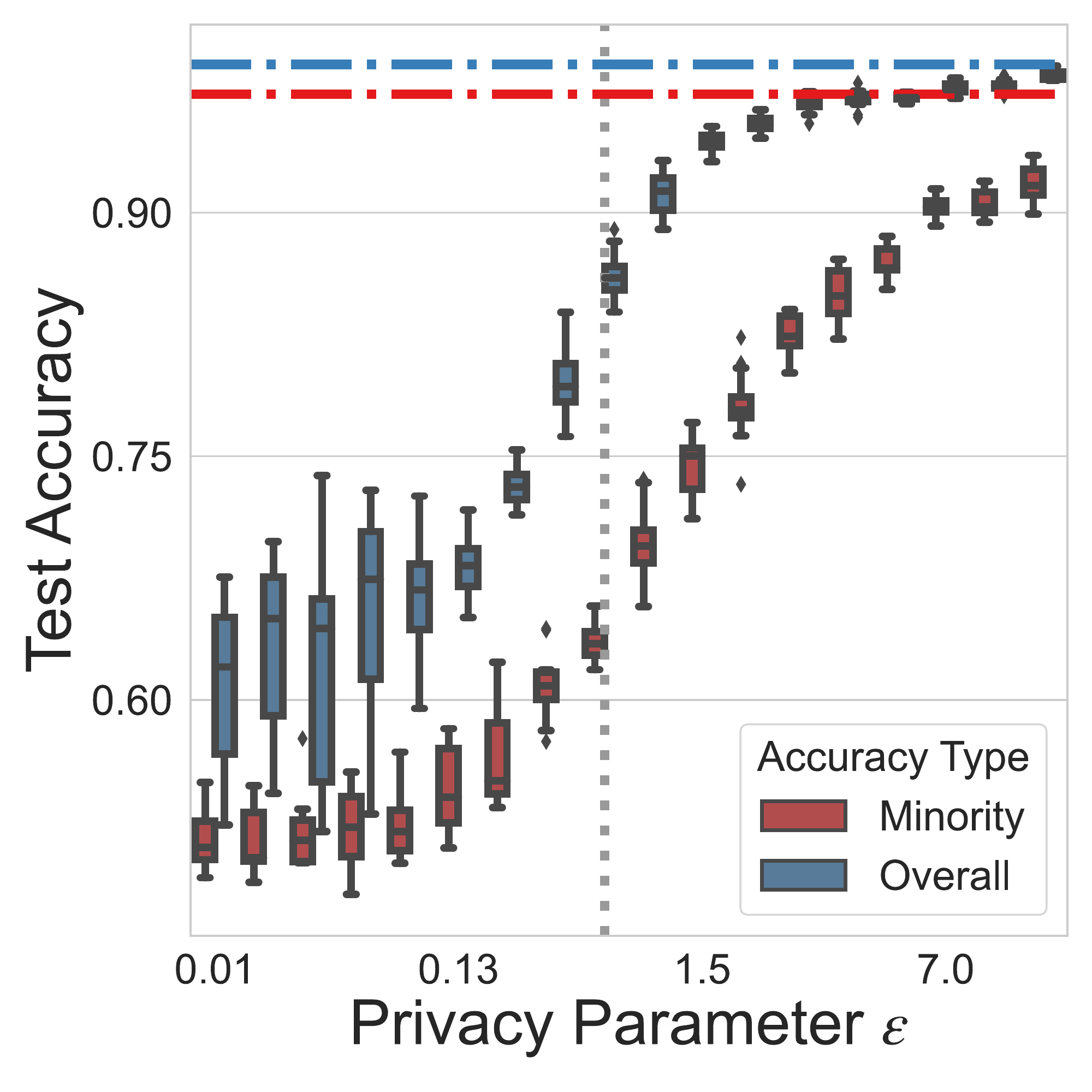%
      \end{subfigure}
\caption{\textbf{Left:} {\color{OliveGreen} Accuracy discrepancy
(green)} where the box plots reflect the variance when run several
times. \textbf{Right:} Overall {\color{NavyBlue} (blue boxes)} and
minority {\color{red} (red boxes)} accuracies for varying
\(\epsilon\). The horizontal dashed line of different colors show the
respective metrics for vanilla training without privacy constraints.
The gray vertical dashed line marks the privacy parameter for which
significant~(\(\geq 80\%\)) overall test accuracy is achieved.
}
\label{fig:synth-fairn-acc-with-eps}
\end{figure}
In the experiments, we choose \(k=64,m=10^4,\) vary the ratio \(c\)
from \(0.01\) to \(0.2\), set the number of minority subpopulations to
\(N = mc\), and choose \(p\in\bc{0.2,0.5}\). We train a five-layer
fully connected neural network with ReLU activations using
DP-SGD~\citep{Abadi_2016} for
varying levels of \(\epsilon\) while setting
\(\delta=10^{-3}\). We refer
to~\Cref{app:exp-details} for a more detailed description of the data
distribution and the training algorithm.

\noindent\textbf{Unfairness aggravates with increasing number of minority subpopulations}
As discussed in~\Cref{sec:prelim}, increasing the number of
subpopulations compared to the number of samples via \(c\)
decreases accuracy on the minority subpopulations while
the majority subpopulations remain unaffected.
\Cref{fig:incr-C-hurts-fairness} shows how increasing \(c\)
hurts fairness since the accuracy discrepancy (green dashed line)
increases, most pronounced for small values $\epsilon$ (i.e. more private algorithms).
This corroborates our theoretical results
from~\Cref{thm:thm_2} regarding the dependence of accuracy discrepancy
on \(c\).
We further observe that the
increase in unfairness is almost entirely due to the drop in the
minority accuracy (red solid) whereas the overall accuracy (blue) stays relatively constant.
This highlights our claim that, in the presence of strong privacy, fairness can be poor even when overall error is low.

\noindent\textbf{Privacy constraints hurt fairness for accurate models}
In this section, we analyse the dependence of fairness on the privacy
parameter \(\epsilon\) for a
fixed \(c\). In~\Cref{fig:synth-fairn-acc-with-eps}~(left), we plot
the disparate accuracies~\(\Gamma\) for varying privacy
parameter~\(\epsilon\) and~\Cref{fig:synth-fairn-acc-with-eps}~(right)
depicts the minority and overall accuracy as a function
of \(\epsilon\).

There are two distinct phases in the development of
the accuracy discrepancy with increasing~\(\epsilon\) separated by the
gray dashed line:
For a very small \(\epsilon\), the learned classifier is essentially a
trivial classifier as evidenced by the very low overall
accuracy~\(\br{\approx 60\%}\). This is a trivial way of achieving
fairness without learning an accurate classifier and is explained
by~\Cref{thm:thm_3} in our theoretical section. As the privacy
restrictions are relaxed, the classifier becomes more accurate and
less fair in the first phase.

The interesting regime is when classifier obtains decent overall
accuracy~(\(\approx 80\%\)) and is marked by the vertical gray dashed
line. In the region to the right of the vertical dashed line,
assumption~\(\frac{1}{\epsilon}=o\br{m}\) is fulfilled and
~\Cref{fig:synth-fairn-acc-with-eps}~(left) reflects the behavior as
predicted in ~\Cref{thm:thm_2}: {\em loosening privacy increases
fairness or smaller \(\epsilon\) implies larger accuracy discrepancy}.
\vspace{-10pt}
\subsection{Experiments on vision datasets}
\label{sec:vision}

\begin{figure}[t]\centering
  \begin{subfigure}[c]{0.35\linewidth}
  \centering
  \def\svgwidth{0.99\columnwidth}
  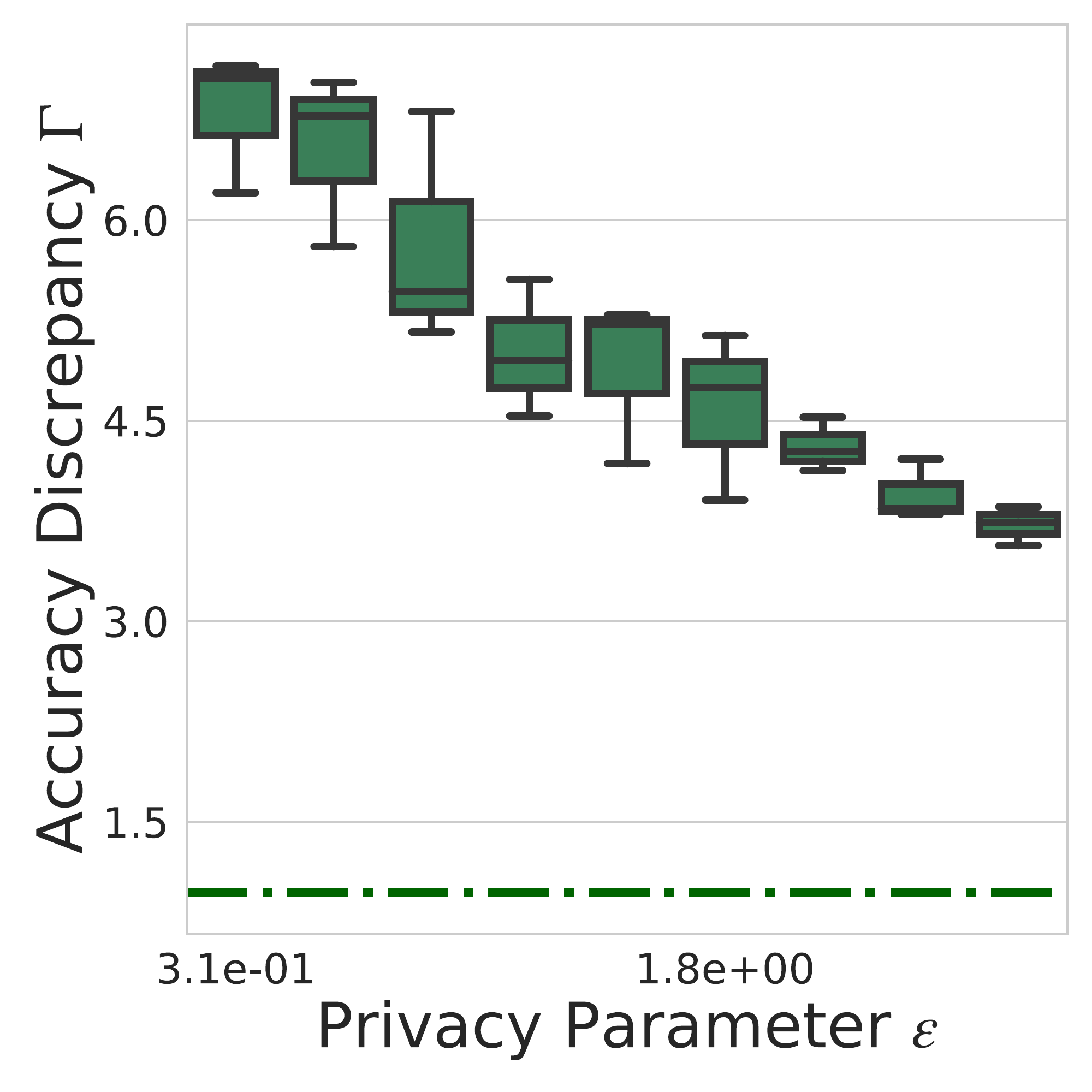%
  \end{subfigure}
  \begin{subfigure}[c]{0.35\linewidth}
      \centering
      \def\svgwidth{0.99\columnwidth}
      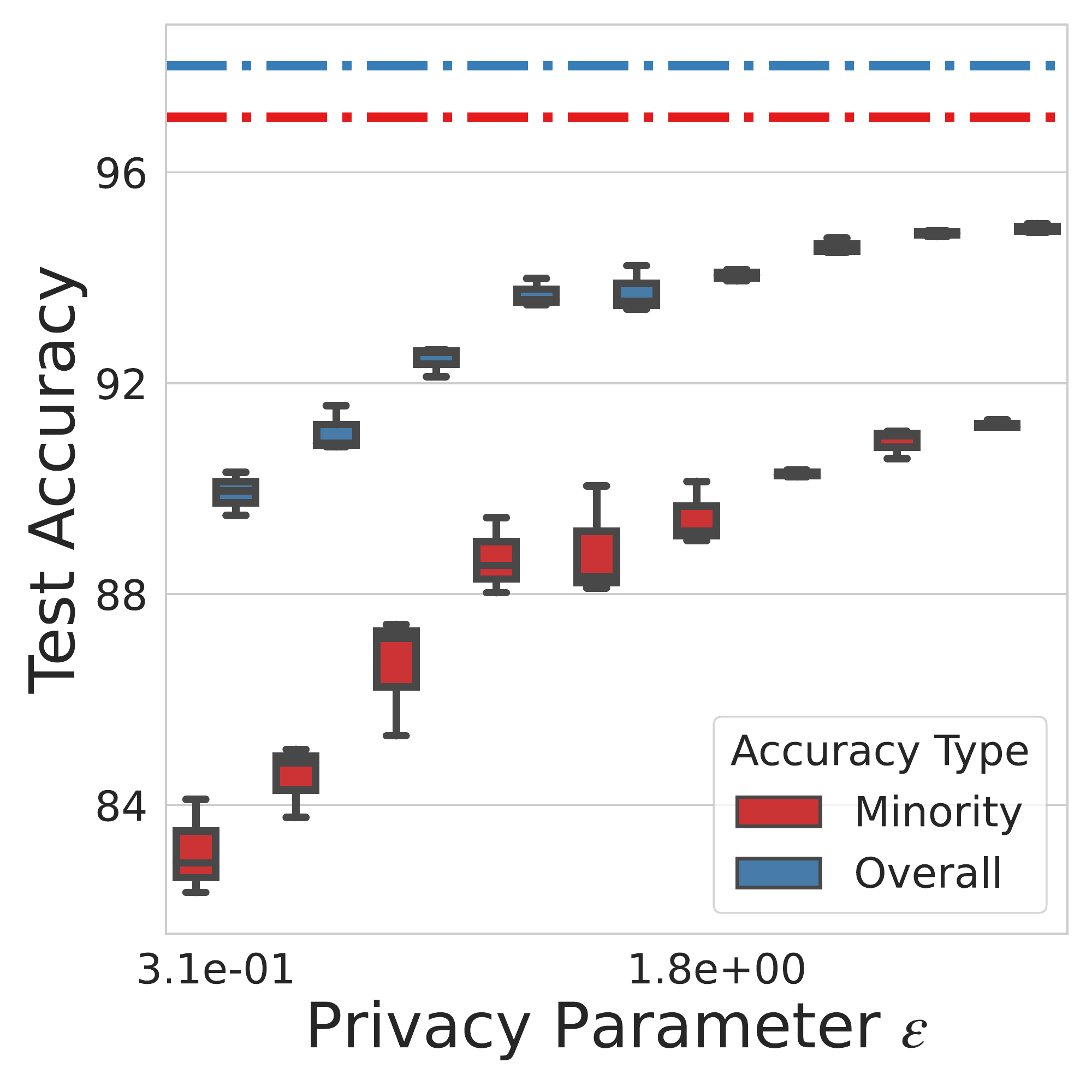%
      \end{subfigure}
          \caption{CelebA: \textbf{Left:} {\color{OliveGreen} Accuracy
          discrepancy (green)} where the box plots reflect the
          variance when run several times. \textbf{Right:} Overall
          {\color{NavyBlue} (blue boxes)} and minority {\color{red}
          (red boxes)} accuracies for varying \(\epsilon\). The
          horizontal dashed line of different colors show the
          respective metrics for vanilla training without privacy
          constraints.
          }
          \label{fig:celeb-fairness-acc}
\end{figure}

In this section, we show that our claims resulting
from~\Cref{thm:thm_2,thm:thm_3} %
do not only hold in
synthetic settings but can also be observed in real-world computer
vision datasets. In particular, we conduct experiments on two
popular computer vision datasets ---
CelebA~ and
\cifar. CelebA is a dataset of
approximately \(160k\) training images of dimension \( 178\times 218\)
and another \(20k\) of the same dimension for testing. \cifar is a
10-class classification dataset where there are \(50k\) training
images and \(10k\) test images of dimension \(3\times 32\times 32\).
For \cifar, we use a ResNet-\(18\) and for CelebA,
we use a ResNet-\(50\) architecture.

\subsubsection{Minority and majority subpopulations}
In practice, datasets like CelebA and \cifar often do not come with a
label of what constitutes a subpopulation. In this section, we
describe how we define the minority and majority subpopulations for
\cifar and CelebA.

\noindent\textbf{CelebA}~The CelebA dataset provides \(40\) attributes
for each image including characteristics like gender,
hair color, facial hair etc. We create a binary classification problem
by using the gender attribute as the target label. In addition, we use
\(11\) of the remaining \(39\) binary attributes to create \(2^{11}\)
subpopulations and categorise each example into one of these
\(2^{11}\) subpopulations. Then, we create various
groups of minority subpopulations
by aggregating the samples of all the unique
subpopulations that appear less than \(s\in\bc{5,10,20,40,60,80,100}\) times in the test set.
The remaining examples constitute the majority group. In this section, we run experiments
using \(s=40\). We report results for the other values of $s$
in~\Cref{app:more-p-celebA}.

\noindent\textbf{\cifar}~Unlike the synthetic distribution and CelebA as described above,
\cifar cannot be readily grouped into subpopulations using explicit attributes.
However, recent works~\citep{Zhang2020,Sanyal2020Benign} have shown
the presence of subpopulations in CIFAR-10 in the context of influence
functions and adversarial training respectively. We use the influence
score estimates from~\citet{Zhang2020} to create the minority and
majority subpopulations. Intuitively, we treat examples that are
atypical i.e.  unlike any other examples in the dataset as minority
examples belonging to minority subpopulations; and examples that are
{\em typical} i.e. similar to a significant number of other examples
in the dataset as examples belonging to majority
subpopulations.%

To define these subpopulations, first, we sort the examples in the
training set according to their self-influence~\citep{Zhang2020}. We
define all of those that surpass a threshold $\rho$ as minority
populations. In order to find the samples belonging to each
subpopulation $x$ in the test set, we search for images that are
heavily influenced~(influence score is greater than the threshold) by
at least one of the samples in $x$ in the training set.
In this section, we report results with \(\rho=0.1\). Other values of \(\rho\)
show a similar trend and we plot results using \(\rho=0.01\)
in~\Cref{app:more-p-CIFAR}.

\begin{figure}[t]\centering
    \begin{subfigure}[c]{0.35\linewidth}
      \centering
      \def\svgwidth{0.99\columnwidth}
      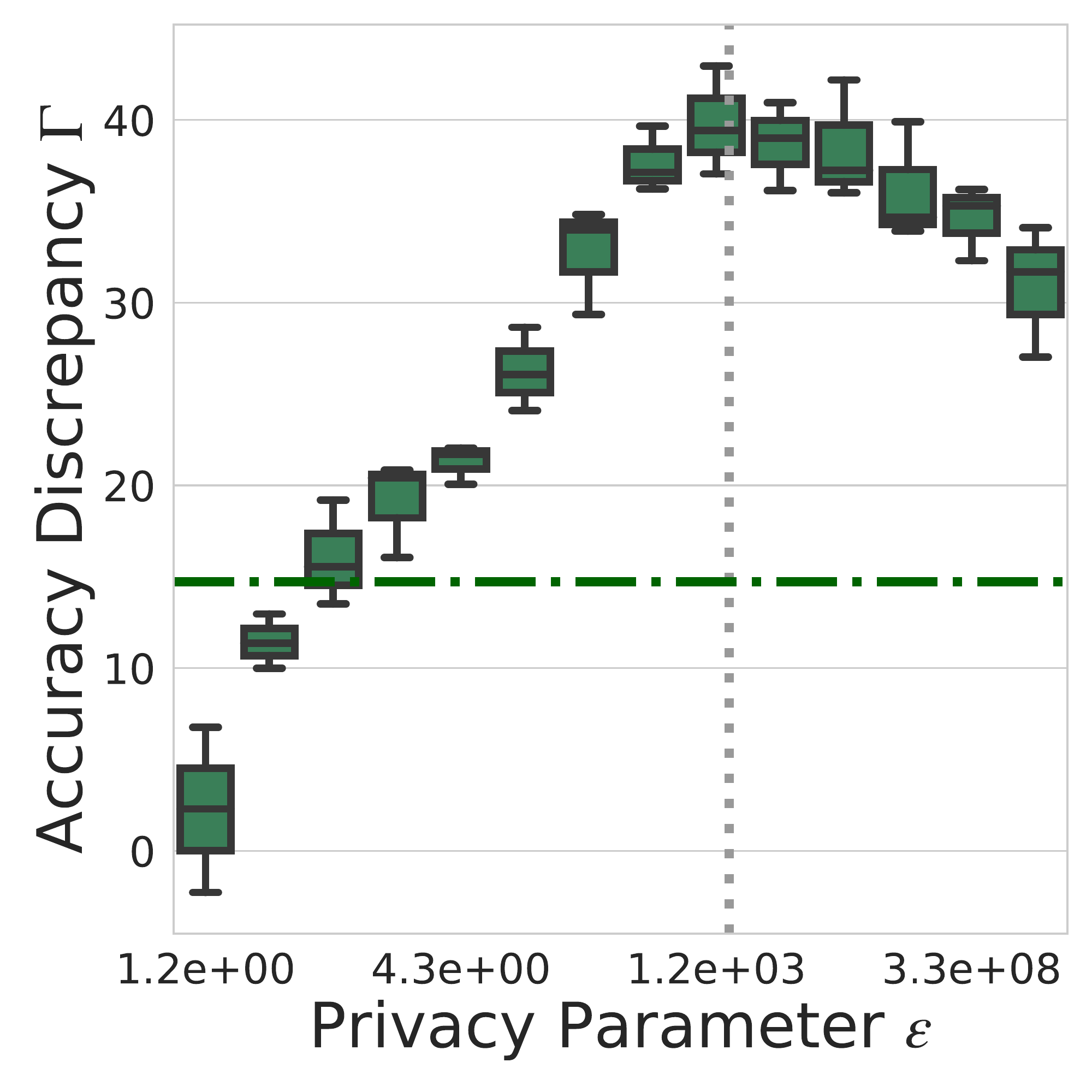%
      \end{subfigure}
      \begin{subfigure}[c]{0.35\linewidth}
          \centering
          \def\svgwidth{0.99\columnwidth}
          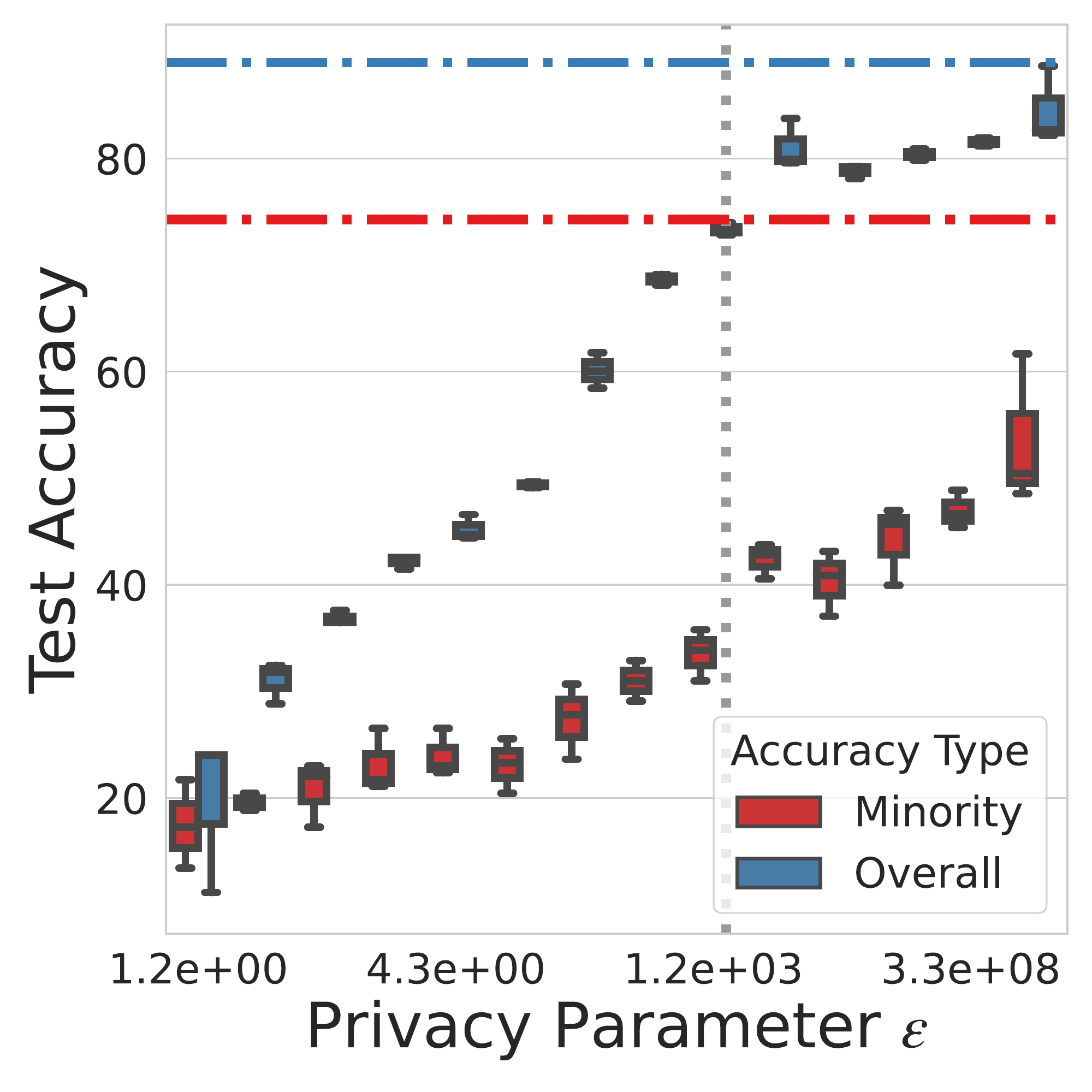%
          \end{subfigure}
      \caption{\cifar\textbf{Left:} {\color{OliveGreen} Accuracy
      discrepancy (green)} where the box plots reflect the variance when
      run several times. \textbf{Right:} Overall {\color{NavyBlue} (blue
      boxes)} and minority {\color{red} (red boxes)} accuracies for
      varying \(\epsilon\). The horizontal dashed line of different
      colors show the respective metrics for vanilla training without
      privacy constraints. The vertical dashed line marks the
      \(\epsilon\) for which significant~(\(\geq 75\%\)) overall test
      accuracy is achieved.}
      \label{fig:cifar10-fair-acc}
  \end{figure}

\subsubsection{Privacy leads to worse fairness for accurate models}

In this section, we use the above  definitions of minority and
majority groups to measure the impact of privacy on
fairness~(using~\Cref{defn:fairness}). Like~\Cref{sec:synth}, we
measure both the accuracy discrepancy and the individual minority and
overall accuracies.

\noindent\textbf{CelebA} ~\Cref{fig:celeb-fairness-acc} plots the change in
accuracy discrepancy, with respect to the \(\epsilon\) parameter of
differential privacy~(smaller \(\epsilon\) indicates stricter
privacy).~\Cref{fig:celeb-fairness-acc}~(left) shows that smaller
\(\epsilon\), with high
accuracy~(see~\Cref{fig:celeb-fairness-acc}~(right)) implies a larger
accuracy discrepancy. This aligns with our theoretical results
from~\Cref{sec:theory}. Note that
unlike~\Cref{fig:synth-fairn-acc-with-eps}~(left), the accuracy
discrepancy here monotonically decreases with increasing \(\epsilon\)
without exhibiting a two-phase
behavior.~\Cref{fig:celeb-fairness-acc}~(right) shows that, the reason
why we do not observe the two-phase behavior is that throughout the
range of observed \(\epsilon\), %
we are in the regime of high accuracy.

\noindent\textbf{\cifar}
\Cref{fig:cifar10-fair-acc}~(left) plots the change of accuracy
discrepancy \(\Gamma\) with respect to the privacy parameter
\(\epsilon\). Interestingly, the results here  exactly mimic those
from the synthetic experiments in~\Cref{fig:synth-fairn-acc-with-eps},
which are based on our theoretical setting. This indicates that our
theoretical setting is indeed relevant for real world observations.
Similar to the synthetic experiments, we observe two distinct phases
in how the accuracy discrepancy changes with \(\epsilon\). 

For small values of \(\epsilon\),~\Cref{fig:cifar10-fair-acc}~(right)
shows that the learned classifier is highly inaccurate. As discussed
in~\Cref{thm:thm_3}, this is a trivial way to achieve fairness and
this is reflected in~\Cref{fig:cifar10-fair-acc}~(left).
However, if we restrict ourselves to classifiers with high average
accuracy, marked by the area to the right of the vertical gray dashed
line,~\Cref{fig:cifar10-fair-acc}~(left) shows that accuracy
discrepancy increases with decreasing \(\epsilon\). This corresponds
to the setting in~\Cref{thm:thm_2}.

\subsection{Experiments on tabular data}
\label{sec:tab-data}
\looseness=-1
To show that our observations hold across a wider range of publicly
used datasets, we next conduct similar experiments using tabular data.
We run our experiments on the the Law school
dataset~\citep{wightman1998lsac} that has previously been used in
fairness-awareness studies like~\citet{Quy2021}. It is a binary
classification dataset with \(21k\) data points and 12 dimensional
features. Out of the 12 attributes, two binary attributes are used to
obtain the minority group as defined in~\citet{Quy2021}.

In contrast to previous experiments, we use
random forest model from~\citet{fletcher2017differentially} instead of
neural networks as in~\Cref{sec:synth,sec:vision}. For our
implementation, we use the publicly available code
in~\citet{diffprivlib} with \(10\) trees and of a maximum depth
\(50\). The results are plotted in~\Cref{fig:fair-acc-law} and they
show a similar two phase behavior as in our previous experiments with \cifar. 

\begin{figure}[t]\centering
    \begin{subfigure}[c]{0.35\linewidth}
    \centering
    \def\svgwidth{0.99\columnwidth}
    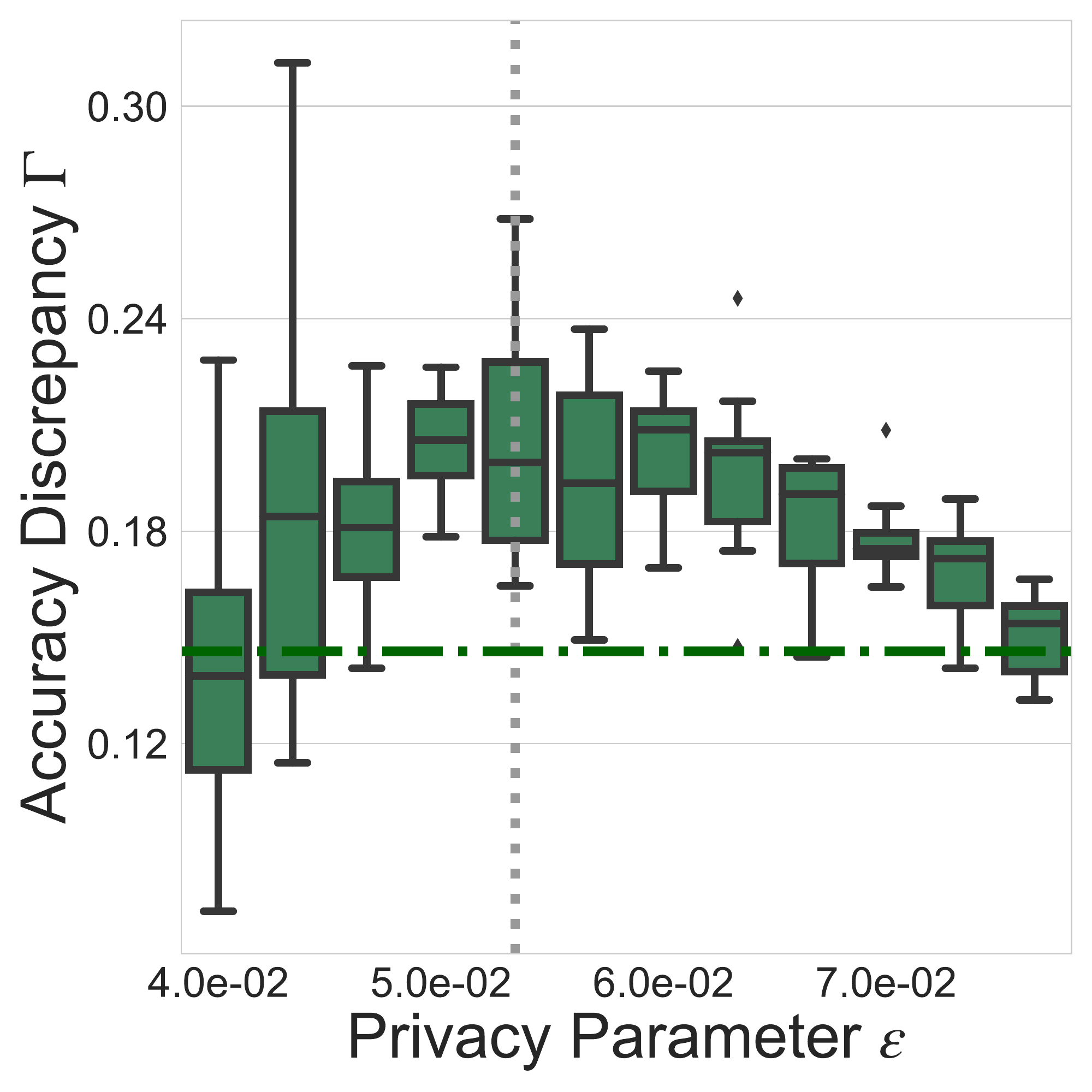%
    \end{subfigure}
    \begin{subfigure}[c]{0.35\linewidth}
        \centering
        \def\svgwidth{0.99\columnwidth}
        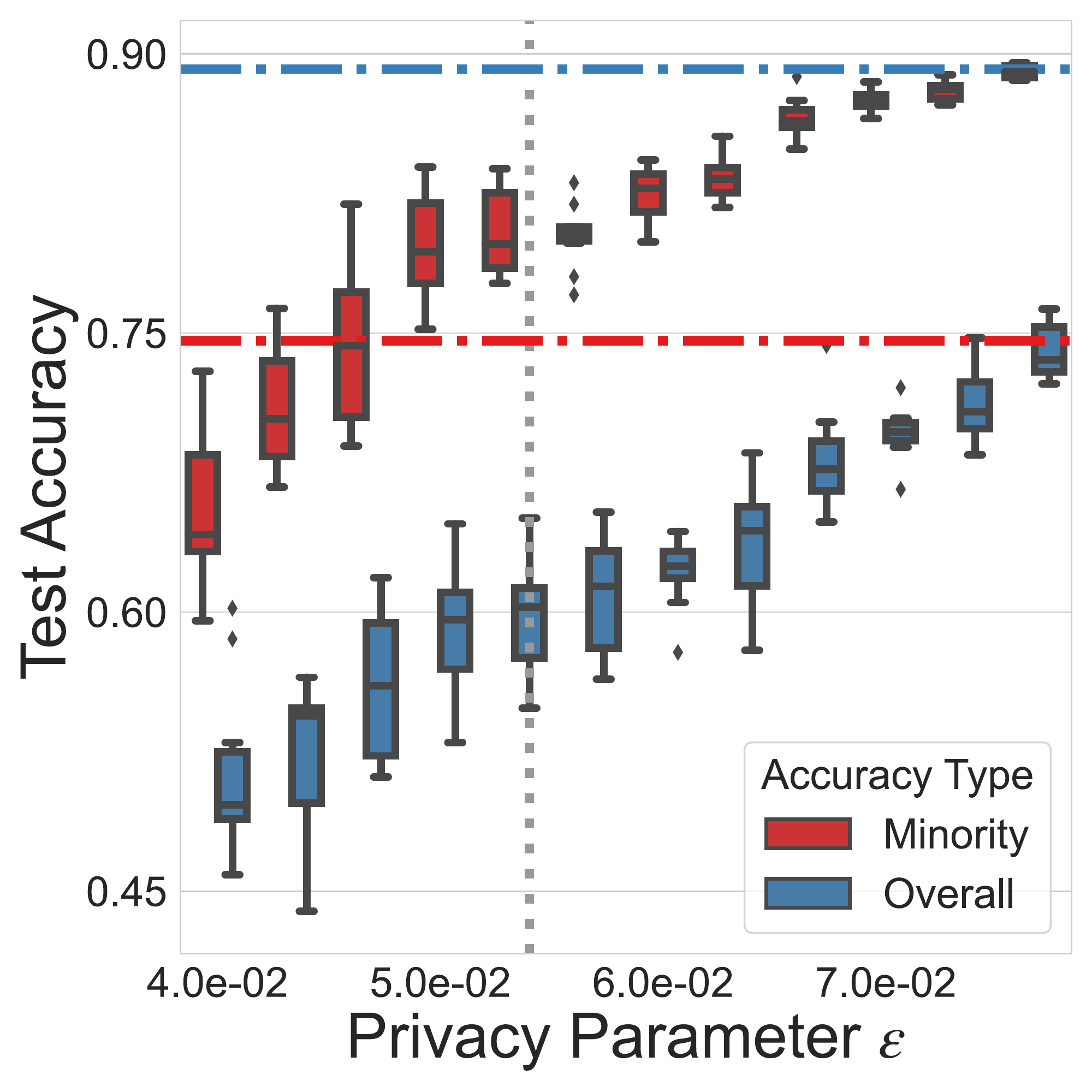%
        \end{subfigure}
      \caption{Law School---\textbf{Left:} {\color{OliveGreen} Accuracy
      discrepancy (green)} where the box plots reflect the variance when
      run several times. \textbf{Right:} Overall {\color{NavyBlue} (blue
      boxes)} and minority {\color{red} (red boxes)} accuracies for
      varying \(\epsilon\). The horizontal dashed line  show the respective metrics for vanilla training without
      privacy constraints. The vertical dashed line marks the
      \(\epsilon\) for which the test accuracy is largr~(\(\geq 80\%\)).}
      \label{fig:fair-acc-law}
      \vspace{-10pt}
  \end{figure}

Thus, all our experiments provide empirical evidence in support of the
theoretical arguments in~\Cref{sec:theory}. The behavior is
consistent across multiple kinds of datasets, machine learning models,
and learning algorithms.

\section{Future work}
\label{sec:future}
\begin{figure}[t]\centering
    \begin{subfigure}[c]{0.48\linewidth}
    \centering
    \def\svgwidth{0.99\columnwidth}
    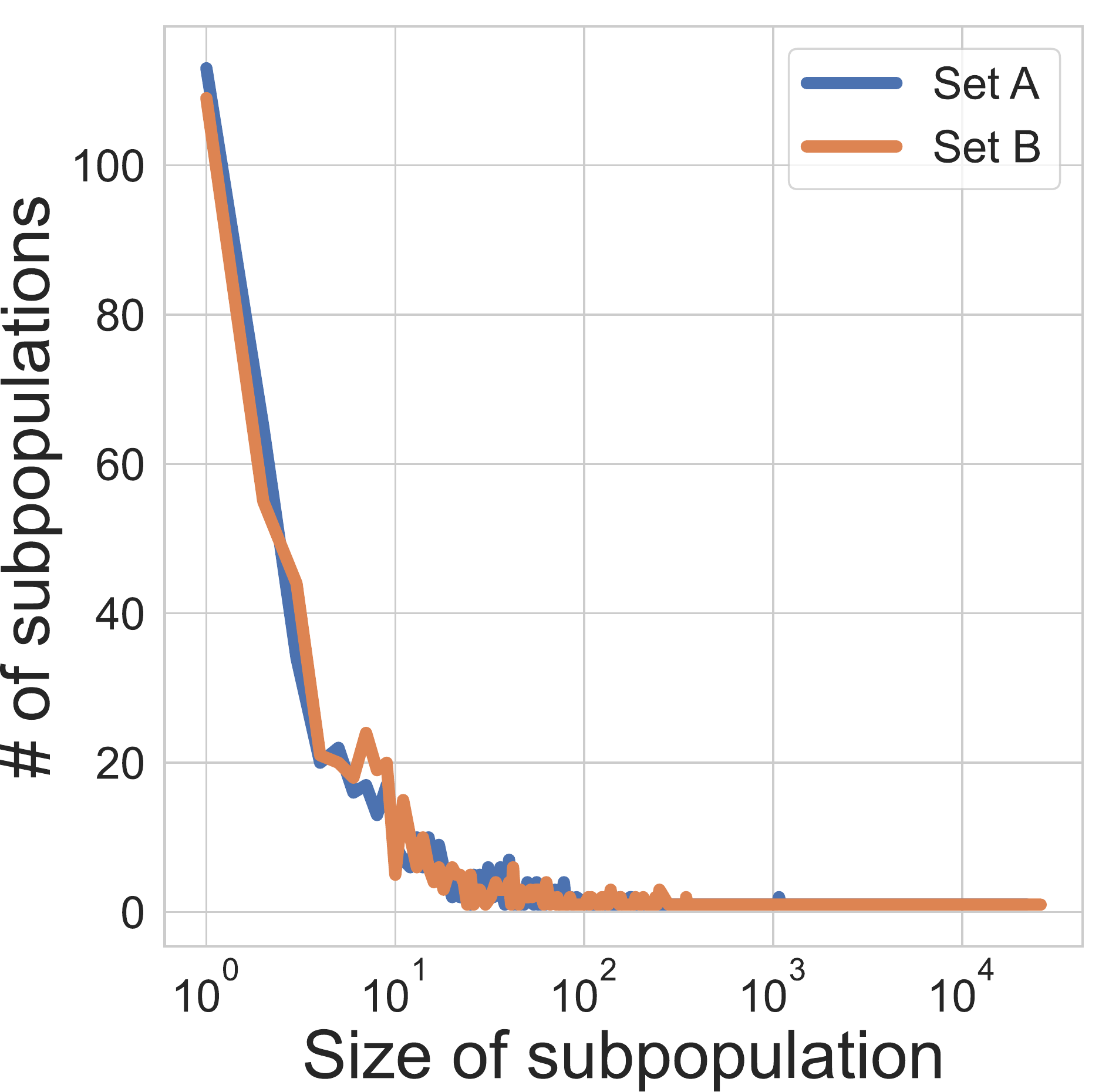%
    \end{subfigure}
    \begin{subfigure}[c]{0.48\linewidth}
        \centering
        \def\svgwidth{0.99\columnwidth}
        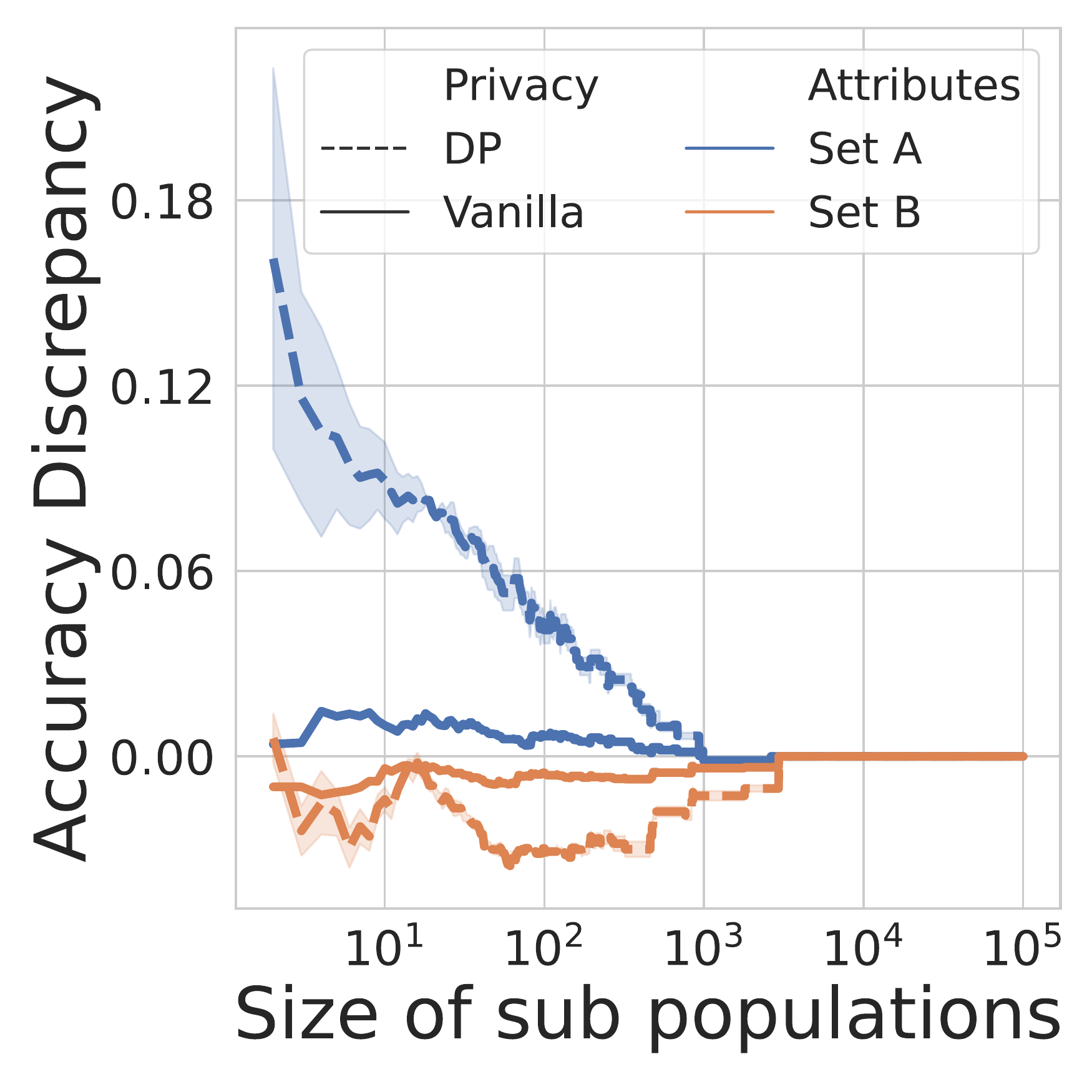%
        \end{subfigure}
        \caption{{\bfseries Left} Both sets of attributes induces a
        similar distribution over sizes of subpopulations on CelebA.
        {\bfseries Right} Set A has high accuracy discrepancy for
        small sized subpopulations whereas Set B does not.}
        \label{fig:fair-subpop-size}
      \end{figure}

  The experimental results on CelebA in~\Cref{sec:vision} shows that
  when the minority group is composed of small sized subpopulations,
  differential privacy requirements hurt the fairness of the algorithm.
  Here, we highlight that not all small-sized subpopulations are hurt
  equally in this process.~\Cref{fig:fair-subpop-size} shows that
  a different partition of CelebA composed of similar sized populations
  do not show similar behaviours in terms of how accuracy discrepancy
  changes with sizes of subpopulations. We refer to the \(11\)
  attributes we chose to partition the testset for our experiments so
  far as {\em Set A} and, here, we choose another set of \(11\) attributes and refer to them as {\em Set B}.~\Cref{fig:fair-subpop-size}~(left)
  shows that both Set A and Set B induces a very similar distribution
  over sizes of subpopulations on the test set.
  However,~\Cref{fig:fair-subpop-size}~(right) shows that while the
  group of minority subpopulations induced by Set A suffers very high
  accuracy discrepancy from private training compared to vanilla
  training, Set B does not~(see~\Cref{app:more-p-celebA} for
  more details on Set A and Set B). This indicates that, irrespective of
  sizes, private training hurts fairness disproportionately more for
  certain subpopulations compared to others. In particular, an
  interesting direction of further research is to investigate where
  these minority subpopulations that are worse-affected by private
  training intersects with the subpopulations that are relevant for the
  specific domain.  While most past
  works~\citep{bagdasaryan2019differential,Chang_2021} have also used
  sizes of subpopulations to differentiate between disparately impacted
  subpopulations, this suggests that that is not always the case. %
  
  In this paper, we have shown theoretically that when the minority
  group in the data is composed of multiple subpopulations, a DP
  algorithm can achieve very low error but necessarily incurs worse
  fairness. Further, we corroborated our theoretical results with
  experimental evidence on synthetic and real world computer vision
  datasets. 
  However, our model-agnostic results, that shed a rather pessimistic
  light on algorithmic fairness and differential privacy, only apply
  under certain distributional assumptions. It is possible that in some
  real-world datasets there are fair and private algorithms that achieve
  a more optimistic trade-off. This begs further research to develop
  fair and private algorithms that are closer to the pareto optimal
  frontier.

\section{Acknowledgements} 
AS is supported by the ETH AI Center and Hasler Stiftung.

\bibliographystyle{abbrvnat}
\bibliography{fairness}

\onecolumn
\appendix

\section{Proofs}
\label{sec:app_proofs}

\subsection{Proof for Theorem 1}
\label{app:thm_1}

\begin{thm}[Detailed version of~\Cref{thm:thm_1}]
    \label{thm:thm_1_detailed}
For any $\alpha \in (0, 0.025)$, $\epsilon \in (0, 1.1)$, $\delta \in (0,
0.01)$,  consider any \(\br{\epsilon,\delta}\)-DP
  algorithm \(\cA\) that does not make mistakes on points occurring
  more than once in the dataset. 
Then, there exists \(p\in\br{0,\nicefrac{1}{2}},c>0\) such that 

\[\err{\cA, \dist, \lblprior} \leq \alpha~\] and 
 \[\fair{\cA, \dist, \lblprior}\geq
    0.5\]

where \(\nicefrac{N}{m}\to c\) as \(N,m\to\infty\) and $\lblprior$
belongs to a family of label priors such that $\max_{x\in
X_2}||\lblprior(x)||_{\infty} \leq 0.1$ for $||\lblprior(x)||_{\infty} =
\max_{y\in \mathcal{Y}}\bP_{f\sim \lblprior}[f(x)= y]$. 
\end{thm}

\begin{proof}
    Recall the definition of the majority and minority subpopulations
    \(X_1\) and \(X_2\) from~\Cref{defn:fairness}. Given a dataset
    \(S\), define \(S_1\) to be the partition of \(S\) that belongs to
    \(X_1\) and \(m_1=\abs{S_1}\) to be the number of examples in
    \(S\) that belong to the majority subpopulation. Similarly, define
    \(S_2\) and \(m_2 = m - m_1\) as the set of minority examples and
    the size of the set of minority examples respectively. We also use
    \(S_i^\ell\) to denote the set of \(x\in X\) that appears
    \(\ell\) times in \(S_i\).

    First, we expand the expression for error,~defined in~\Cref{defn:err_dist} as follows.
    
\begin{equation}
       \begin{aligned}\label{eq:expand_error}
        \err{\mathcal{A}, \dist, \mathcal{F}} =&~\bE_{S,h,f}\bs{\sum_{x\in X}\dist(x)\mathbbm{1}\{h(x) \neq f(x)\}}\\
        =&\bE_{S,f}\bs{\sum_{\ell=0,i = 0}^{\ell=m,i=1} \sum_{x \in S_{i+1}^{\ell}}\dist(x)\bP_{\hsample{f}}[h(x)\neq f(x)]}\\
           =&\bE_{S,f}\bs{\sum_{\ell,i = 0}^1 \sum_{x \in S_{i+1}^{\ell}}\dist(x)\bP_{\hsample{f}}[h(x)\neq f(x)]}\\
            \leq&   \dfrac{1}{k} \bE_S\bs{\abs{S_{1}^0} + \abs{S_1^1}} + \dfrac{c_1p}{N} ~\bE_S\bs{\abs{S_2^0} + \abs{S_2^1}}\\
            \end{aligned}
\end{equation}
            where \( \bP_{f\sim\lblprior,\hsample{f}}\bs{h(x)\neq
            f(x)}\leq \max_{x\in X_2}\max_{y\in
            \mathcal{Y}}\bP_{f\sim \lblprior}[f(x)\neq
            y]:=c_1\leq 0.1\).%
            To see why, note that this upper bound can be achieved by
            an algorithm that ignores the labels in the dataset
            \(\Slabel\) and returns a deterministic classifier that
            predicts a fixed label for an example, possibly different
            for different examples. Similarly, we decompose the
            expression of accuracy discrepancy. For this purpose, for
            any \(\dist\), define the marginal distribution on the
            group of  {\em majority} subpopulations \(X_2\) as
            \begin{equation}
                \label{eq:minority-marginal-dist}
                    \enskip \dist^1\br{x}=\begin{cases} 
                    \dfrac{\Pi_{p,N}\br{x}}{\sum_{x\in X_1}\Pi_{p,N}\br{x}}=\dfrac{\Pi_{p,N}\br{x}}{1-p} & x\in X_1 \\
                    0 & x\not\in X_1
                 \end{cases}
            \end{equation}
        Now, we decompose the definition of accuracy discrepancy
        from~\Cref{defn:fairness} as follows
\begin{equation}\label{eq:expand_fairness}
    \begin{aligned}
        \fair{\mathcal{A}, \dist, \mathcal{F}} =&~\err{\mathcal{A}, \dist^2, \mathcal{F}} - \err{\mathcal{A}, \dist, \mathcal{F}}\\
        =&~(1-p) \left[\err{\mathcal{A}, \dist^2,\mathcal{F}} - \err{\mathcal{A}, \dist^1, \mathcal{F}}\right]\\
        =& (1-p)\sum_{\ell,i = 0}^1 \bE_{S,f}\bs{\br{-1}^{i+1}\sum_{x\in S_{i+1}^{\ell}}\dist^{i+1}\br{x}\bP_{\hsample{f}}[h(x) \neq f(x)]} \\
        \geq& \frac{c_2(1-p)}{N}\bE_S\bs{\abs{S_2^1}+\abs{S_{2}^0}}  - \frac{\br{1-p}}{k}\bE_S\bs{\abs{S_1^1}+\abs{S_1^0}}
        \end{aligned}
        \end{equation}
where we define \(c_2\) as follows. Let $\tilde{f}$ be sampled from a
class of functions that only differ from $f$ at $x$, then \(S_f\) and
\(S_{\tilde{f}}\) are neighboring datasets differing only at $x$.  
As $\cA$ is $(\epsilon, \delta)$-differentially private, for all
\(S\sim{\br\dist}^m\) and for all $x\in S_2$,  we have the following
inequality by the definition of differential privacy

\begin{equation*}%
\begin{aligned}
\bP_{h\sim \cA(S, f)}[h(x)\neq f(x)] &= 1- \bP_{h\sim \cA(S, f)}[h(x)=f(x)]\\
&\geq 1- e^{\epsilon}\bP_{h\sim \cA(S, \tilde{f})}[h(x)=f(x)] - \delta\\
&\geq \min_{x\in X_2} 1- e^{\epsilon}\max_{y\in \cY} \bP_{f\sim \lblprior}[f(x) = y] - \delta \\
&= \min_{x\in X_2} 1- e^{\epsilon}\norm{\lblprior(x)}_{\infty} - \delta := c_2
\end{aligned}
\end{equation*}
where $\norm{\lblprior(x)}_{\infty} = \max_{y\in \mathcal{Y}}
\bP_{f\sim \lblprior}[f(x) = y]$.  Now, we will bound the expectation
of the terms \(\abs{S_1^0},\abs{S_1^1},\abs{S_2^0},\abs{S_2^1}\)
individually to obtain the relevant upper and lower bounds. First, we
define the following random event over the sampling of the \(m\)-sized
dataset from \(\dist\).
\[\cE = \bc{\frac{p}{2}\leq\frac{m_2}{m}\leq\frac{3p}{2}}\]

\noindent As $m_2$ is a binomial distribution with parameter $(m, p)$,
as \(m\to\infty\) it could be  estimated with a Gaussian distribution
with mean $mp$ and variance $mp(1-p)$ by Central Limit Theorem (CLT).
Using Hoeffding's inequality and the assumption \(p<\frac{1}{2}\), we
lower bound the probability of \(\cE\) as follows:

\begin{equation}\label{eq:high_prob_bin_pop}
\begin{aligned}
\bP\bs{\cE}=&= 1-\bP\bs{m_2 \leq \frac{mp}{2}} - \bP\bs{m_2 \geq \frac{3mp}{2}} \\
&= 1-\bP\bs{|m_2 - mp|\geq \frac{mp}{2}}\\
&\geq 1-2\exp\br{-\frac{1}{2}mp^2} \to 0
\end{aligned}
\end{equation}

\noindent Then, by law of total expectation, we have that
\begin{equation}
    \begin{aligned}
        \lim_{m\to\infty}\bE_S\bs{\abs{S_1^0}} & = \lim_{m\to\infty}\bE_S\bs{\abs{S_1^0}\vert\cE}\bP\bs{\cE} + \lim_{m\to\infty}\underbrace{\bE_S\bs{\abs{S_1^0}\vert\cE^c}}_{\bigO{m}}\underbrace{\bP\bs{\cE^c}}_{o\br{\frac{1}{m^2}}}\\
    &=\lim_{m\to\infty}\sum_{x \in X_1} \bP_{S_1\sim \br{\dist^{1}}^{m_1}} \bs{x\text{ occurs } 0 \text{ times in }S_1\vert\cE}\bP\bs{\cE} \\
    &\leq \lim_{m\to\infty}k(1-p)\br{1-\frac{1}{k(1-p)}}^{\br{1-\frac{3p}{2}}m}\bP\bs{\cE}=0 
    \end{aligned}
    \end{equation}
where the last step follows because on the event
\(\cE\),~\(m_1 = m - m_2 \geq\br{1-\frac{3p}{2}}m\) and
\(\frac{1}{k\br{1-p}}\in\br{0,1}\).  It follows from similar arguments
that \(\lim_{m\to\infty}\bE\bs{\abs{S_1^1}}=0\).

\noindent Next, we compute \(\bE_S\bs{\abs{S_2^0}}\) and \(\bE_S\bs{\abs{S_2^1}}\). By simple
counting argument, we have that \begin{equation}
    \label{eq:maj_infreq_ball_thm1}
    \bE_S\bs{S_2^\ell}=\sum_{x\in X_2}\bP_{S_2\sim \br{\dist^{2}}^{m_2}}\bs{x\text{ occurs
} \ell \text{ times in }S_2} =
N\binom{m_2}{\ell}\br{\frac{1}{N}}^\ell\br{1-\frac{1}{N}}^{m_2 - \ell}
\end{equation}

\noindent Plugging in \(\ell=1\) and \(\ell=0\), we get 
\begin{equation}
        \bE_S\bs{\abs{S_2^1}} = m_2\br{1 - \frac{1}{N}}^{m_2-1}\quad~\text{and}\quad~ \bE_S\bs{\abs{S_2^0}} = N\br{1 - \frac{1}{N}}^{m_2}
\end{equation} respectively.
Plugging these expressions back into the expression of error
in~\Cref{eq:expand_error}, we obtain the following upper bound
\begin{equation}\label{eq:thm1-final-err}
\begin{aligned}
\errinf{\cA,\dist,\cF}=&\lim_{m,N\to\infty}  \err{\mathcal{A}, \dist, \mathcal{F}}\\ 
\leq& \underbrace{\frac{1}{k}\lim_{m\to\infty}  \bE_S\bs{\abs{S_{1}^0} + \abs{S_1^1}}}_{\to 0} + \lim_{m,N\to\infty}\frac{c_1 p}{N}  \bE_S\bs{\abs{S_{2}^0} + \abs{S_2^1}} \\
=&\lim_{m,N\to\infty} c_1 p\br{\frac{3pm}{2N}\br{1-\frac{1}{N}}^{\nicefrac{pm}{2}-1} + \br{1-\frac{1}{N}}^{\nicefrac{pm}{2}}}\\
=&c_1p\br{\frac{3p}{2c}e^{-\nicefrac{p}{2c}} + e^{-\nicefrac{p}{2c}}}
\end{aligned}
\end{equation}
where the last step follows from limit rules because \(\frac{N}{m}\to
c\) as \(m,N\to\infty\).

Similarly, we simplify the expression of accuracy discrepancy
from~\Cref{eq:expand_fairness}.
\begin{equation}\label{eq:thm1-final-fair}
    \begin{aligned}
        \fairinf{\cA,\dist,\cF} &= \lim_{m,N\to\infty}\fair{\cA,\dist,\cF}\\
        &\geq \lim_{m,N\to\infty}\frac{c_2(1-p)}{N}\bE_S\bs{\abs{S_2^1}+\abs{S_{2}^0}}  - \underbrace{\lim_{m,N\to\infty}\frac{\br{1-p}}{k}\bE_S\bs{\abs{S_1^1}+\abs{S_1^0}}}_{= 0}\\
        &\geq c_2\br{1-p}\br{\frac{p}{2c}e^{-\nicefrac{3p}{2c}}+e^{-\nicefrac{3p}{2c}}}
    \end{aligned}
\end{equation}

\noindent Next, we show that for any $\alpha \in [0, 1]$, there is some $p,c$
such that $\errinf{\cA, \Pi_{p, N}, \lblprior}\leq \alpha$ and $ \fairinf{\cA,
\Pi_{p, N}, \lblprior}\geq 0.5$. Setting \(c=10,p=0.26\), the minority group
accuracy discrepancy evaluates to $\fairinf{\cA, \Pi_{p, N}, \lblprior}\geq 0.5$.
For $p\in (0,0.26)$, the lower bound of accuracy discrepancy increases
with decreasing $p$ as its derivative is negative. 
Therefore,  $\fairinf{\cA, \Pi_{p, N}, \lblprior}\geq 0.5$ is always
satisfied for all $p \leq 0.26$ and \(c=10\) 

Now, we show that for $\alpha \in \bs{0, 1}$, there exists a
distribution $\dist$ with $p \leq 0.26, c = 10$ such that
$\errinf{\cA, \Pi_{p, N}, \lblprior}\leq \alpha$. Specifically, we
first argue that the upper bound of error monotonically increases with
$p$ for $p\in \bs{0, 0.26}$ for \(c=10\), then we show that the upper
bound achieves its maximum and minimum values at $0$ and $0.26$
respectively and invoke the intermediate value theorem to complete the
proof.

Taking the derivative of the upper bound of the error, for $p \in [0,
0.26]$ and \(c=10\),
we observe that the upper bound on error monotonically increases from
\(p=0\) to \(p=0.26\) for \(c=10\). Plugging in \(p=0.26\) and
\(p=0\), we obtain the maximum and minimum values to be \(0.027\) and
\(0\) respectively. Invoking the intermediate value theorem on this
completes the proof.

\end{proof}

\subsection{Proof for Theorem 2}
\label{app:thm_2}
In this section, we prove the following more detailed version
of~\Cref{thm:thm_2}. 

{\color{black}

\begin{thm}[Detailed version of~\Cref{thm:thm_2}]\label{thm:thm_2_detailed}

    For any \(p\in \br{0, \frac{1}{2}}\), \(c> 0\) such that
    $\frac{p}{c}\leq 1$, consider the distribution \(\Pi_{p, k, N}\)
    where \(\nicefrac{N}{m}\to c\) as \(N,m\to\infty\), \(k =
    o\br{m}\), and \(k(1-p)\geq 2\). For any $\epsilon>0,
    \frac{1}{2}\geq\delta > 0$, such that
    \(\frac{1}{\epsilon},\log\br{\frac{1}{\delta}}=o\br{m}\), consider
    an \(\br{\epsilon,\delta}\)-DP algorithm $\cA$ that satisfies
    Assumption~\ref{assump:A1} with $s_0=
    \floor{\min\br{\frac{1}{\epsilon}\log{2},\log{\br{\frac{1}{2\delta}}}}}$
    and some \(p_1\in\br{0,1}\). Also, let \(\lblprior\) be
    subpopulation-wise independent. Then, for the limits
    \(\errinf{\cA, \dist, \lblprior}=\lim_{m,N\to\infty}\err{\cA,
    \dist, \lblprior}\) and \(\fairinf{\cA, \dist,
    \lblprior}=\lim_{m,N\to\infty}\fair{\cA, \dist, \lblprior} \), we
    have 
    \[\fairinf{\cA, \dist, \lblprior} \geq
    \frac{\br{1-\norm{\lblprior}_\infty}\br{1-p}}{3}\br{1-\sqrt{\frac{3p}{\pi c}}\frac{e^{-\frac{c}{3p}\br{s_0 - \frac{3p}{2c}}^2}}{2\br{s_0 - \frac{3p}{2c}}}}- 
        (1-p)p_1\] while the error is upper
   bounded by \[\errinf{\cA,
      \dist, \lblprior} \leq (1-p_1)p\br{1-\frac{(s_0 -
        \frac{p}{2c})^2\frac{2c}{p} -1}{\sqrt{2\pi}(s_0 -
        \frac{p}{2c})^3\br{\sqrt\frac{2c}{p}}^3}e^{-(s_0 -
        \frac{p}{2c})^2\frac{c}{p}}} + p_1\] where  $||\lblprior||_{\infty} =
    \max_{x\in X,y \in \cY}\bP_{f\sim\lblprior}\bs{f(x) = y}$.
\end{thm}}

\begin{proof}
    We use the same extra notation for \(\bc{m_i,S_i,S_i^\ell}\) for
    \(i\in\bc{1,2}\) and \(\ell\in\bN,~\ell\leq m\) as in the proof
    of~\Cref{thm:thm_1_detailed}. First, we decompose the overall error. 
\begin{equation}\label{eq:thm2-error-decomp}
\begin{aligned}
\err{\cA, \dist, \lblprior} %
&=  \bE_{S,f}\bs{\sum_{\ell = 0}^{s_0} \sum_{x\in \Sell} \dist(x) \bP_{\hsample{f}}\bs{h(x)\neq f(x)}+\sum_{\ell = s_0+1}^m \sum_{x\in \Sell}\dist(x)\bP_{\hsample{f}}\bs{h(x)\neq f(x)}}\\
&\leq \sum_{\ell = 0}^{s_0}\bE_{S}\bs{\sum_{x\in \Sell}\dist(x)}+p_1\bE_{S}\bs{\sum_{\ell = 0}^m\sum_{x\in \Sell}\dist(x)} - p_1\sum_{\ell = 0}^{s_0}\bE_{S}\bs{\sum_{x\in \Sell}\dist(x)}\\
&\labelrel{=}{eq:rid-2-term}(1-p_1)\sum_{\ell = 0}^{s_0}\bE_{S}\bs{\sum_{x\in \Sell}\dist(x)} + p_1\\
&\leq \br{1-p_1} \sum_{\ell=0}^{s_0} \br{\bE_{S}\bs{\abs{S_1^\ell}+\abs{S_2^\ell}}}+p_1
\end{aligned}
\end{equation}
where \(S,f\) are distributed as \(S\sim\dist^m\) and
\(f\sim\lblprior\). In step~\eqref{eq:rid-2-term}, the second
expectation vanishes by noting that \(\sum_{\ell = 0}^m\sum_{x\in
\Sell}\dist(x) = 1\).

Further, by definition of accuracy discrepancy in~\Cref{eq:defn-fair}
we can write,\begin{equation}\label{eq:thm2-fair-decomp} \fair{\cA,
\dist, \lblprior} = \br{1-p}\br{\err{\cA, \dist^2, \lblprior} -
\err{\cA, \dist^1, \lblprior}}. \end{equation} We now expand the above
two terms in the decomposition of \(\fair{\cA,
\dist, \lblprior}\).

{
\color{black}
\begin{equation}\label{eq:err-2-lower}
\begin{aligned}
    \err{\cA, \dist^2, \lblprior} &=  \bE_{S}\bs{\sum_{\ell = 0}^{m_2}\sum_{x\in \Soell} \dist^2(x)\bP_{\hsample{f}f\sim\lblprior}\bs{h(x) \neq f(x)} }\\
&\geq \sum_{\ell = 0}^{s_0} \bE_S\bs{\sum_{x\in \Soell} \dist^2(x)\bP_{\hsample{f},f\sim\lblprior}\bs{h(x) \neq f(x)} }\\
&\geq\frac{\br{1-\norm{\lblprior}_\infty}}{3N}\sum_{\ell = 0}^{s_0}\bE_S\bs{\abs{\Soell}}\\
\end{aligned}
\end{equation}

where the last step follows by
applying~\Cref{lem:large_error_infreq_sample_app_v2}~(which is the
detailed version of~\Cref{lem:large_error_infreq_sample} in the main text) with the subpopulation-wise independent
\(\lblprior\) and with \(s_0=
\floor{\min\br{\frac{1}{\epsilon}\log{2},\log{\br{\frac{1}{2\delta}}}}}\).
} Next, we upper bound the error of the majority group.

\begin{equation}\label{eq:thm5_majority_error_upper_bound}
\begin{aligned}
\err{\cA, \dist^1, \lblprior} =&\bE_{S,f}\bs{\sum_{\ell = 0}^{s_0}\sum_{x\in \Stell} \dist^1(x)\bP_{\hsample{f}}\bs{h(x) \neq f(x)} }+\bE_{S,f}\bs{\sum_{\ell = s_0+1}^{m_1}\sum_{x\in \Stell} \dist^1(x)\bP_{\hsample{f}}\bs{h(x) \neq f(x)} }\\
\labelrel{\leq}{step:ass-1}& \bE_S\bs{\sum_{\ell = 0}^{s_0} \sum_{x\in \Stell} \dist^1(x)}+p_1\bE_S\bs{\sum_{\ell = s_0+1}^{m_1}\sum_{x\in \Stell} \dist^1(x)}\\
\labelrel{\leq}{step:not-whole-X-1}&  \frac{1}{k(1-p)}\bE_S\bs{\sum_{\ell = 0}^{s_0}\abs{\Stell} } + p_1%
\end{aligned}
\end{equation}
where step~\eqref{step:ass-1} follows from Assumption~\ref{assump:A1}
and step~\eqref{step:not-whole-X-1} follows because \( \sum_{\ell =
s_0+1}^{m_1} \sum_{x\in \Stell} \frac{1}{k(1-p)} \dist^1(x) \leq 1\) and
\(\dist^1(x) = \frac{1}{k(1-p)}\). 

In order to complete the proof, next we evaluate the
terms \(S_1^\ell\) and \(S_2^\ell\) respectively. We recall the
definition of the random event \(\cE\) from the proof
of~\Cref{thm:thm_1_detailed}

\[\cE = \bc{\frac{p}{2}\leq\frac{m_2}{m}\leq\frac{3p}{2}}\quad~\text{and}\quad \bP\bs{\cE}=1-o\br{e^{-\frac{mp^2}{2}}}~\enskip\text{as}~m\to\infty\]

\noindent Then, by law of total expectation, for all \(\ell\), we have

\begin{align*}
        \lim_{m\to\infty}\bE\bs{\abs{S_1^\ell}} & = \lim_{m\to\infty}\bE\bs{\abs{S_1^\ell}\vert\cE}\bP\bs{\cE} + \lim_{m\to\infty}\underbrace{\bE\bs{\abs{S_1^\ell}\vert\cE^c}}_{\bigO{m}}\underbrace{\bP\bs{\cE^c}}_{o\br{\frac{1}{m^2}}}\\
    &=\lim_{m\to\infty}\sum_{x \in X_1} \bP_{S_1\sim \br{\dist^{1}}^{m_1}} \bs{x\text{ occurs } \ell \text{ times in }S_1\vert\cE}\bP\bs{\cE} \\
    &\labelrel{=}{eq:count-thm2} \lim_{m\to\infty}k(1-p)\binom{m_1}{\ell}\br{\frac{1}{k(1-p)}}^\ell\br{1-\frac{1}{k(1-p)}}^{m_1-\ell}\bP\bs{\cE}\\
    &\labelrel{\leq}{eq:binom-ineq-thm2}\lim_{m\to\infty} k(1-p) \br{\frac{em_1}{\ell}}^\ell\br{\frac{1}{k(1-p)}}^\ell\br{1-\frac{1}{k(1-p)}}^{m_1 - \ell}\bP\bs{\cE}\\
\end{align*}
where step~\eqref{eq:count-thm2} follows from a simple counting
argument, step~\eqref{eq:binom-ineq-thm2} follows from the well known
inequality \(\binom{m_1}{\ell}\leq \br{\frac{e m_1}{\ell}}^\ell\).
Using standard limit rules we have
\begin{equation}\label{eq:majority-zero-thm2}
   \lim_{m\to\infty} \sum_{\ell=0}^{s_0}\bE\bs{\abs{S_1^\ell}}=0
\end{equation}
under the assumption that \(s_0,k=o\br{m}\) and conditioned on the
event \(\cE\) which requires \(m_1\geq \frac{\br{1-p}m}{2}\).

Using a similar technique, we next lower bound \(\lim_{m,N\to\infty}\sum_{\ell=0}^{s_0}\bE\bs{\abs{S_2^\ell}}\)

\begin{equation}\label{eq:S-2lower-bound}
    \begin{aligned}
        \lim_{m,N\to\infty}\sum_{\ell=0}^{s_0}\bE\bs{\abs{S_2^\ell}} & =\lim_{m,N\to\infty} \sum_{\ell=0}^{s_0}N\binom{m_2}{\ell}\br{\frac{1}{N}}^\ell\br{1-\frac{1}{N}}^{m_2-\ell}\bP\bs{\cE}\\
        &= \lim_{m,N\to\infty} N\bP_{\ell\sim \text{binom}\br{m_2, \frac{1}{N}}}\bs{\ell\leq s_0}\bP\bs{\cE}\\
        &\labelrel{\to}{step:clt-1-thm2} \lim_{m,N\to\infty}N\bP_{\ell\sim \cN\br{\frac{m_2}{N}, \frac{m_2}{N}\br{1-\frac{1}{N}}}}\bs{\ell\leq s_0}\bP\bs{\cE}\\
&\labelrel{\geq}{step:small-m-1-thm2}  N\Phi\br{\br{s_0 - \frac{3p}{2c}\sqrt{\frac{2c}{3p}}}},
    \end{aligned}
    \end{equation}

    where step~\eqref{step:clt-1-thm2} follows due to CLT, \(\Phi\) is
    the Gaussian CDF function, and step~\eqref{step:small-m-1-thm2}
    follows from the condition on \(\cE\) which enforces \(m_2\leq
    \frac{\br{1+p}m}{2}\) and from the limit \(\frac{N}{m}\to c\) as
    \(m,N\to\infty\). Using a similar argument, but using the fact
    that conditioning on \(\cE\) also enforces \(m_2\geq
    \frac{pm}{2}\), we obtain an upper bound on
    \(\lim_{m,N\to\infty} \sum_{\ell=0}^{s_0}\bE\bs{\abs{S_2^\ell}}\) as follows

    \begin{equation}\label{eq:S-2upper-bound}
        \lim_{m,N\to\infty} \sum_{\ell=0}^{s_0}\bE\bs{\abs{S_2^\ell}} \leq N\Phi\br{\br{s_0 - \frac{p}{2c}}\sqrt{\frac{2c}{p}}}
    \end{equation}
    We use the following well known tail bounds for Gaussian
    distributions to bound the CDF function in the expression for the
    upper~(\Cref{eq:S-2upper-bound}) and
    lower~(\Cref{eq:S-2lower-bound}) bounds on
    \(\lim_{m,N\to\infty} \sum_{\ell=0}^{s_0}\bE\bs{\abs{S_2^\ell}}\).
    \begin{inequality}[From~\citet{alma991060059979705501}]\label{ineq:gaussian_tail_bound}
    Let \(W\) be a random variable following the normal distribution
    \(N(\mu, \sigma^2)\). Then, we have the following lower and upper tail
    bounds on \(W\) for all \(x> 0\) : \[\br{\frac{1}{x} -
    \frac{1}{x^3}}\frac{e^{-\frac{x^2}{2}}}{\sqrt{2\pi}}\leq \bP\bs{W\geq
    \mu + \sigma x} \leq \frac{e^{-\frac{x^2}{2}}}{x\sqrt{2\pi}}\]
    \end{inequality}
    
    \noindent This gives the following upper bound on
    \(\lim_{m,N\to\infty} \sum_{\ell=0}^{s_0}\bE\bs{\abs{S_2^\ell}}\):

    \begin{equation}\label{eq:final-upper-S2}
        \lim_{m,N\to\infty}
    \sum_{\ell=0}^{s_0}\bE\bs{\abs{S_2^\ell}}\leq Np\br{1-\frac{(s_0 - \frac{p}{2c})^2\frac{2c}{p} -1}{\sqrt{2\pi}\br{s_0 - \frac{p}{2c}}^3\br{\sqrt\frac{2c}{p}}^3}e^{-\br{s_0 - \nicefrac{p}{2c}}^2\nicefrac{c}{p}}},
    \end{equation}
and the following lower bound
    \begin{equation}\label{eq:final-lower-bound-S2}
        \lim_{m,N\to\infty} \sum_{\ell=0}^{s_0}\bE\bs{\abs{S_2^\ell}}\geq
        N\br{1-\sqrt{\frac{3p}{\pi c}}\frac{e^{-\frac{c}{3p}\br{s_0 - \frac{3p}{2c}}^2}}{2\br{s_0 - \frac{3p}{2c}}}}.
\end{equation} 

    \noindent Plugging~\Cref{eq:final-upper-S2,eq:majority-zero-thm2}
    in~\Cref{eq:thm2-error-decomp}, we obtain the upper bound on the (overall) error in the asymptotic limit

    \begin{equation}
        \begin{aligned}
        \errinf{\cA,\dist,\cF}&=\lim_{m,N\to\infty}\err{\cA,\dist,\cF}\\
        &\leq (1-p_1)p\br{1-\frac{(s_0 - \frac{p}{2c})^2\frac{2c}{p} -1}{\sqrt{2\pi}\br{s_0 - \frac{p}{2c}}^3\br{\sqrt\frac{2c}{p}}^3}e^{-\br{s_0 - \nicefrac{p}{2c}}^2\nicefrac{c}{p}}} + p_1.
        \end{aligned}
    \end{equation}

    Plugging~\Cref{eq:final-lower-bound-S2}
    into~\Cref{eq:err-2-lower} and~\Cref{eq:majority-zero-thm2}
    into~\Cref{eq:thm5_majority_error_upper_bound}, we obtain the
    required expressions for accuracy discrepancy
    in~\Cref{eq:thm2-fair-decomp} and the desired result follows.
    \begin{equation}
        \begin{aligned}
        \fairinf{\cA,\dist,\cF}&=\lim_{m,N\to\infty}\fair{\cA,\dist,\cF}\\
        &\geq \frac{\br{1-\norm{\lblprior}_\infty}\br{1-p}}{3}\br{1-\sqrt{\frac{3p}{\pi c}}\frac{e^{-\frac{c}{3p}\br{s_0 - \frac{3p}{2c}}^2}}{2\br{s_0 - \frac{3p}{2c}}}}- 
        (1-p)p_1.
        \end{aligned}
    \end{equation}

\end{proof}

\subsection{Proof for Theorem 3}
\label{app:thm_3}

In this section, we prove the following more detailed version
of~\Cref{thm:thm_3}.
\begin{thm}[Detailed version
    of~\Cref{thm:thm_3}]\label{thm:thm_3_detailed} For any $p\in (0,
    \nicefrac{1}{2}), c> 0$ such that $p/c\leq 1$, consider the
    distribution \(\Pi_{p,N}\) where \(N\) is the number of minority
    subpopulations. The number of majority subpopulations is $k(1-p)$
    which satisfies $k(1-p)\geq 2$ and $k = o(m)$. For any $\alpha >
    0$, \(\epsilon = \frac{\log{2}}{\alpha\sqrt{m} +
    \frac{2-3p}{2k\br{1-p}}m}\) and
    \(\delta<2^{-\frac{1}{\epsilon}-1}\) consider any
    \(\br{\epsilon,\delta}\)-DP algorithm $\cA$.  Then,  for the
    limits  \(\nicefrac{N}{m}\to c\)~as
    \(m,N\to\infty\) we have

    \[\errinf{\cA, \dist, \lblprior} \geq \br{\frac{1-\norm{\lblprior}_\infty}{3}}\br{1-\br{1-p}\br{e^{-4\frac{(2-p)\alpha^2}{\br{2-3p}^2}}}}\]  
    
        \[\fairinf{\cA, \dist, \lblprior} \leq \br{1-p}\bs{1 - (1-e^{-\frac{4\br{2-p}\alpha^2}{\br{2-3p}^2}})\frac{\br{1-\norm{\lblprior}_\infty}}{3}}\]
        
    where  \(\errinf{\cA, \dist,
    \lblprior}=\lim_{m,N\to\infty}\err{\cA, \dist, \lblprior}\), and
    \(\fairinf{\cA, \dist, \lblprior}=\lim_{m,N\to\infty}\fair{\cA,
    \dist, \lblprior}\).
    \end{thm}

    \begin{proof}
        We use the same extra notations as in the proof
        of~\Cref{thm:thm_1_detailed,thm:thm_2_detailed}.
        {\color{black} In addition, we define \(s_0=
        \floor{\min\br{\frac{1}{\epsilon}\log{2},\log{\br{\frac{1}{2\delta}}}}}.\)}
        As in previous proofs, we first decompose the expression for
        overall error as
{\color{black}
    \begin{equation}\label{eq:thm3-err-decomp}
        \begin{aligned}
    \err{\cA, \dist, \lblprior} %
    =& \bE_{S}\bs{\sum_{\ell = 0}^{s_0}\sum_{x\in \Sell} \dist(x)\bP_{\hsample{f},f\sim\lblprior}[h(x)\neq f(x)]}\\
    &+\bE_{S}\bs{\sum_{\ell = s_0+1}^{m}\sum_{x\in \Sell} \dist(x)\bP_{\hsample{f},f\sim\lblprior}[h(x)\neq f(x)]}\\
    \geq& \frac{\br{1-\norm{\lblprior}_\infty}}{3}\sum_{\ell=0}^{s_0}\br{\frac{1}{k}\bE_S\bs{\abs{S_1^\ell}} + \frac{p}{N}\bE_S\bs{\abs{S_2^\ell}}}.
    \end{aligned}
    \end{equation} where the last step follows from the definition of \(s_0\) and~\Cref{lem:large_error_infreq_sample_app_v2}.}
   Further, by definition of accuracy discrepancy in~\Cref{eq:defn-fair}
    \begin{equation}\label{eq:thm3-fair-decomp} \begin{aligned}
        \fair{\cA, \dist, \lblprior} &=
        \br{1-p}\br{\err{\cA, \dist^2, \lblprior} - \err{\cA, \dist^1,
        \lblprior}}\\
        &\leq \br{1-p}\br{1 - \err{\cA, \dist^1,
        \lblprior}}, \end{aligned}\end{equation}
    where we use the simple upper bound \(\err{\cA, \dist^2, \lblprior}\leq 1\)
    for error in the minority subpopulations. We next lower
    bound the error of the majority subpopulations. 
    
    \begin{equation}\label{eq:lower-majority-thm3}
    \begin{aligned}
        \err{\cA, \dist^1, \lblprior} &= \bE_S\sum_{x\in X_1}\dist^1(x)\bP_{f\sim\cF,\hsample{f}}[h(x)\neq f(x)] \\
        &\labelrel{\geq}{step:assump-A2-2} \bE_S\left[\sum_{\ell = 0}^{s_0}\sum_{x\in \Stell}\frac{1}{k(1-p)} \frac{\br{1-\norm{\lblprior}_\infty}}{3}\right] \\
        &= \frac{(1-\norm{\lblprior}_\infty)}{3k(1-p)} \sum_{\ell = 0}^{s_0} \bE_S\bs{\abs{S_1^\ell}}.
    \end{aligned}
    \end{equation}
    In step~\eqref{step:assump-A2-2}, we drop the terms with $\ell >
    s_0$ and then apply~\Cref{lem:large_error_infreq_sample_app_v2}
    with the definition of \(s_0\).  We recall the
definition of the random event \(\cE\) from the proof
of~\Cref{thm:thm_1_detailed}

\[\cE = \bc{\frac{p}{2}\leq\frac{m_2}{m}\leq\frac{3p}{2}}\quad~\text{and}\quad \bP\bs{\cE}=1-o\br{e^{-\frac{mp^2}{2}}}~\enskip\text{as}~m\to\infty\]

    Next, we show that all points in the minority group occur less
    than $s_0$ times in \(S\). 
    
    \begin{equation*}\label{eq:thm6_infreq_minority}
        \begin{aligned}
        \lim_{m,N\to\infty}\sum_{\ell = s_0 + 1}^{m_2}\bE_S\bs{\abs{S_2^{\ell}}} 
        &=\lim_{m,N\to\infty}\sum_{\ell = s_0 + 1}^{m_2}\bE_S\bs{\abs{S_2^{\ell}}\vert\cE}\bP\bs{\cE} + \bE_S\bs{\abs{S_2^{\ell}}\vert\cE^c}\bP\bs{\cE^c} \\ 
        &=\lim_{m,N\to\infty}\sum_{\ell = s_0 + 1}^{m_2}N\binom{m_2}{\ell}\br{\frac{1}{N}}^{\ell}\br{1-\frac{1}{N}}^{m_2 - \ell}\bP\bs{\cE} \\
        &\leq\lim_{m,N\to\infty}    N e^{\br{-2m_2\br{1-\frac{1}{N} - \frac{m_2 - s_0}{m_2}}}^2}
        \end{aligned}
        \end{equation*}
 where the last step is an instantiation of the Hoeffding's inequality
    on the upper tail of a Bernoulli random variable. Noting that
    \(\delta<2^{-\frac{1}{\epsilon}-1}\), we have
    \(s_0=\frac{\log{2}}{\epsilon} = \alpha\sqrt{m} +
    \frac{\br{1-\frac{3}{2}p}}{k\br{1-p}}m\). Using the limit \(\frac{N}{m}\to
    c\) as \(m,N\to\infty\) and  \(m_2\geq \frac{pm}{2}\) due
    to the conditioning on the event \(\cE\),
    
    \begin{equation}\label{eq:thm6_infreq_minority}
        \lim_{m,N\to\infty}\sum_{\ell = s_0 + 1}^{m_2}\bE_S\bs{\abs{S_2^{\ell}}} \leq\lim_{N,m\to\infty} Ne^{-m\br{\frac{2\alpha\sqrt{m}}{3p} + \frac{2-3p}{3kp(1-p)} - \frac{1}{N}}} = 0.
    \end{equation}  
 {\color{black}Therefore,
    \begin{equation}\label{eq:all-min-infreq}
        \lim_{m,N\to\infty}\sum_{\ell = 0}^{s_0}\bE_S\bs{\abs{S_2^{\ell}}} = N - \lim_{m,N\to\infty}\sum_{\ell = s_0+1}^{m_2}\bE_S\bs{\abs{S_2^{\ell}}} = N
    \end{equation}}
    Similarly, using the law of total expectation and a counting argument, we lower bound
    \(\sum_{\ell=0}^{s_0}\bE_S\bs{\abs{S_1^\ell}}\)
    \begin{equation}\label{eq:majority-err-lower-thm3}
        \begin{aligned}
        \sum_{\ell=0}^{s_0}\bE_S\bs{\abs{S_1^\ell}}&\geq \sum_{\ell=0}^{s_0} k(1-p)\binom{m_1}{\ell}\br{\frac{1}{k(1-p)}}^{\ell}\br{1-\frac{1}{k(1-p)}}^{m_1 - \ell}\\
        &= k\br{1-p}\bP_{\ell\sim \text{binom}\br{m_1, \frac{1}{k(1-p)}}}\bs{\ell\leq s_0}.
        \end{aligned}
    \end{equation}
        To calculate \(\bP_{\ell\sim \text{binom}\br{m_1,
    \frac{1}{k(1-p)}}}\bs{\ell\leq s_0}\), we apply Hoeffding's inequality
    on Bernoulli distribution. 
    \begin{equation}\label{eq:maj_upper_bound_hoeffding_thm6}
    \begin{aligned}
    \bP_{\ell\sim \text{binom}\br{m_1, \frac{1}{k(1-p)}}}\bs{\ell\leq s_0} %
    &\geq 1-e^{-2m_1\br{1-\frac{1}{k(1-p)} - \frac{m_1 - s_0}{m_1}}^2 }\\
    &\labelrel{=}{step:expand-s0-2} 1-e^{-\br{2-p}m\br{\frac{\alpha}{\br{1-\frac{3}{2}p}\sqrt{m}}}^2}= 1-e^{-\frac{(2-p)\alpha^2}{\br{1-3p/2}^2}}
    \end{aligned}
    \end{equation}
    where step~\eqref{step:expand-s0-2}  follows from
    conditioning on \(\cE\) and due to \( s_0 = \frac{1-3p/2}{k(1-p)}m +
    \alpha\sqrt{m}\). Plugging~\Cref{eq:maj_upper_bound_hoeffding_thm6}
    into~\Cref{eq:majority-err-lower-thm3}, we have
    \begin{equation}\label{eq:maj-lower-thm3}
        \lim_{m,N\to\infty}\sum_{\ell=0}^{s_0}\bE_S\bs{\abs{S_1^\ell}}\geq k\br{1-p}\br{1-e^{-\frac{(2-p)\alpha^2}{\br{1-3p/2}^2}}}.
    \end{equation}
    \noindent Plugging~\Cref{eq:all-min-infreq,eq:maj-lower-thm3}
    into~\Cref{eq:thm3-err-decomp}, we lower bound overall error
    \begin{equation}
        \begin{aligned}
        \errinf{\cA,\dist,\lblprior}&=\lim_{m,N\to\infty}\err{\cA, \dist, \lblprior}\\
        & \geq \br{\frac{1-\norm{\lblprior}_\infty}{3}}\br{1-\br{1-p}\br{e^{-\frac{\br{2-p}\alpha^2}{(1-3p/2)^2}}}}.
        \end{aligned}
    \end{equation} 
    Further,
    plugging~\Cref{eq:maj-lower-thm3,eq:lower-majority-thm3}
    in~\Cref{eq:thm3-fair-decomp} we upper bound unfairness
    \begin{equation}
        \begin{aligned}
        \fairinf{\cA,\dist,\lblprior}&=\lim_{m,N\to\infty}\fair{\cA, \dist, \lblprior}\\
        &\leq (1-p)\bs{1 -  \frac{\br{1-\norm{\lblprior}_\infty}}{3}\br{1- e^{-\nicefrac{\br{2-p}\alpha^2}{(1-3p/2)^2}}}},
        \end{aligned}
    \end{equation}
    which completes the proof of~\Cref{thm:thm_3_detailed}.
    
    \end{proof}

\subsection{Proof of Lemma 1}
Next, we restate and
prove~\Cref{lem:large_error_infreq_sample}.
The proof makes use of the concept of vertex
cover and independent set that we define next.

\begin{defn}[Vertex Cover and Independent Set]\label{defn:vertex-cover}
    Consider any undirected graph \(\cG=\br{\cV,\cE}\) where \(\cV\)
    is the set of vertices and \(\cE\) is the set of edges. Also, let \(P\)
    define a probability distribution over the vertices \(\cV\).

    \textbullet~A set \(\cC\subseteq\cV\) is called the vertex cover
    of \(\cG\) if for all edges \(\br{u,v}\in\cE\) at least one of
    \(u\) and \(v\) belongs to \(\cC\). The smallest such set is known
    as the {\em minimum vertex cover}. The vertex cover with the
    smallest probability mass under \(P\) is the known as the {\em
    probabilistically minimum vertex cover}~(\(\pmvc{\cG}\). If a
    vertex cover does not remain a vertex cover upon the removal of
    any of its vertices, then it is known as a {\em minimal vertex
    cover}. 

    \textbullet~A set \(\cI\subseteq\cV\) is called an independent set
    of \(\cG\) if for any two vertices \(u,v\in\cI\), the edge
    \(\br{u,v}\not\in\cE\). The largest such set is known as the {\em
    maximum independent set} and the independent set with the largest
    probability mass is known as the {\em probabilistically maximum
    independent set}. If an independent set does not remain an
    independent set upon including a vertex from outside the set, then
    it is known as a {\em maximal independent set}.

    \textbullet~All minimum vertex covers and maximum independent sets are minimal vertex covers and maximal independent sets respectively but not vice versa.

    \textbullet~The complement of any vertex cover \(\cC\) is an independent set \(\cI=\cV\setminus\cC\). Therefore, the complement of a minimum vertex cover is a maximum independent set. 
\end{defn}

Further, for some $x\in S$, we let \(Q\in\cY^{m-1}\) represent an
arbitrary labeling of the points in \(S\setminus\bc{x}\) and use
\(\Qlblprior\) to denote the marginal distribution of \(\lblprior\)
over the set of functions that satisfies the labeling \(Q\) on
\(S\setminus\bc{x}\) and \(\Qlblset\) denote the support of
\(\Qlblprior\). We also define conditions ~\ref{cond:C1}
and~\eqref{cond:C2_2} on labeling functions as follows.

Let $S$ be a dataset of size $m$. Then,  for any \(x\in S\), we say that two
labelling functions \(f_1,f_2\) satisfy Condition~\ref{cond:C1} with
respect to \(x\) if they differ only in \(x\) i.e.
\begin{equation}\label{cond:C1}\tag{C1}\begin{aligned}
    f_1\br{z} = f_2\br{z}&~\text{for all}~z\neq x,~z\in S \\
f_1\br{z}\neq f_2\br{z}&~\text{for}~z=x.\end{aligned}\end{equation}

Further, a labelling function \(f\) satisfies
condition~\ref{cond:C2_2} with respect to \(x\) and an algorithm
\(\cA\) with parameter \(q\) if 
\begin{equation}\label{cond:C2_2}\tag{C2}
\bP_{h\sim \cA(\Slabel)}[h(x)\neq f(x)]> 1-q.
\end{equation}

\color{black}\begin{restatable}{lem}{infreqLemmaApp}[Detailed version of~\Cref{lem:large_error_infreq_sample}]\label{lem:large_error_infreq_sample_app_v2}
    Consider a label prior \(\lblprior\), an m-sized unlabeled
    \(S\sim\dist^m\), and an \(\br{\epsilon,\delta}\)-DP
    algorithm \(\cA\). For any \(q\in\br{0,1}\), let \(s_0\) be the largest integer that satisfies \(q\geq \frac{1 + s_0e^{-\epsilon}\delta}{1+e^{-s_0\epsilon}}\). Then for any \(x\in\Sell\), where \(\ell\in\bN,~\ell\leq s_0\), there exists a set of labeling functions \(\widehat{F}(x)\subset \lblset\) such that for all functions \(f\in\widehat{F}(x)\),
    \begin{equation}\label{eq:lem-high-err-cond}
        \bP_{h\sim \cA(\Slabel)}[h(x)\neq f(x)]> 1-q
    \end{equation}
    and the size of \(\widehat{F}(x)\) is large 
    \begin{equation}\label{eq:size-of-wrong-set}
        \abs{\widehat{F}(x)} \geq \abs{\lblset} -\sum_{Q\in\cY^{m-1}}\max_{y\in \cY}\abs{\bc{f\in
    \Qlblset: f\br{x}=y}  }.
    \end{equation}

    \noindent In particular, if \(\lblprior\) is subpopulation-wise
    independent, for \(s_0 =
    \floor{\min\br{\frac{1}{\epsilon}\log{2},\log{\br{\frac{1}{2\delta}}}}}\),
    we have 
    \begin{equation}\label{eq:lem-high-err-cond-const}
        \bP_{h\sim \cA(\Slabel),f\sim\lblprior}[h(x)\neq f(x)]>  \frac{1-\norm{\cF}_\infty}{3}.
    \end{equation}
\end{restatable}

\begin{proof}[Proof of~\Cref{lem:large_error_infreq_sample_app_v2}]
    Let $S$ be a dataset of size $m$.
By~\Cref{lem:two_fn_disagree}, if \(q\geq \frac{1 +
s_0e^{-\epsilon}\delta}{1+e^{-s_0\epsilon}}\) and if \(f_1,f_2\)
satisfy Condition~\ref{cond:C1} then at least one of them satisfies
condition~\ref{cond:C2_2} with respect to all \(x\in\Sell\) for
\(\ell\leq s_0\) and all \(\br{\epsilon,\delta}\)-DP algorithm \(A\)
with parameter \(q\) i.e.%

\begin{equation*}
    \bP_{h\sim \cA(\Slabel)}[h(x)\neq f(x)]>1-q.
\end{equation*}

\noindent Next, we first measure the minimum size of the set of such
labelling functions \(f\) and then lower bound the probability with
which such a labelling function will be drawn when sampling from
\(\lblprior\). Let \(\zeta\br{Q,x}\) denote the total number of
functions in \(\Qlblset\) that satisfy condition~\ref{cond:C2_2}. We
show that \(\zeta\br{Q,x}\) is bigger than the size of the minimum
vertex cover~(see~\Cref{defn:vertex-cover}) of the following graph.\\

Construct a graph \(\cG_{S,x,Q}\) where every function
\(f_i\in\Qlblset\) represents a vertex.  Construct an edge between two
vertices \(f_i,f_j\) if they satisfy Condition~\ref{cond:C1} i.e.
\(f_i\br{x}\neq f_j\br{x}\). It is easy to see that this is a complete
\(\abs{\cY}\)-partite graph where each partition consists of functions
that predict a different label $y$ on \(x\), i.e. they are defined as
\(\bc{f\in\Qlblset:~f\br{x}=y}\).
A vertex cover of a graph~(defined in~\Cref{defn:vertex-cover}) is a
set of vertices such that for every edge at least one of its end
points are included in the set. As for every edge in \(\cG_{S,x,Q}\),
both its endpoints satisfy condition~\ref{cond:C1}, at least one of
them should satisfy condition~\ref{cond:C2_2}.  Therefore, the total
number of functions that satisfy condition~\ref{cond:C2_2} is at least
the size of the minimum vertex cover of \(\cG_{S,x,Q}\).
Using~\Cref{lem:graph_mult_prob} to calculate the minimum vertex cover
of \(\cG_{S,x,Q}\), we have that
\[\zeta\br{Q,x}\geq\abs{\Qlblset}-\max_{y\in\cY}\abs{\bc{f\in
\Qlblset: f\br{x}=y}}.\]

\noindent Summing this up for all labellings, the total number of labeling functions that  satisfy condition~\ref{cond:C2_2} is
\begin{equation}\label{eq:sum-all-labels}
    \begin{aligned}
        \abs{\widehat{F}(x)}=\sum_{Q\in\cY^{m-1}}\zeta\br{Q,x}&\geq\sum_{Q\in\cY^{m-1}} \abs{\Qlblset}-\max_{y\in\cY}\abs{\bc{f\in
\Qlblset: f\br{x}=y}}\\
&=\abs{\lblset} -\sum_{Q\in\cY^{m-1}}\max_{y\in\cY}\abs{\bc{f\in
\Qlblset: f\br{x}=y}}.
\end{aligned}
\end{equation}
This completes the proof of~\Cref{eq:size-of-wrong-set}. To
prove~\Cref{eq:lem-high-err-cond-const}, consider \(\lblprior\) to be
a subpopulation-wise independent label prior and recall that for all
\(Q\in\cY^{m-1}\), \(\cG_{S,x,Q}\) is a complete multi-partite graph.
Using the same argument as above, with~\Cref{lem:graph_mult_prob}, we
can lower bound the probability mass of \(\widehat{F}\br{x}\)
\begin{align*}
    \bP_{f\sim\lblprior}\bs{f\in\widehat{F}\br{x}}&\geq \sum_{Q\in{\cY}^{m-1}}\bP_f\bs{f\in\pmvc{\cG_{S,x,Q}}}\\
     &=\sum_{Q\in{\cY}^{m-1}}\bP_f\bs{f\in\pmvc{\cG_{S,x,Q}}\vert f\in\Qlblset}\bP_f\bs{f\in\Qlblset}\\
     &\geq\min_{Q\in{\cY}^{m-1}}\bP_f\bs{f\in\pmvc{\cG_{S,x,Q}}\vert f\in\Qlblset}\underbrace{\sum_{Q\in{\cY}^{m-1}}\bP_f\bs{f\in\Qlblset}}_{=1}
\end{align*}
where \(\sum_{Q\in{\cY}^{m-1}}\bP_f\bs{f\in\Qlblset}=1\) by the law of
total probability. Using the assumption that \(\lblprior\) is
subpopulation-wise independent and the definition of
\(\norm{\cF}_{\infty}\),
we obtain\[\min_{Q\in{\cY}^{m-1}}\bP_f\bs{f\in\pmvc{\cG_{S,x,Q}}\vert
f\in\Qlblset}\geq\br{1-\norm{\lblprior}_{\infty}}\] and thus
\begin{equation}\label{eq:f-is-bad}\bP_{f\sim\lblprior}\bs{f\in\widehat{F}\br{x}}\geq
\br{1-\norm{\lblprior}_{\infty}}.\end{equation}
Using~\Cref{lem:two_fn_disagree,eq:f-is-bad}, we can lower
bound 
\begin{equation}\label{eq:incorrec-pred-prob}
\begin{aligned}
    \bP_{f\sim\lblprior, h\sim \cA(\Slabel)}[h(x)\neq f(x)] &\geq \bP_{f\sim\lblprior, h\sim \cA(\Slabel)}\bs{h(x)\neq f(x)\vert f\in\widehat{F}(x)}\bP_f\bs{f\in\widehat{F}(x)} \\
    &\geq\br{1-q}\br{1-\norm{\lblprior}_{\infty}}
\end{aligned}
\end{equation}

Substituting the value of
\(s_0=\floor{\min\br{\frac{1}{\epsilon}\log{2},\log{\br{\frac{1}{2\delta}}}}}\),
we obtain \(q=\frac{2}{3}\geq \frac{1 +
s_0e^{-\epsilon}\delta}{1+e^{-s_0\epsilon}}\). Using \(q=\frac{2}{3}\)
in~\Cref{eq:incorrec-pred-prob} completes the proof
of~\cref{eq:lem-high-err-cond-const}. 
\end{proof}

\begin{restatable}{lem}{twofndisagree}\label{lem:two_fn_disagree} Let
    $S$ be a dataset of size $m$ and $q \in (0, 1)$. Further, let
    $s_0$ be the largest integer that fulfills $q\geq \frac{1 + s_0
    e^{-\epsilon}\delta}{1 + e^{-s_0\epsilon}}$. Then, for an
    $(\epsilon, \delta)$-differentially private algorithm $\cA$, for
    any $x \in \Sell$ where $\ell \in \mathbb{N}, \ell \leq s_0$, if
    two functions \(f_1,f_2\) agree on all points in
    \(S\setminus\bc{x}\) but disagree on \(x\), i.e. satisfies
    condition~\ref{cond:C1}, then at least one of the two functions
    \(f\in\bc{f_1,f_2}\) satisfies condition~\ref{cond:C2_2},
    \[\bP_{\hsample{f}}\bs{h(x)\neq f(x)}> 1-q.\]
    \end{restatable}
    
    \begin{proof}[Proof of~\Cref{lem:two_fn_disagree}]
   
     We prove this via contradiction. Let's assume that two
    labeling functions \(f_1,f_2\) satisfying
    condition~\ref{cond:C1} also satisfies 
    
    \begin{equation}\label{assump:contra-low-err}
        \bP_{h\sim \cA(\Slabel)}[h(x)\neq f(x)]\leq 1-q
        \end{equation}
    for some \(x\in\Sell\). Then, we lower bound the probability that
    an \(\br{\epsilon,\delta}\)-DP algorithm trained on a dataset,
    labeled with \(\tilde{f}\), disagrees with \(f\) on \(x\in
    \Sell\).
    \begin{equation}\label{eq:infreq_lemma1_contra1}\begin{aligned}
    \bP_{\hsample{\tilde{f}}}\bs{h(x)= f(x)}&\geq \br{\bP_{\hsample{f}}\bs{h(x)= f(x)} - \ell e^{(\ell - 1)\epsilon}\delta}e^{-\ell\epsilon}\\
    &\geq \br{q - \ell e^{\br{\ell-1}\epsilon}\delta}e^{-\ell \epsilon}
    \end{aligned}
    \end{equation}
    The first inequality follows from the notion of group differential privacy and the second inequality follows from Assumption~\eqref{assump:contra-low-err}.
    Further, by definition of $f_1, f_2$ and Assumption~\ref{assump:contra-low-err},  we have
    \begin{equation}\label{eq:infreq_lemma1_contra2}
    \bP_{\hsample{\tilde{f}}}\bs{h(x)= f(x)}\leq \bP_{\hsample{\tilde{f}}}	\bs{h(x)\neq \tilde{f}(x)} \leq 1-q.
    \end{equation}
    Combining the two
    inequalities~(\ref{eq:infreq_lemma1_contra1},\ref{eq:infreq_lemma1_contra2}),
    it has to hold that  $q\leq \frac{1 + \ell e^{-\epsilon}\delta}{1 + e^{-\ell \epsilon}}$.
    However, this leads to a contradiction as \(q\geq \frac{1 + s_0
    e^{-\epsilon}\delta}{1 + e^{-s_0\epsilon}}\). This implies that
    for any $\ell\leq s_0$ such that $q\geq \frac{1 + \ell
    e^{-\epsilon}\delta}{1 + e^{-\ell \epsilon}}$,
    Assumption~\ref{assump:contra-low-err} is incorrect which
    concludes our proof.
    
    \end{proof}

\begin{lem}\label{lem:graph_mult_prob} Given  \(k\) sets of vertices
    \(\bc{\cV_i}_{i=1}^k\), consider a complete k-partite graph
    \(\cG=\br{\cV,\cE}\) where the \(k^{\it th}\) partition consists
    of \(\cV_i\). Then the size of the minimum vertex cover is
    \(\abs{\cV} - \max_i{\abs{\cV_i}}\).

    Further, if there exists a probability distribution \(\mu\) over
    the vertices, the probabilistically minimum vertex cover~(or
    \(\pmvc{\cG}\)) has a probability mass of \(1 -
    \max_i\mu\bs{\cV_i} \) where
    \(\mu\bs{\cV_i}=\bP_{V\sim\mu}\bs{V\in\cV_i}\).
    \end{lem}
    
    \begin{proof}[Proof of~\Cref{lem:graph_mult_prob}] A set of vertices is
    known as an independent set if and only if its compliment is a vertex
    cover. Therefore, we use the identity that the size of the minimum
    vertex cover and the maximum independent set of a graph equals to the
    total number of vertices of the graph. Therefore, the problem of finding the minimum vertex cover is equivalent to finding the maximum independent set. 
    
    By the definition of independent set, no two vertices in the
    independent set are adjacent. Without loss of generality, let
    \(v\in\cV_i\) for some \(i\) be in the independent set \(\cI\). Then,
    by definition of complete k-partite graph, all vertices belonging to
    \(\bigcup_{j\neq i}\cV_j\) are adjacent to \(v\) and  therefore cannot
    be in \(\cI\). However, all other vertices in \(\cV_i\) can be
    included in \(\cI\). This creates a {\em maximal independent set}.
    Thus, there are only \(k\) maximal independent sets and the set
    corresponding to \(i^*=\argmax_i{\abs{\cV_i}}\) is the maximum
    independent set.  Therefore, the size of the minimum vertex cover
    is \(k - \max_i{\abs{\cV_i}}\).
    
    By a similar reasoning, the set corresponding to
    \(i^*=\argmax_i{\mu\bs{\cV_i}}\) is the probabilistically maximum
    independent set.  Therefore, the probability mass of \(\pmvc{\cG}\) is \(1 -
    \max_i{\mu\bs{\cV_i}}\).
    \end{proof}

\subsection{Proof of Lemma 2}
We restate a detailed version of~\Cref{lem:noisy-majority} and provide
the proof below. 
{\color{black}
\begin{restatable}{lem}{noisyMajorityApp}
    \label{lem:noisy-majority-app} The algorithm \(A_\eta:S\rightarrow \cY^X\) is
    \(\br{\log{\br{\frac{\br{1-\eta}\br{\abs{\cY}-1}}{\eta}}},0}\)- differentially private.
    Further for any dataset \(\Slabel\) and \(s_0\in\bN\), \begin{itemize}
        \item if a subpopulation \(x\) appears more than \(s_0\) times in \(S\), \(\bP_{h\sim\cA_{\eta}\br{\Slabel}}\bs{h(x)\neq f(x)}\leq
    e^{-\nicefrac{s_0\br{1-2\eta}^2}{8\br{1-\eta}}}\) and \\
    \item if a subpopulation \(x\) appears less than \(s_0\) times in \(S\),  \(\bP_{h\sim\cA_{\eta}\br{\Slabel}}\bs{h(x)\neq f(x)}\geq \frac{1}{\sqrt{2s_0}}{\br{4\eta\br{1-\eta}}}^{\nicefrac{s_0}{2}}\).
    \end{itemize}
    \end{restatable}
}
\begin{proof}
We restate the algorithm for the sake of completeness. The algorithm
\(A_\eta\) accepts a dataset \(\Slabel\in\br{X\times\cY}^m\) and a
noise rate \(\eta\in\br{0,\frac{1}{2}}\) as input. Then it creates a
dictionary where the set \(X\) is the set of keys. In order to assign
values to every key, it first randomly flips the label of every
element in \(\Slabel\) with probability \(\eta\), then for every
unique key in \(\Slabel\), the algorithm computes the majority label
of that key in the flipped dataset and assigns that majority label to
the corresponding key. For elements in \(X\) not present in
\(\Slabel\), it assigns a random element from \(\cY\).

\noindent\textbf{Privacy}: We first prove that $A_{\eta}$ is $(\epsilon, 0)$-differentially
private. For any two neighboring datasets \(\Slabel\) and
\(S_{\tilde{f}}\) that differ in one element, the classifiers
obtained from the algorithm \(A_{\eta}\) on the two datasets differs
the most if the datasets differ on a singleton point $x'$ i.e. a point
which appears just once in the dataset.  
\begin{equation}
\begin{aligned}
\mathrm{For~all}~ x'\in S,f\in\lblset,~\mathrm{and}~\tilde{f}\in\clblprior&\\
\frac{\bP_{\hsample{f}}[h(x)= f(x), \forall x\in S]}{\bP_{\hsample{\tilde{f}}}[h(x)= f(x), \forall x\in S]} &= \frac{\prod_{x\in S}\bP_{\hsample{f}}[h(x)= f(x)]}{\prod_{x\in S}\bP_{\hsample{\tilde{f}}}[h(x)= f(x)]} \\
&= \frac{\bP_{\hsample{f}}[h(x')= f(x')]}{\bP_{\hsample{\tilde{f}}}[h(x')= f(x')]} \\
&= \frac{\bP_{\hsample{f}}\bs{x' \text{ is not flipped }}}{\bP_{\hsample{\tilde{f}}}\bs{x' \text{ is flipped and flipped to }f(x')}}\\
&= \frac{1-\eta}{\eta\frac{1}{\abs{\mathcal{Y}}-1}} = \frac{1-\eta}{\eta}(|\mathcal{Y}|-1)
\end{aligned}
\end{equation}
where \(\clblpriorx{x}\) is the subset of $\lblset$ consisting of all
labeling functions that agree with $f$ at all points in $S$ except
$x$. This shows that the algorithm \(A_\eta\) is
\(\br{\bigO{\log{\frac{1}{\eta}}},0}\)-DP. 

\noindent\textbf{Utility:} Next, we bound the probability that
$A_{\eta}$ misclassifies a point occurring $\ell$ times in $S$ in
order to establish the utility of the algorithm. For a point $x$
occurring $\ell$ times in $S$, we denote all replicates of this point
in the dataset as $\bs{X_2, ..., x_{\ell}}$. As each point is flipped
with probability $\eta$, let $k_i^x = \mathbbm{I}\{x_i \text{ is not
flipped}\}$, then $\sum_{i = 1}^{\ell}k_i^x$ is a binomial random
variable with $\ell$ trials and success probability $1-\eta$. The
algorithm $A_{\eta}$ would correctly classify $x$ if $\sum_{i =
1}^{\ell} k_i^x > \ell/2$. 

The upper bound of the error probability is shown in \citet{Sanyal2020Benign}, but we reproduce it here for completeness. 

\begin{equation}\label{eq:majority-vote-ub-result}
\begin{aligned}
\bP\bs{\sum_{i = 1}^{\ell} k_i^x \leq \frac{\ell}{2}} &= \bP\bs{\sum_{i = 1}^{\ell} k_i^x\leq \frac{\ell}{2\mu}\mu + \mu - \mu}\\
&= \bP\bs{\sum_{i = 1}^{\ell}k_i^x\leq\br{1-\br{1-\frac{\ell}{2\mu}}}\mu}\\
&\labelrel{\leq}{step:chernoff-1} e^{-\frac{(1-2\eta)^2}{8(1-\eta)^2}\mu}\\
&\labelrel{=}{step:replace-mean} e^{-\frac{(1-2\eta)^2}{8(1-\eta)}\ell}
\end{aligned}
\end{equation}

where step~\eqref{step:chernoff-1} follows from Chernoff's
concentration bound on the lower tail and
step~\eqref{step:replace-mean} due to \(\mu = (1-\eta)\ell\).

Then, we show a lower bound for the error probability using an
anti-concentration bound on the lower tail. Define $D(a||p)$ to be the
Kullback-Leibler divergence between two Bernoulli random variables
with parameter $a$ and $p$, i.e. $D(a\|p) = a \log \br{\frac{a}{p}} +
(1-a)\log \br{\frac{1-a}{1-p}}$. 

\begin{equation}\label{eq:majority-vote-lb-result}
\begin{aligned}
\bP\bs{\sum_{i = 1}^{\ell}k_i^x \leq \frac{\ell}{2}} &\labelrel{\geq}{step:binom-anticoncentration-1}\frac{1}{\sqrt{2\ell}}e^{-\ell D\br{\frac{\ell}{2\ell}\| 1-\eta}}\\
&= \frac{1}{\sqrt{2\ell}}e^{-\frac{\ell}{2} \br{\log\br{\nicefrac{1}{2\br{1-\eta}}} + \log\br{\nicefrac{1}{2\eta}}}}\\
&= \frac{1}{\sqrt{2l}}\bs{4\eta\br{1-\eta}}^{\nicefrac{\ell}{2}}
\end{aligned}
\end{equation}
where step~\eqref{step:binom-anticoncentration-1} follows from the binomial anti-concentration bound as in~\Cref{ineq:binom_anticoncentration_bound}. This completes the proof.

\begin{inequality}[From~\citet{alma991061165399705501}]
    \label{ineq:binom_anticoncentration_bound}
Let \(X\) be a binomial random variable with parameters \((m, p)\),
then for any \(k < mp\), we can lower bound the left tail of \(X\) as
\[\bP\bs{X\leq k} \geq \frac{1}{\sqrt{2m}}e^{-mD(\frac{k}{m}||p)}.\]
\end{inequality}

\end{proof}

\section{Experimental Details and Additional Experiments}
\label{app:exp-details}
\subsection{Synthetic data}
\label{app:more-p}
\paragraph{Constructing the synthetic data distribution} Given
\(m\in\bN,p\in\br{0,0.5}, C>0,\) and \(k\in\bN\), we first create two
boolean hypercubes of dimension \(d_{\mathrm{min}}\) and
\(d_{\mathrm{maj}}\) respectively, where
\(d_{\mathrm{min}}=\bigO{\ceil{\log{Cm}}}\) and
\(d_{\mathrm{maj}}=\bigO{\ceil{\log\br{k\br{1-p}}}}\) respectively.
For the minority subpopulation, the data is distributed uniformly
across a mixture of \(N\) Gaussian clusters positioned on a randomly
chosen set of \(d_\mathrm{min}\) vertices of this hyper-cube.
Similarly, for the majority class the data is distributed across a
uniform mixture of \(d_{\mathrm{maj}}\) Gaussian clusters situated on
randomly chosen vertices of the second hypercube. Each of the clusters
is randomly assigned a binary label. Finally, \(m\) samples are
sampled from the minority distribution, along with their labels, and
each data covariate is appended with a \(d_{\mathrm{maj}}\)-dimensional vector of
all \(10^{-4}\) to create a \(d_{\mathrm{min}}+{d_\mathrm{maj}}\)
dimensional dataset of size \(m\). Similarly, \(m\) samples are
sampled from the majority distribution and the covariates are appended
with a \(d_{\mathrm{min}}\) vector of all \(10^{-4}\) to create a
\(d_{\mathrm{min}}+{d_\mathrm{maj}}\) dimensional dataset of size
\(m\). Finally, \(m_1\) points are chosen randomly from the minority dataset and \(m-m_1\) points are chosen randomly from the majority dataset where \(m_1\) is sampled from a binomial distribution with parameters \(m\) and \(p\). 

\paragraph{DP Algorithm for synthetic data}

We use PyTorch Opacus~\citep{opacus} for all our experiments. We use
multiple values of \(\epsilon\) in the range \(\br{0.1,10}\). We clip
the maximum gradient norm to \(1\), and then train the algorithm for
\(10^3\) epochs. Note that the clipping is necessary to control the
lipschitzness of the gradient descent algorithm update step. The
DP-SGD algorithm uses the value of \(\delta\), the maximum gradient
norm, and the number of epochs mentioned above to compute the
noise-multiplier of the algorithm that is necessary to attain the
required privacy guarantee.

\paragraph{Additional weight of minority group}

In this section, we plot additional experiments with the same setup
as~\Cref{sec:synth}. In particular, we show results from the same
experiment as~\Cref{fig:incr-C-hurts-fairness} but with \(p=0.5\)
instead of \(p=0.2\). Our experiments show that the same trend is held in particular for how the accuracy discrepancy changes with \(c\) and with \(\epsilon\).
\begin{figure}[!htb]
  \centering
  \def\svgwidth{0.99\columnwidth}
  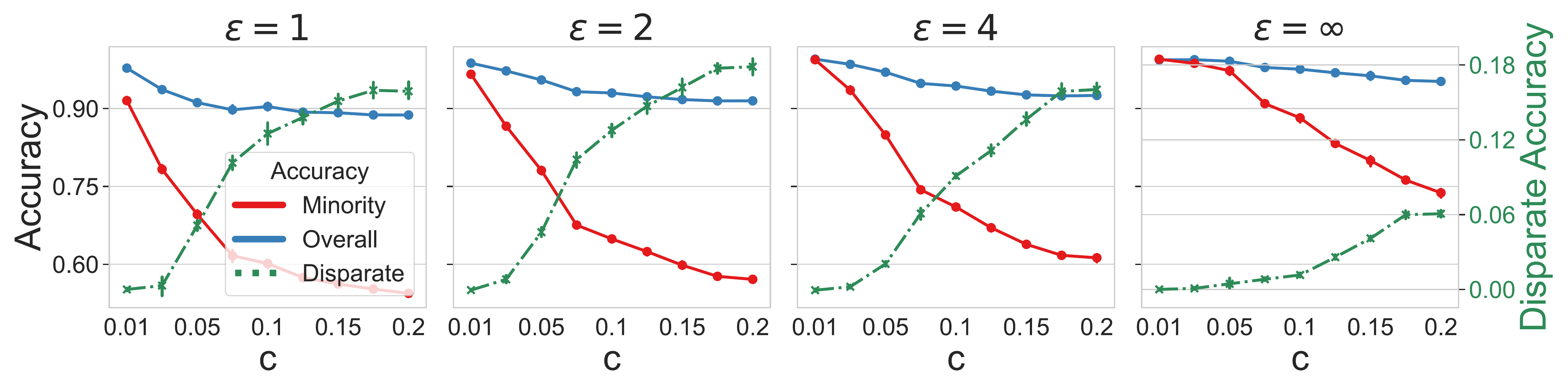%
  \caption{Here \(p=0.5\). Each figure plots the disparate
  accuracy~(higher is less fair) in
  {\color{OliveGreen} green dashed line}, the accuracy of the minority
  group with {\color{red} red}, and the overall accuracy with
  {\color{NavyBlue} blue} on the y-axis and the parameter \(c\) in the
  X-axis. The left most~\(\br{\epsilon=1}\) achieves the strictest
  level of privacy and the right most~\(\br{\epsilon=\infty}\) is
  vanilla training without any privacy constraints. }
  \label{fig:incr-C-hurts-fairness-p05}
\end{figure}

\subsection{Additional sizes of subpopulations in CelebA}
\label{app:more-p-celebA}

\begin{figure}[!h]\centering
    \begin{subfigure}[c]{0.28\linewidth}
        \centering
        \def\svgwidth{0.99\columnwidth}
        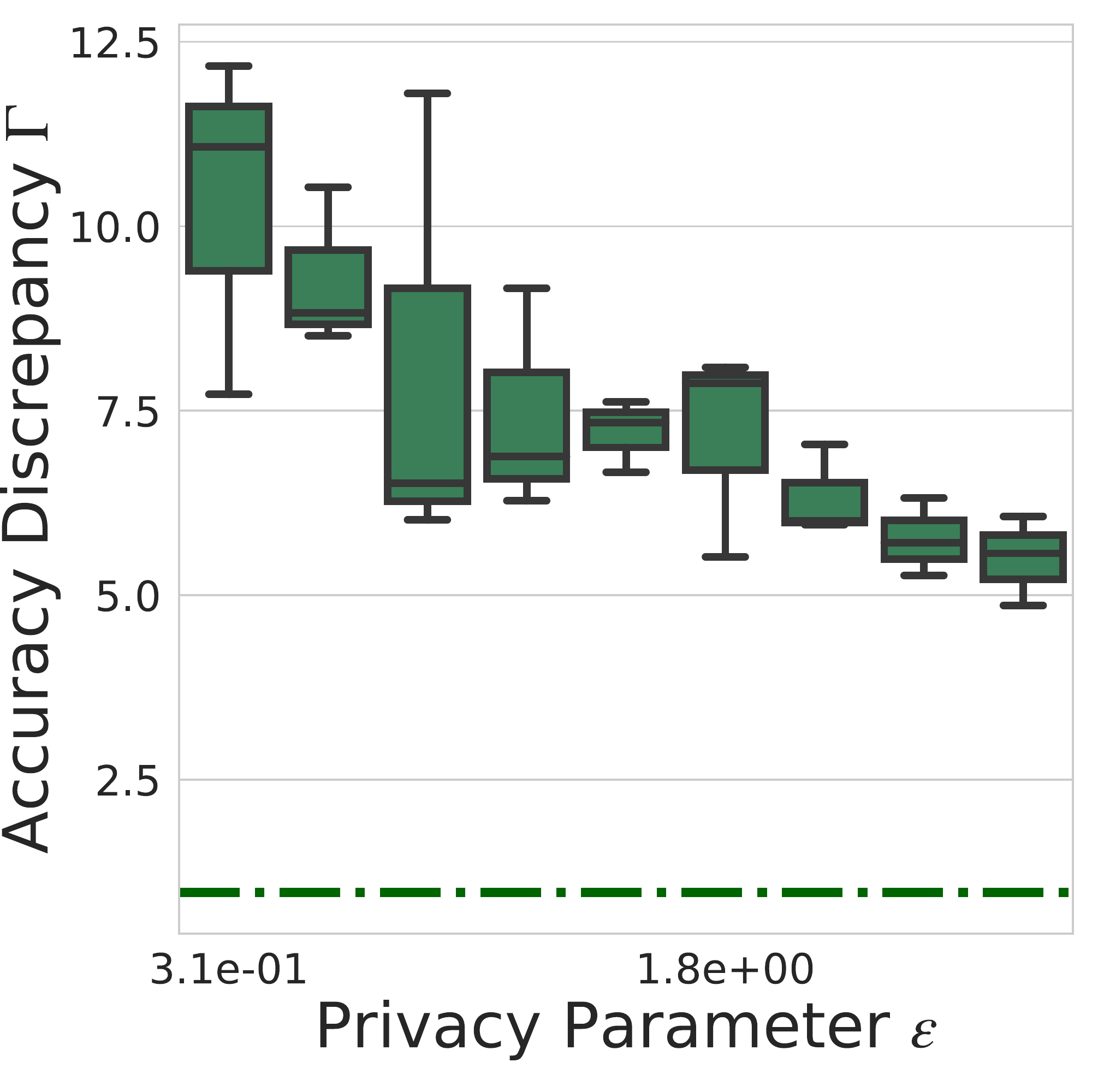%
    \subcaption*{CelebA-5}
        \end{subfigure}
        \begin{subfigure}[c]{0.28\linewidth}
          \centering
          \def\svgwidth{0.99\columnwidth}
          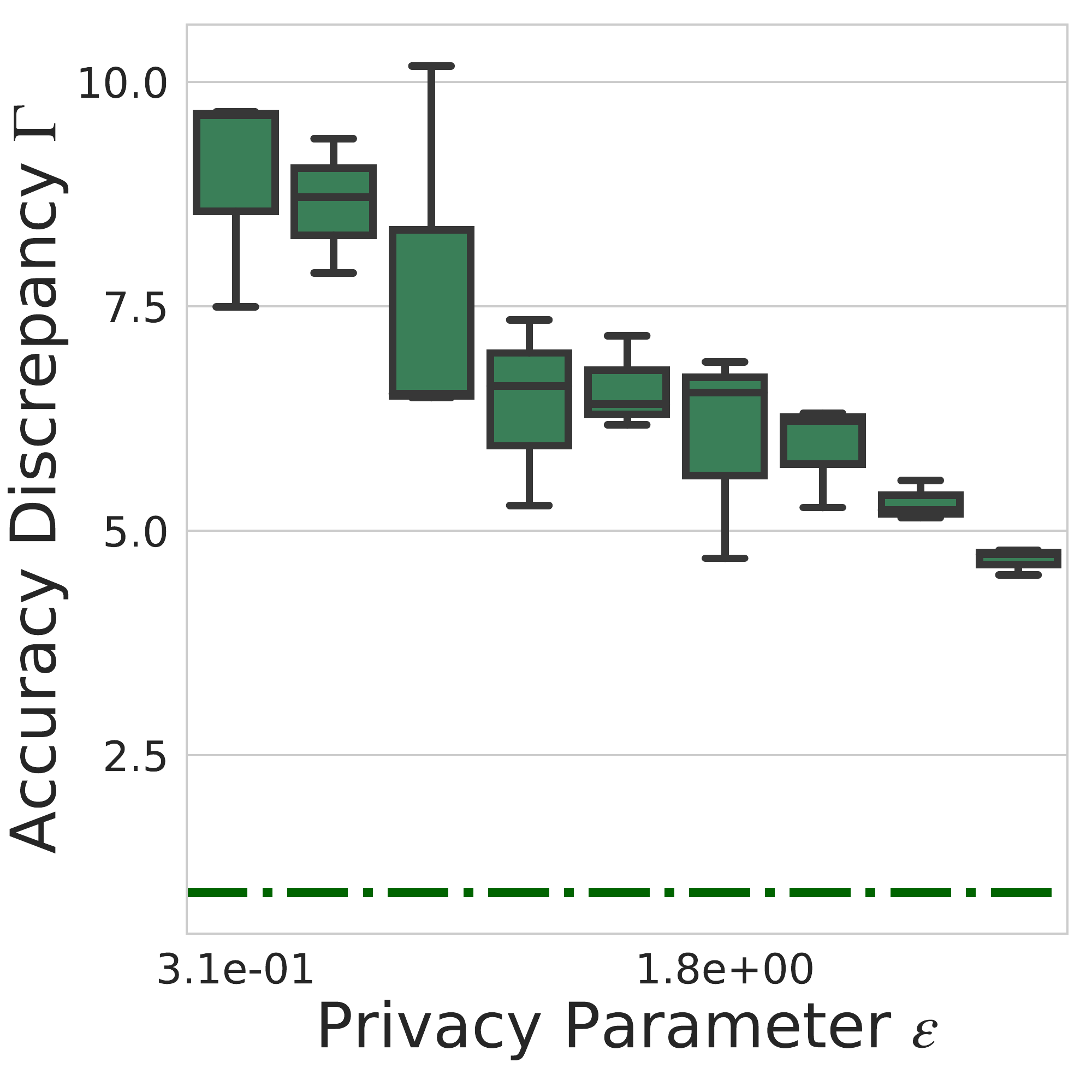%
    \subcaption*{CelebA-10}          
          \end{subfigure}
          \begin{subfigure}[c]{0.28\linewidth}
              \centering
              \def\svgwidth{0.99\columnwidth}
              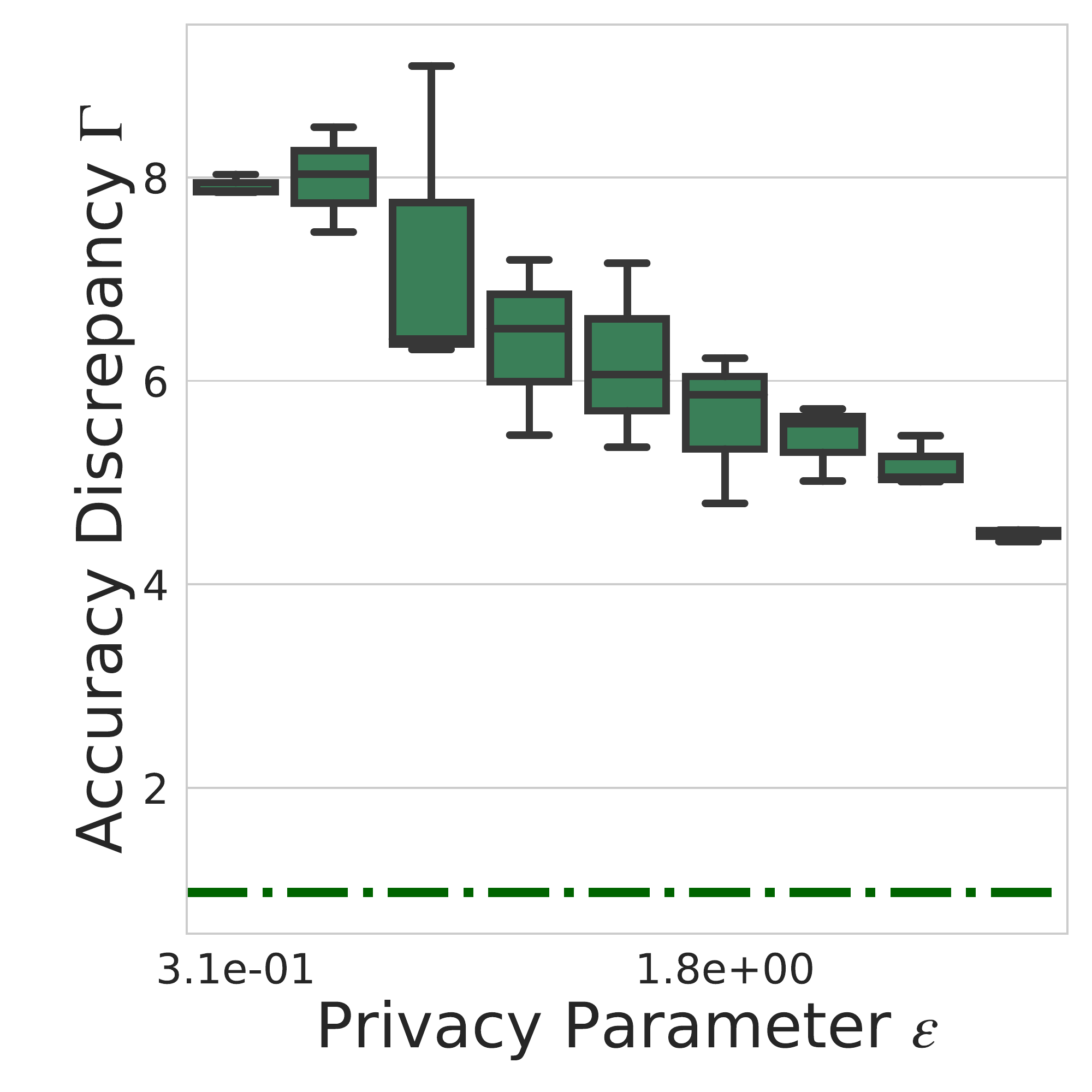%
                \subcaption*{CelebA-20}
              \end{subfigure}\\
              \begin{subfigure}[c]{0.28\linewidth}
                \centering
                \def\svgwidth{0.99\columnwidth}
                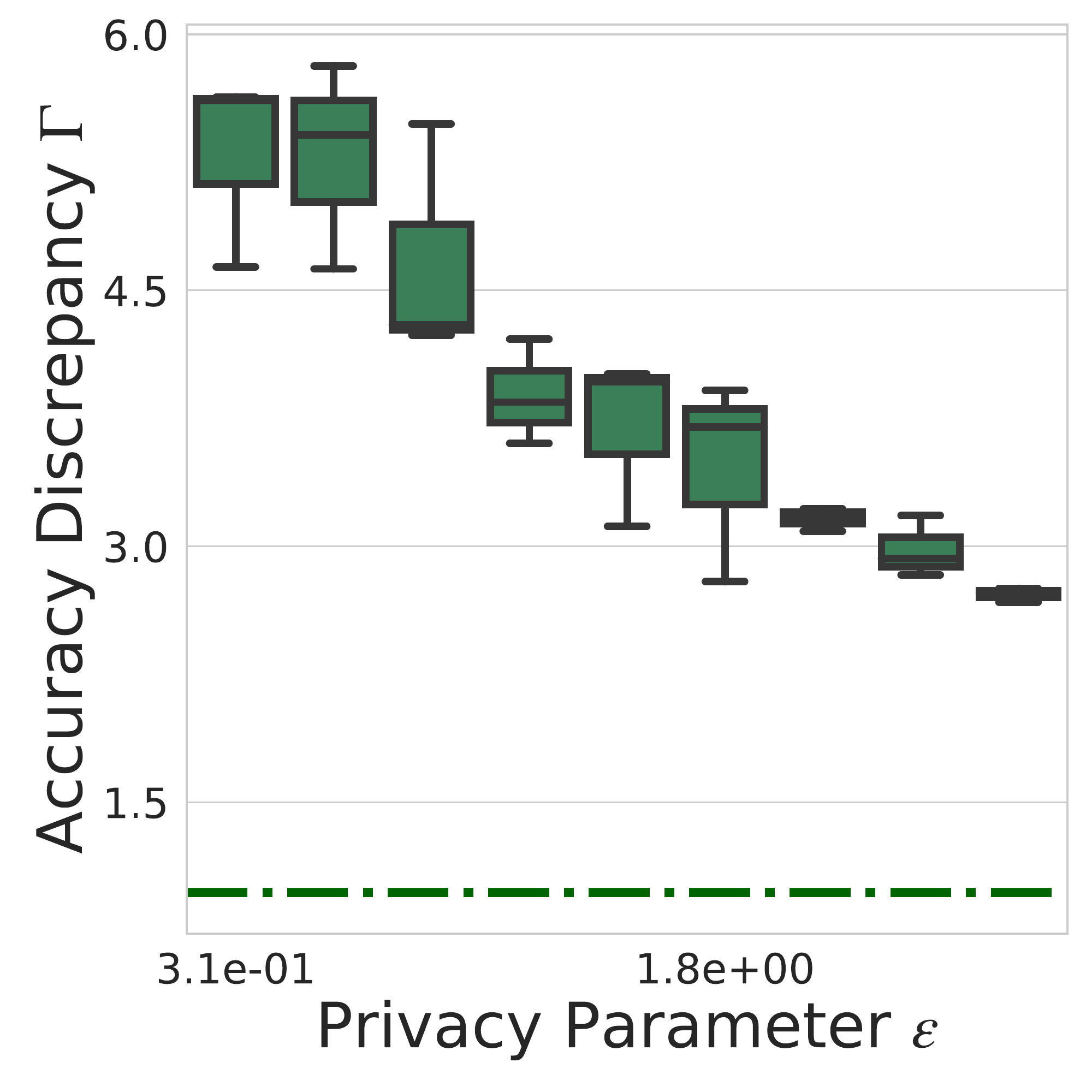%
                \subcaption*{CelebA-60}
                \end{subfigure}
                \begin{subfigure}[c]{0.28\linewidth}
                    \centering
                    \def\svgwidth{0.99\columnwidth}
                    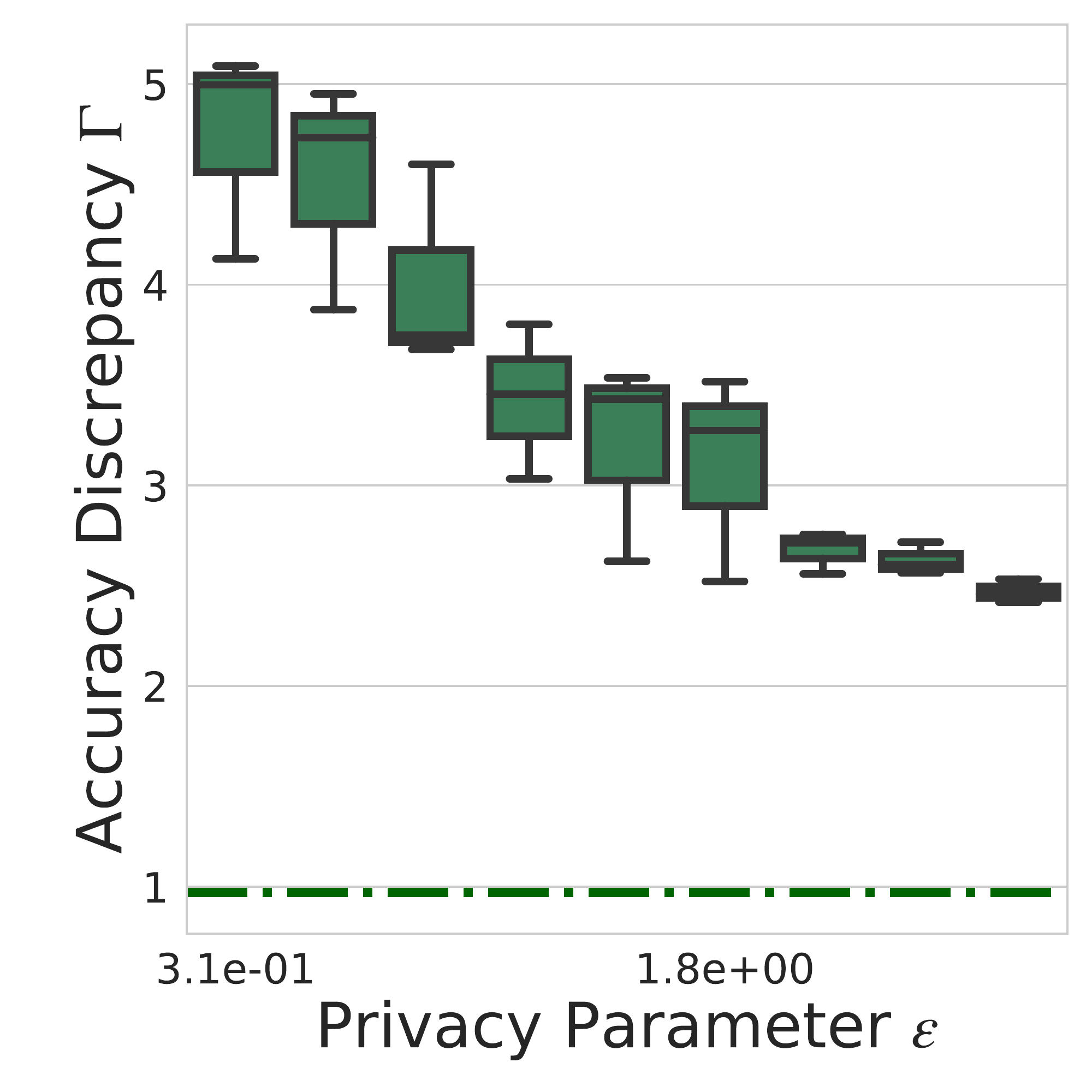%
                    \subcaption*{CelebA-80}
                    \end{subfigure}
                    \begin{subfigure}[c]{0.28\linewidth}
                      \centering
                      \def\svgwidth{0.99\columnwidth}
                      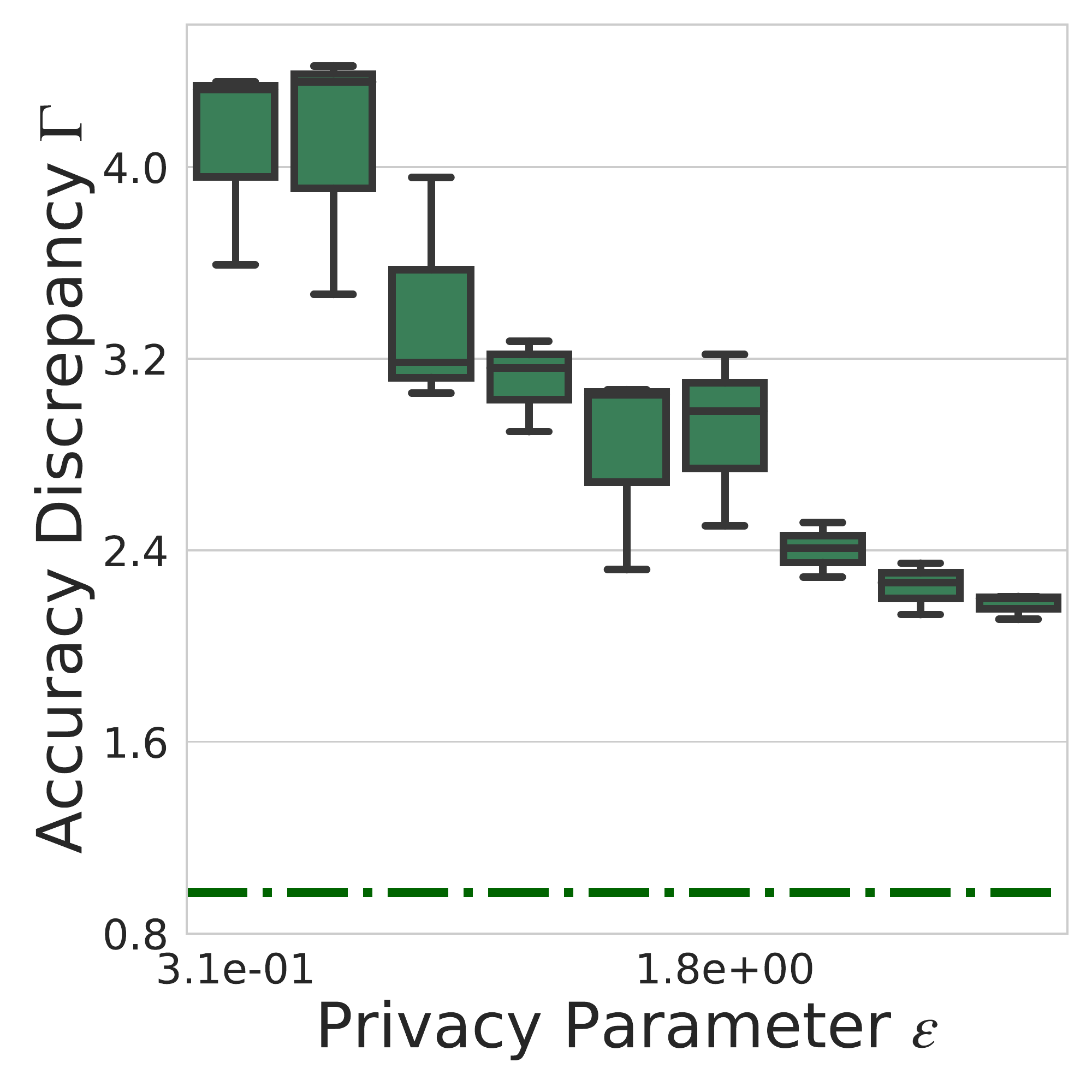%
                      \subcaption*{CelebA-100}
                      \end{subfigure}
            \caption{Accuracy Discrepancy against Privacy}
            \label{fig:celeb-fairness-acc-all-subpop}
  \end{figure}
  
  To construct the subpopulations in CelebA, we use the following attributes \begin{enumerate*}
    \item Pointy Nose, \item Wearing Earrings, \item Wavy Hair, \item Wearing Lipstick, \item Heavy Makeup, \item Attractive, \item Receding Hairline, \item Blurry, \item Bangs, \item Wearing Hat, \item ~Eyeglasses.\end{enumerate*}.
    This also corresponds to Set A used in~\Cref{fig:fair-subpop-size}.
    
    We use the following features for Set B in our experiments in~\Cref{fig:fair-subpop-size}.
    \begin{enumerate*}
      \item Arched Eyebrows
      \item Bags Under Eyes
      \item Blond Hair
      \item Double Chin
      \item High Cheekbones
      \item Pale Skin
      \item Rosy Cheeks
      \item Straight Hair
      \item Wearing Necklace
      \item Wearing Necktie
      \item Young
    \end{enumerate*}

  In~\Cref{fig:celeb-fairness-acc}, we reported results with the
  subpopulation size set to \(40\).In this section, we report results
  where the size belongs to \(\bc{5,10,20,60,80,100}\).
  In~\Cref{fig:celeb-fairness-acc-all-subpop}, we show how disparate
  accuracy changes with \(\epsilon\) for various subpopulation sizes.
  CelebA-x is the minority majority division where the minority
  subpopulations' sizes are less than \(x\).
  in~\Cref{fig:celeb-wg-acc-acc-all-subpop}, we show how the minority
  group accuracy and the overall accuracy changes with \(\epsilon\)
  for various subpopulation sizes. Overall the results, reflect the
  same trend as~\Cref{fig:celeb-fairness-acc} but the worst case
  accuracy discrepancy decreases for increasing subpopulation sizes.
  
  \begin{figure}[H]\centering
    \begin{subfigure}[c]{0.28\linewidth}
        \centering
        \def\svgwidth{0.99\columnwidth}
        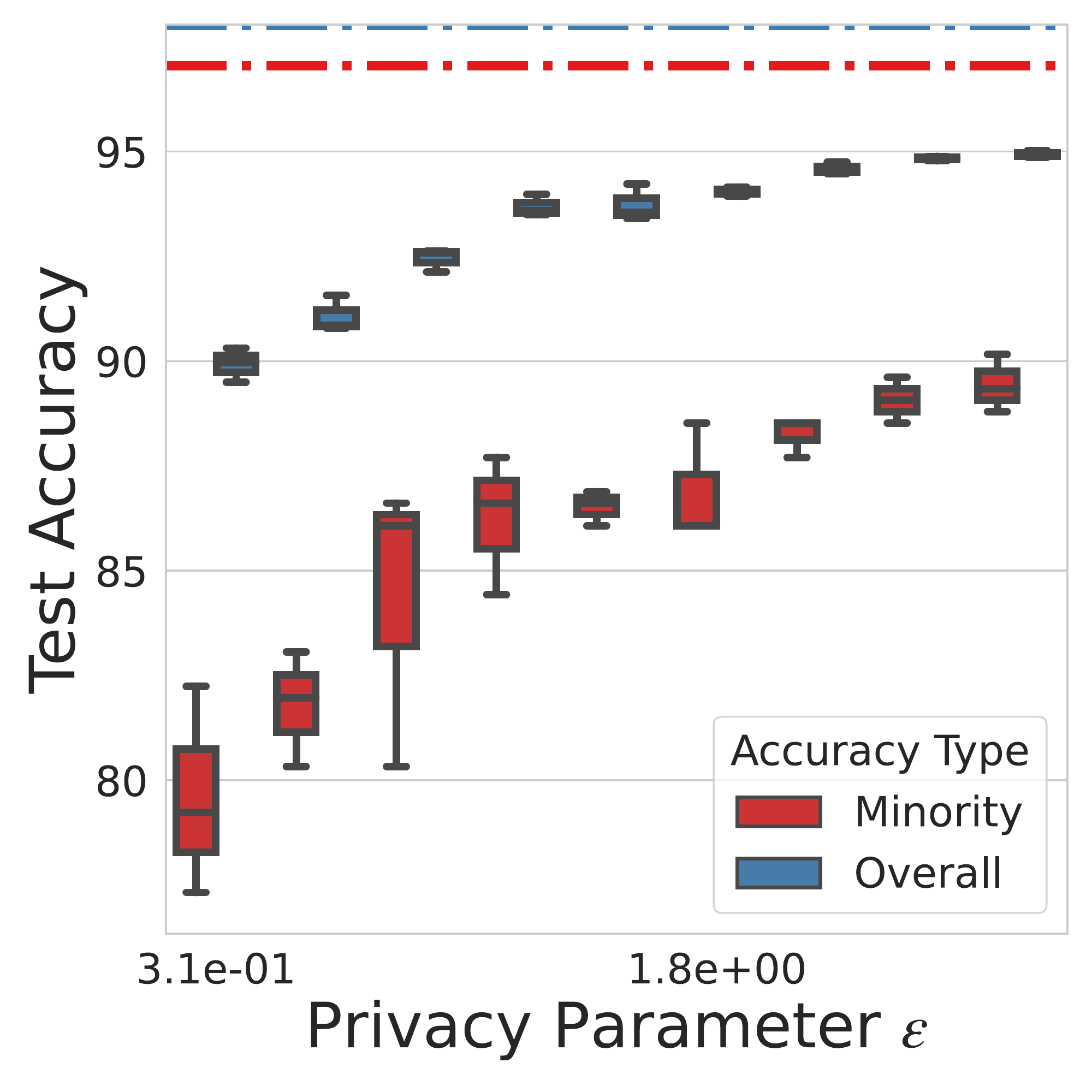%
    \subcaption*{CelebA-5}
        \end{subfigure}
        \begin{subfigure}[c]{0.28\linewidth}
          \centering
          \def\svgwidth{0.99\columnwidth}
          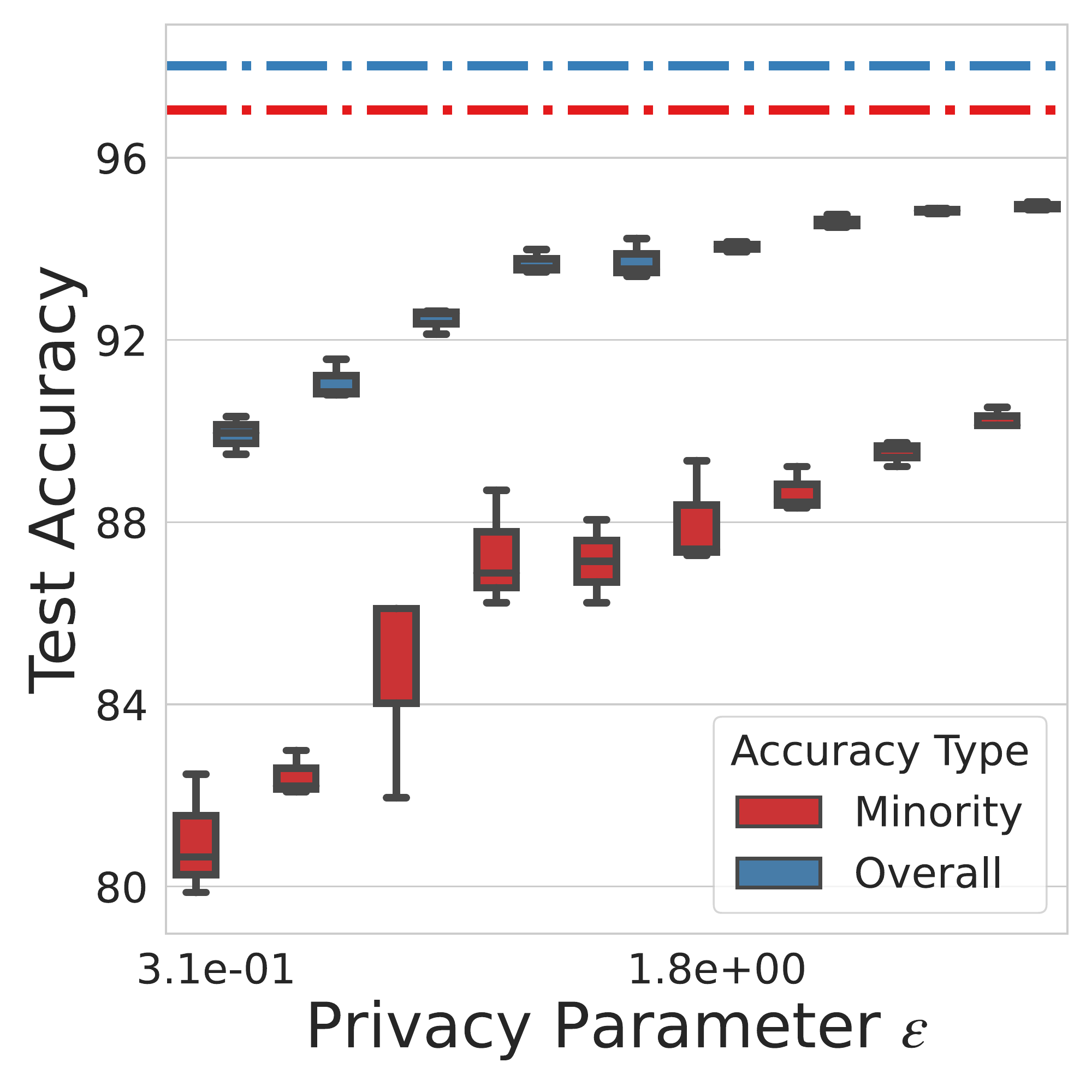%
    \subcaption*{CelebA-10}          
          \end{subfigure}
          \begin{subfigure}[c]{0.28\linewidth}
              \centering
              \def\svgwidth{0.99\columnwidth}
              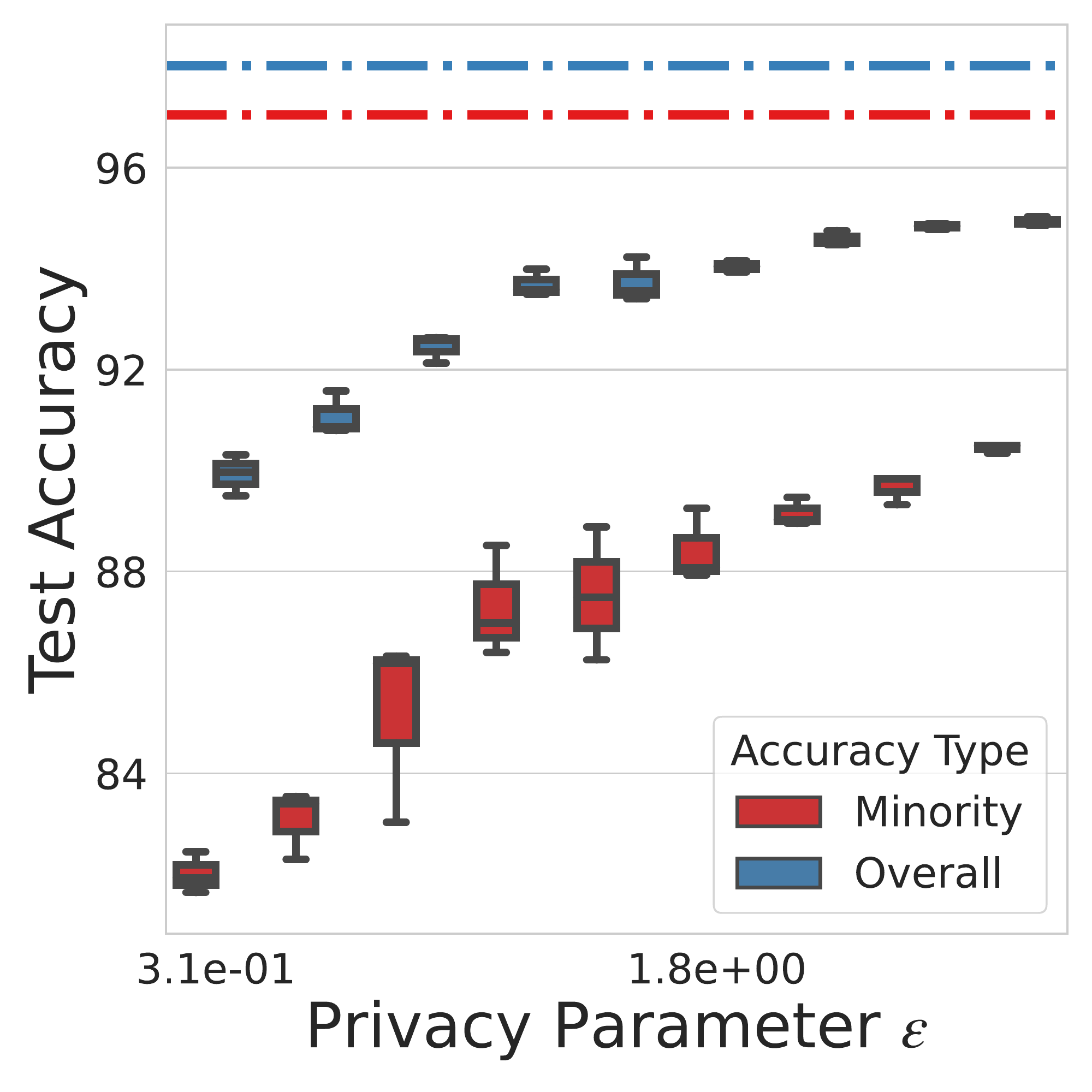%
                \subcaption*{CelebA-20}
              \end{subfigure}\\
              \begin{subfigure}[c]{0.28\linewidth}
                \centering
                \def\svgwidth{0.99\columnwidth}
                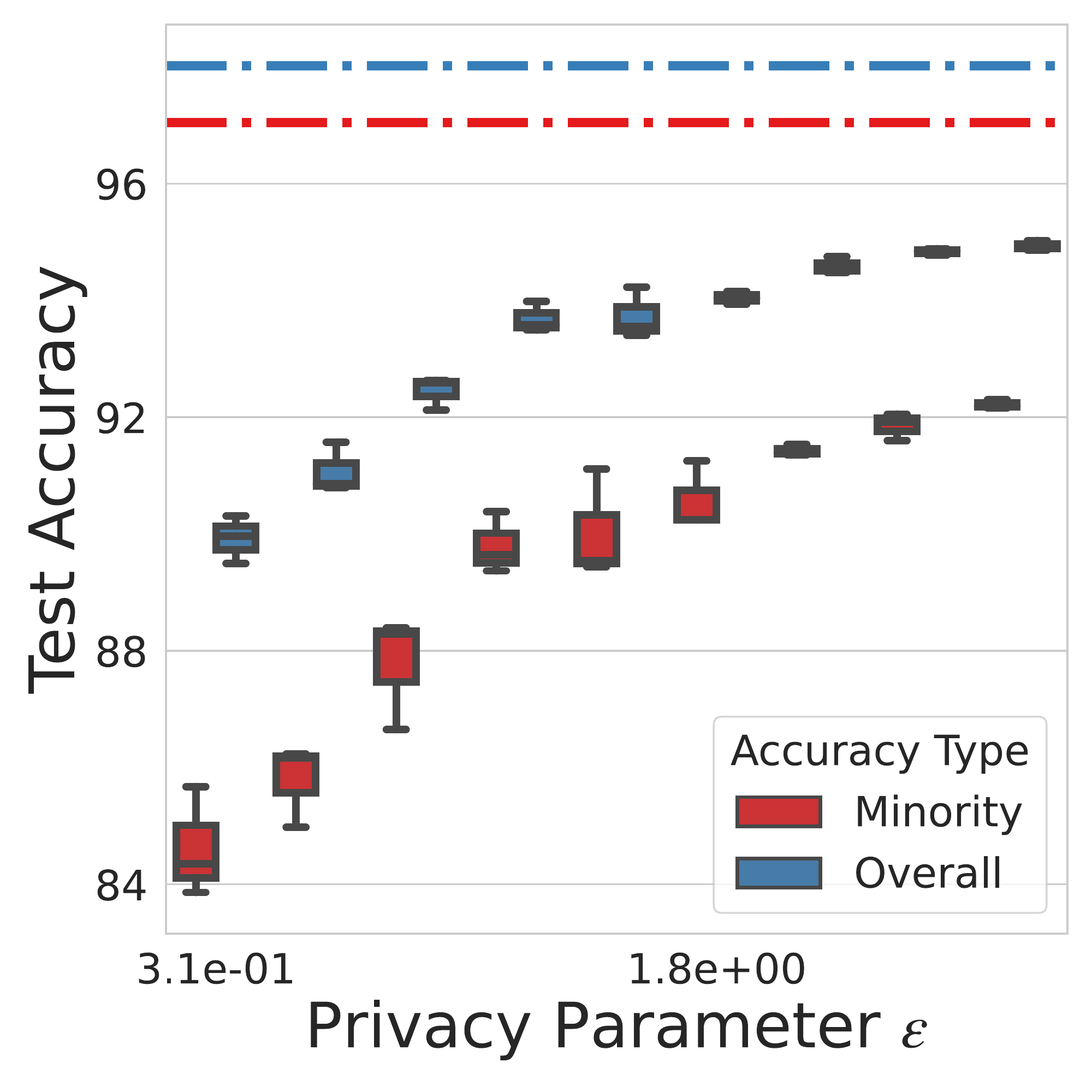%
                \subcaption*{CelebA-60}
                \end{subfigure}
                \begin{subfigure}[c]{0.28\linewidth}
                    \centering
                    \def\svgwidth{0.99\columnwidth}
                    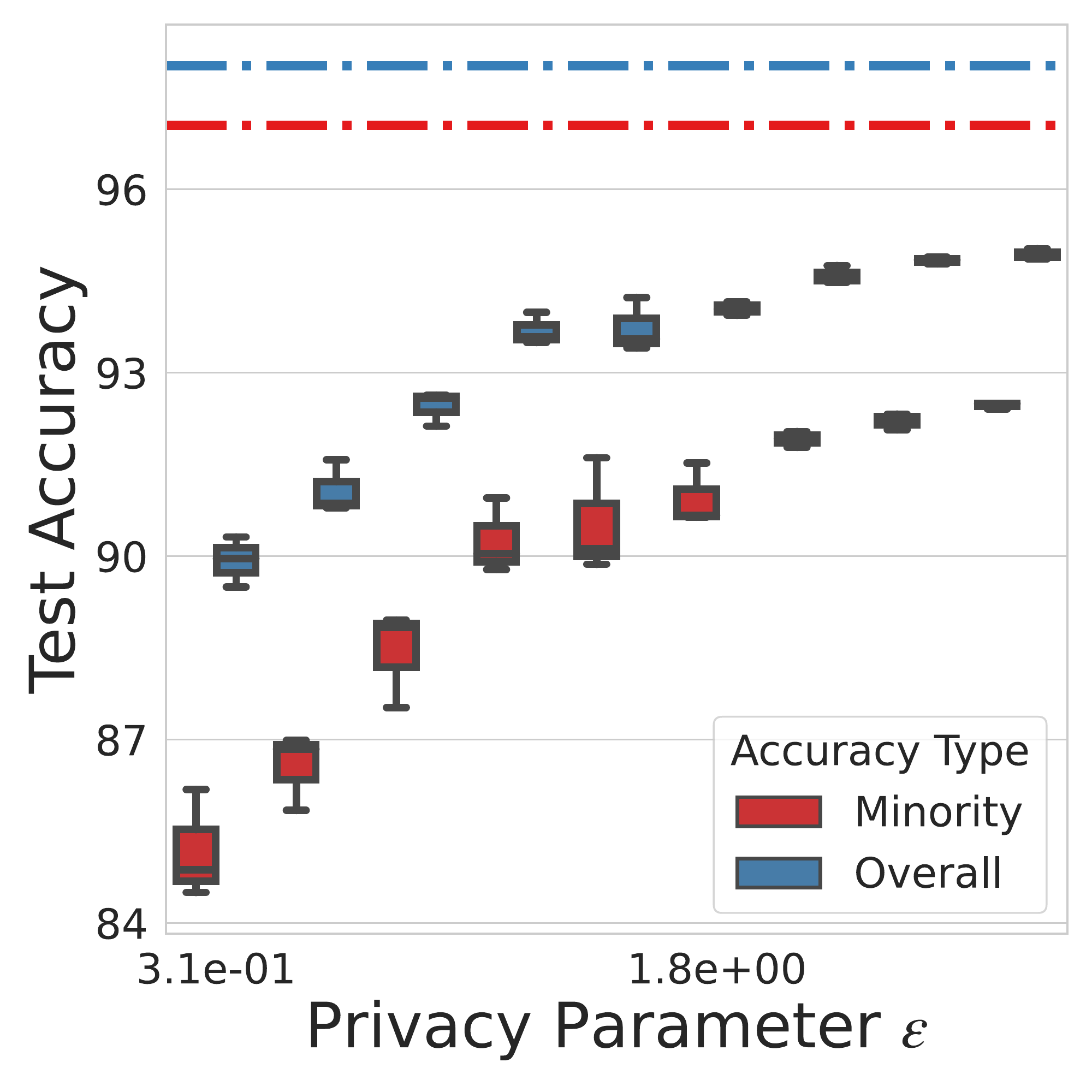%
                    \subcaption*{CelebA-80}
                    \end{subfigure}
                    \begin{subfigure}[c]{0.28\linewidth}
                      \centering
                      \def\svgwidth{0.99\columnwidth}
                      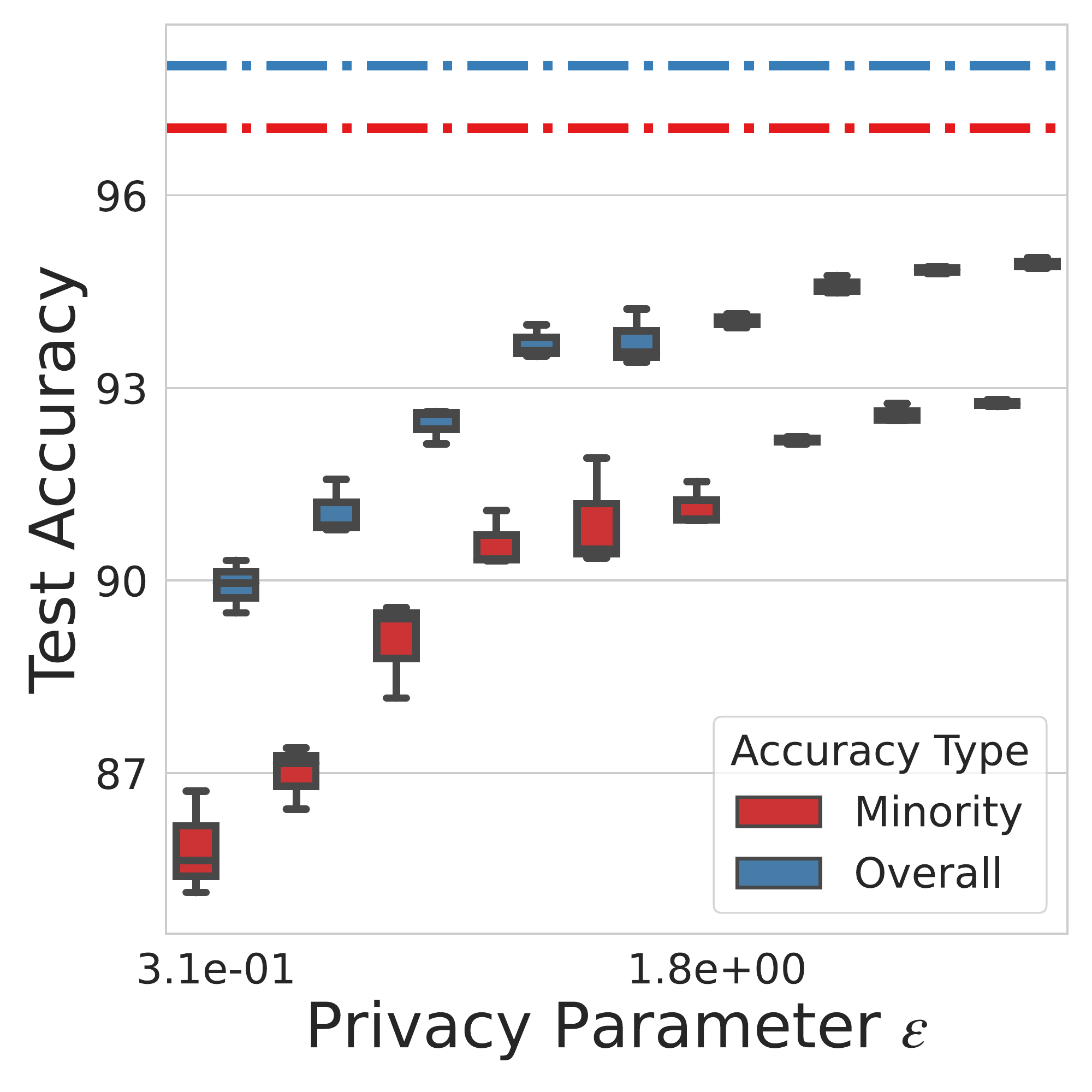%
                      \subcaption*{CelebA-100}
                      \end{subfigure}
            \caption{Minority group and overall accuracy against privacy.}
            \label{fig:celeb-wg-acc-acc-all-subpop}
  \end{figure}

\subsection{Additional details and experiments on \cifar}
  \label{app:more-p-CIFAR}
In~\Cref{fig:fair-acc-cifar10-large}, we plot the same results as~\Cref{fig:cifar10-fair-acc} but with the threshold \(\rho\) set to \(0.01\).
  \begin{figure}[H]\centering
  \begin{subfigure}[c]{0.33\linewidth}
  \centering
  \def\svgwidth{0.99\columnwidth}
  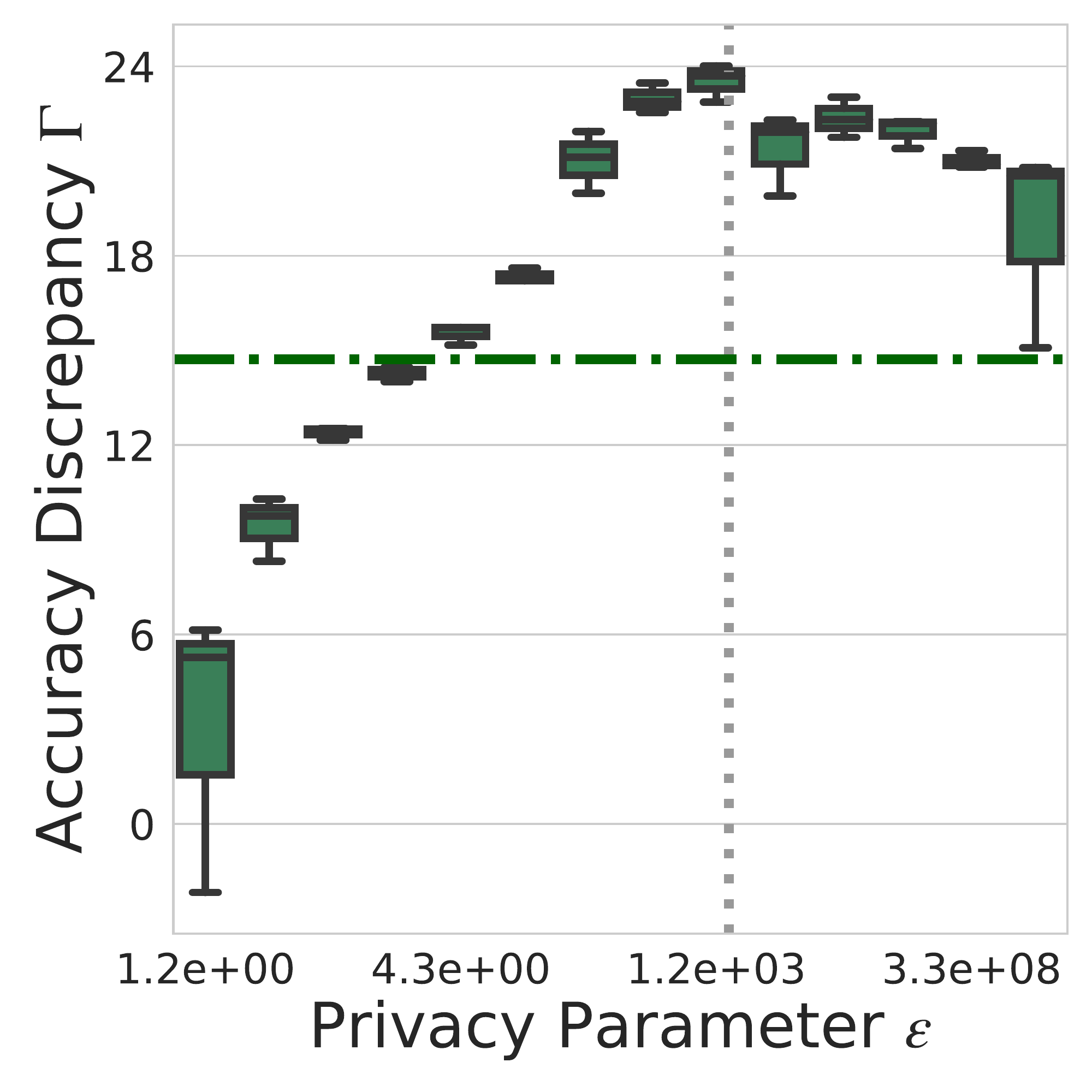%
  \end{subfigure}
  \begin{subfigure}[c]{0.33\linewidth}
      \centering
      \def\svgwidth{0.99\columnwidth}
      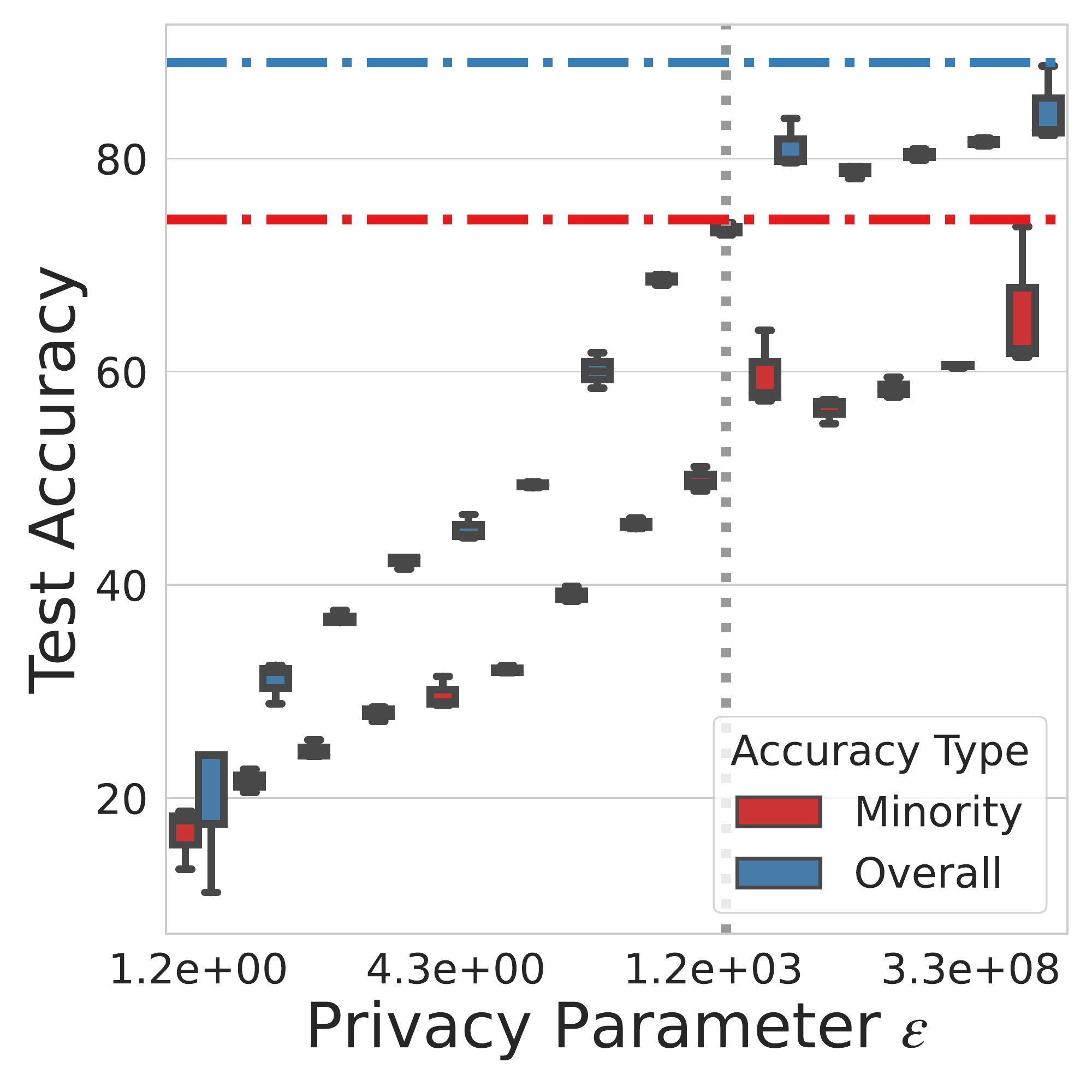%
      \end{subfigure}
    \caption{\cifar---{\bfseries Left} Disparate accuracy
    decreases~(low is fairer) with increasing \(\epsilon\)~(high is
    less privacy) i.e. fairness is worse for stricter privacy.
    {\bfseries Right} figure shows that both minority and overall
    accuracy increases with \(\epsilon\). Dashed lines show the
    corresponding metric of a vanilla model~(no privacy criterion).}
    \label{fig:fair-acc-cifar10-large}
\end{figure}

\end{document}